\documentclass{article}

\PassOptionsToPackage{round}{natbib} 
\usepackage[preprint]{neurips_2026}

\usepackage{amsmath}      
\usepackage{amssymb}      
\usepackage{amsthm}       
\usepackage{mathtools}    

\usepackage{booktabs}     
\usepackage{array}        
\usepackage{multirow}     
\usepackage{tabularx}     
\usepackage{graphicx}     
\usepackage{float}        
\usepackage{rotating}     
\usepackage[normalem]{ulem} 
\usepackage{caption}      
\usepackage{subcaption}   
\usepackage{xcolor}       
\usepackage{tikz}         
\usetikzlibrary{arrows.meta, positioning, decorations.pathreplacing, fit, backgrounds}

\definecolor{wrongbg}{RGB}{253,235,235}
\definecolor{wrongfg}{RGB}{176,0,32}
\definecolor{fixbg}{RGB}{232,245,233}
\definecolor{fixfg}{RGB}{27,120,55}
\definecolor{royalviolet}{RGB}{85,37,150}
\newcommand{\wrongbox}[2]{%
  \begin{minipage}[t]{\linewidth}
    \colorbox{wrongbg}{\parbox{\dimexpr\linewidth-2\fboxsep}{%
      \textbf{\textcolor{wrongfg}{#1}}\\[2pt]#2}}
  \end{minipage}}
\newcommand{\fixbox}[2]{%
  \begin{minipage}[t]{\linewidth}
    \colorbox{fixbg}{\parbox{\dimexpr\linewidth-2\fboxsep}{%
      \textbf{\textcolor{fixfg}{#1}}\\[2pt]#2}}
  \end{minipage}}
\definecolor{covviolet}{RGB}{109,76,169}
\definecolor{covblue}{RGB}{21,101,192}
\definecolor{teacheramber}{RGB}{191,119,23}
\definecolor{contrastamber}{RGB}{192,144,80}

\usepackage[utf8]{inputenc}
\usepackage[T1]{fontenc}
\usepackage{microtype}

\usepackage{enumitem}

\usepackage{algorithm}
\usepackage{algpseudocode}

\usepackage[compact]{titlesec}
\titlespacing*{\section}{0pt}{2.0ex plus .3ex minus .2ex}{1.2ex plus .2ex}
\titlespacing*{\subsection}{0pt}{1.6ex plus .3ex minus .2ex}{0.9ex plus .2ex}
\titlespacing*{\subsubsection}{0pt}{1.2ex plus .2ex minus .2ex}{0.7ex plus .2ex}

\newcommand\blfootnote[1]{%
  \begingroup\renewcommand\thefootnote{}\footnote{#1}\addtocounter{footnote}{-1}\endgroup}

\usepackage{hyperref}
\definecolor{linkblue}{RGB}{21,84,158}
\hypersetup{
  colorlinks=true,
  citecolor=linkblue,
  linkcolor=linkblue,
  urlcolor=linkblue,
  breaklinks=true,
}
\usepackage{cleveref}     

\usepackage{tikz}
\usetikzlibrary{shapes,backgrounds,patterns}

\DeclareRobustCommand{\wrongmarker}{%
  \tikz[baseline=-0.55ex]\node[circle,fill=wrongfg,draw=white,
    line width=.3pt,inner sep=1.45pt] {};}
\DeclareRobustCommand{\projectionmarker}{%
  \tikz[baseline=-0.65ex]\node[diamond,aspect=1.05,fill=fixfg,draw=white,
    line width=.3pt,inner sep=1.65pt] {};}
\DeclareRobustCommand{\successmarker}{%
  \tikz[baseline=-0.65ex]\node[circle,fill=fixfg,draw=white,text=white,
    line width=.3pt,inner sep=.75pt,font=\fontsize{5}{5}\selectfont]
    {$\checkmark$};}
\DeclareRobustCommand{\teachermarker}{%
  \tikz[baseline=-0.7ex]\node[star,star points=5,star point ratio=2.15,
    fill=teacheramber,draw=white,line width=.3pt,inner sep=1.8pt] {};}

\newcommand{\proj}{\mathrm{proj}}

\DeclareMathOperator*{\argmax}{arg\,max}

\newif\ifshowtodos \showtodostrue
\newcommand{\todo}[2][]{\ifshowtodos{\color{blue}\textbf{[TODO%
  \if\relax\detokenize{#1}\relax\else\ (#1)\fi:]}~#2}\fi}

\theoremstyle{plain}
\newtheorem{theorem}{Theorem}
\newtheorem{lemma}{Lemma}
\newtheorem{proposition}{Proposition}
\newtheorem{corollary}{Corollary}
\theoremstyle{definition}
\newtheorem{definition}{Definition}
\newtheorem{assumption}{Assumption}
\theoremstyle{remark}
\newtheorem{remark}{Remark}

\crefname{assumption}{Assumption}{Assumptions}
\Crefname{assumption}{Assumption}{Assumptions}

\crefname{section}{Sec.}{Secs.}
\Crefname{section}{Sec.}{Secs.}
\crefname{subsection}{Sec.}{Secs.}
\Crefname{subsection}{Sec.}{Secs.}
\crefname{subsubsection}{Sec.}{Secs.}
\Crefname{subsubsection}{Sec.}{Secs.}
\crefname{table}{Tab.}{Tabs.}
\Crefname{table}{Tab.}{Tabs.}
\crefname{figure}{Fig.}{Figs.}
\Crefname{figure}{Fig.}{Figs.}
\crefname{appendix}{Appendix}{Appendices}
\Crefname{appendix}{Appendix}{Appendices}

\title{Projected Distillation: On Learning to Solve Hard Problems from Close, Correct Targets}

\date{}

\author{%
\begin{tabular}{cccc}
    \textbf{Shripad V. Deshmukh}\textsuperscript{2}\thanks{Work done during an internship at AWS AI Labs.}
    & \textbf{Christian Bock}\textsuperscript{1}
    & \textbf{Kangrui Ruan}\textsuperscript{1}
    & \textbf{Abhijeet Awasthi}\textsuperscript{1} \\[0.8em]
  \end{tabular}\\
\begin{tabular}{ccc}
      \textbf{Vijay Lingam}\textsuperscript{1}
    & \textbf{Rein Houthooft}\textsuperscript{1}
    & \textbf{Manish Prajapat}\textsuperscript{1}\thanks{Corresponding author: \texttt{manispj@amazon.com}}
\end{tabular}\\[1em]
\textsuperscript{1}AWS AI Labs \qquad \textsuperscript{2}University of Massachusetts
}

\begin{document}

\maketitle
\blfootnote{Accepted at two NeurIPS 2026 workshops:
\href{https://neurips.cc/virtual/2026/workshop/137567}{\emph{Towards Test-Time Continual Learning
Agents}} (TTCL) and \href{https://neurips.cc/virtual/2026/workshop/137615}{\emph{Foundations of LLM
Post-Training in Changing Environments}} (FLLMPT).}

\begin{abstract}
\looseness -1 On-policy learning is central to post-training: training on a model's own generations limits distribution shift and helps preserve existing capabilities. 
On hard problems, however, on-policy RL provides no reward-based learning signal when all attempts fail. On-policy distillation provides next-token supervision from a teacher, but on hard problems, if a rollout goes wrong early, subsequent supervision is conditioned on prefixes that already contain the mistake.
Correct behavior can instead come from a gold or teacher solution, but such off-policy targets can differ substantially from the model's generations and, as we show, may erode existing capabilities.
We therefore treat on-policyness as a continuum and propose projected distillation, which minimally edits the student’s failed attempts into correct targets. 
The method uses contrastive updates to favor corrections over failed attempts and refreshes these corrections as the student evolves, enabling near-on-policy learning even when all student attempts fail.
Theoretically, we prove that the student's learning regret grows with the edit size: it scales linearly with the horizon under minimal correction, as in DAgger, and quadratically under full-answer replacement, as in behavior cloning. 
Empirically, across two students (\texttt{Qwen3-8B} and \texttt{OLMo-2-7B}) and five domains spanning math, code, science, and tool use, projected distillation gives the highest area under the specialization--retention frontier (best on $8$ of $9$ cells against GRPO and OPD, and $7$ of $9$ against SPOT and SODA). 
Our ablations highlight the importance of online corrections, the contrastive loss, and on-policyness: moving targets farther off-policy steadily worsens the specialization–retention trade-off, with retention falling from $0.916$ to $0.744$ when minimal corrections are replaced by student-agnostic gold solutions.
\end{abstract}

\section{Introduction}
\label{sec:intro}
Post-training adapts general-purpose language models to reason, follow instructions, and use tools. 
The dominant approaches are reinforcement learning from verifiable rewards (RLVR; \citealp{shao2024deepseekmath, schulman2017ppo}), distillation from a stronger teacher \citep{agarwal2023gkd}, and supervised fine-tuning \citep{wei2022flan}. 
Across these approaches, the training signal is most effective when it remains on-policy, that is, close to the model's own output distribution.  
Updates near this distribution can strengthen the target behavior while preserving capabilities acquired during pretraining
\citep{agarwal2023gkd, yang2024sdft}. 
Post-training is particularly important for models with only a few billion parameters: although they are inexpensive to run and easy to deploy, they often lack the proficiency required for specialized tasks. 
We study how to teach such models to solve hard problems that they currently answer only rarely, if at all, without degrading their existing capabilities.

A fundamental challenge of on-policy RL is that on very hard problems the on-policy signal disappears. RLVR requires
the model to succeed occasionally, so when every sampled attempt fails the reward is zero
everywhere and nothing is reinforced. On-policy distillation (OPD; \citealp{agarwal2023gkd})
keeps the update on the model's own rollouts and corrects them token by token to be closer to a
response of a stronger teacher.
However, the student rollout of a difficult task contains many mistakes: feedback on the later tokens is
computed on top of an already-wrong prefix, so it largely patches earlier errors instead of
teaching the step at hand~\citep{zhang2026full}. Additionally, the
later-token feedback in OPD is inefficient: as the student learns, the prefix on which later-token feedback was given and training budget was spent becomes less probable.

Alternatively, the training signal can come from a complete correct answer supplied from outside,
either ground-truth data or a stronger teacher's solution. Unlike OPD, which corrects the student's
own rollout token by token, this target replaces the rollout entirely: the whole training budget
goes on producing a correct answer rather than on patching the one the student wrote. Such a target
is drawn from a distribution other than the student's own, so it is off-policy. The further it lies
from what the student would produce, the more the student specializes on the task while eroding its
abilities elsewhere. This suggests keeping the target a complete
correct answer but choosing the one closest to what the student already writes, so the supervision
stays near-on-policy.

\begin{figure}[t]
  \centering
  \begin{minipage}[c]{0.4\linewidth}
    \centering
    \includegraphics[width=\linewidth]{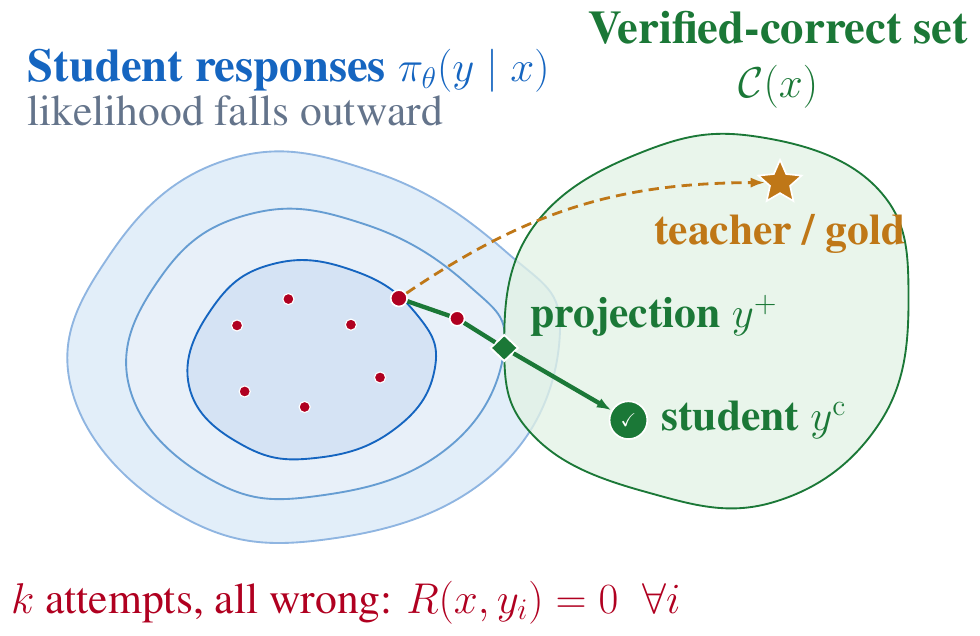}
  \end{minipage}\hfill
  \begin{minipage}[c]{0.58\linewidth}
    \caption{\textbf{Learning from the nearest correct answer.} Correct answers form a set
    $\mathcal C(x)$, but the student assigns them very different likelihoods. On a hard problem,
    its $k$ on-policy attempts (\wrongmarker) are all wrong and earn no reward. A teacher or gold
    answer (\teachermarker) supplies correctness, but may lie far from what the student writes.
    Projected distillation instead minimally edits a failed attempt into a newly computed projection
    (\projectionmarker) and refreshes that target as the student changes. The green path illustrates
    repeated online rounds: successive wrong rollouts may move closer until the student produces a
    correct answer itself (\successmarker). The resulting small edits keep supervision near-policy;
    writing $m$ for edit size and $H$ for the generation horizon, our analysis bounds their drift
    by a term proportional to $mH$.}
    \label{fig:motivation}
  \end{minipage}
\end{figure}

Several recent works pursue similar ideas. SPOT~\citep{lin2026spot},
SODA~\citep{chen2026soda}, and SuperCorrect~\citep{yang2025supercorrect} fix the student's
mistakes with a strong black-box teacher and learn from the corrections. What they leave
unexamined is the on-policyness of the correction itself, how it changes with the size of the
edit, and whether refreshing the corrections online against the evolving student matters. We focus on these aspects. We call the recipe of minimal, online corrections \emph{projected
distillation} (\Cref{fig:motivation}) and study it along five axes: how well it learns from close,
correct solutions relative to other post-training approaches; why the edit must be minimal to keep
the correction near-on-policy; why the corrections must be refreshed against the still-training
student; which objective learns best from a correction; and how the projector itself can be
improved. A regret analysis in the imitation-learning setting shows the importance of the edit size.

\noindent\textbf{Our contributions.}
\begin{enumerate}[leftmargin=1.6em,itemsep=3pt,topsep=4pt,parsep=0pt]
  \item We introduce \emph{projected distillation}, a post-training recipe that learns from minimal,
  online corrections of the student's mistakes. It needs only the correction text, so any capable
  black-box model can serve as the projector. Across two students and five benchmarks it attains the
  best specialization--retention trade-off, and retains prior abilities where standard post-training
  forgets on continual-learning metrics (\Cref{sec:rq1-main}).

  \item Holding correctness fixed, we show that the correction must be both minimal and online.
  Pushing it off-policy monotonically worsens the trade-off (\Cref{sec:rq2-onpolicyness}), and
  freezing the pool instead of refreshing it makes it collapse as the student drifts
  (\Cref{sec:rq3-refresh}). We further give an efficient training objective, a diff-localized
  contrast with self-training on correct rollouts (\Cref{sec:rq4-loss}), and a refinement procedure
  that improves the projections from a mid-capacity ($32$B) teacher without updating its parameters
  (\Cref{sec:rq5-teacher}).

  \item We theoretically analyze projected distillation through imitation learning. The edit size
  interpolates between DAgger's linear-in-horizon regret and behavior cloning's quadratic regret
  (\Cref{sec:theory}).
\end{enumerate}

\section{Related Work}
\label{sec:related}
\noindent\textbf{Methods for improving a model on hard problems.}
Approaches differ in where the correct target comes from and how far that target lies from the
model's own distribution; together these largely determine the specialization--retention trade-off.
Reinforcement learning from verifiable rewards (RLVR; GRPO~\citep{shao2024deepseekmath},
PPO~\citep{schulman2017ppo}) draws the target from the model's own rollouts, but on the
$\mathrm{pass}@k{=}0$ set every rollout is wrong and no signal remains. On-policy distillation
(GKD~\citep{agarwal2023gkd}) and self-distillation~\citep{yang2024sdft} keep the update on those
rollouts while matching a teacher's token distribution, which requires teacher logits. Supervised
fine-tuning on an external gold or teacher solution supplies correctness directly but is maximally
off-policy, so fitting it forgets. Between these, there is scope to manufacture a correct target based on the student's attempt. This opens up the possibility of treating the correct solutions along a continuum from on-policy to off-policy with respect to the model's current generation style. Our work studies how far the edits may go, whether they need to be recomputed as
the model changes, and what objective to use for effective learning.

\noindent\textbf{Closest work.}
The literature presents approaches similar to ours: LeMa~\citep{an2023lema}
mines corrections once, from an assortment of models rather than the student, with free-form
rewrites and an objective that uses response likelihood maximization. SPOT~\citep{lin2026spot} gathers mistakes of the base model, rectifies the mistakes using minimal edits from an oracle, creates a preference dataset and learns with a reward-based binary cross-entropy objective.
SODA~\citep{chen2026soda} creates mistake--correction pairs without emphasizing minimal edits, performs an SFT warm-up, and then, with the SFT'ed policy fixed, learns from static preference pairs. SuperCorrect~\citep{yang2025supercorrect} distills thought templates
from a teacher and applies collaborative DPO over the teacher's correction traces. We also work in the hard-task regime where the model fails. However, our paper focuses mainly on the regime of online learning from close, correct targets to teach the student how to effectively solve hard tasks. We perform a controlled study of the edit sizes used to provide corrections (\Cref{sec:rq2-onpolicyness}), static vs.\ online-refreshed pools of corrections (\Cref{sec:rq3-refresh}) and effective learning objectives for black-box learning (\Cref{sec:rq4-loss}). In addition, we theoretically analyze how the edit-size influences the learning regret, which distinguishes our work from the previous ones. Concurrent work~\citep{karan2026sampling} explores self-projection with an MCMC sampler, in which the student itself rewrites its response with the reference solution in view; \Cref{app:selfproj} compares it with black-box projection on its math benchmark.

\section{Projected Distillation}
\label{sec:method}
We study post-training on problems the model rarely solves, where its own samples earn no reward and
correctness has to come from outside. 

\noindent\textbf{Setup.}\label{sec:prelim}
Given a \emph{problem} $x$, a solution is a token sequence $y \in \mathcal{V}^{\le H}$ drawn from a vocabulary $\mathcal{V}$ with a generation horizon $H$.
The \emph{student} produces $y$ autoregressively via $\pi_{\theta_S}(y \mid x) = \prod_{t} \pi_{\theta_S}(y_t \mid x, y_{<t})$, where $y_t$ denotes the $t$-th token, $y_{<t}$ its prefix, and $\pi_{\mathrm{ref}}$ the frozen initialization of the student.
Correctness is judged by a \emph{verifier} $R(x,y) \in \{0,1\}$ (e.g., exact-match checking or unit tests), inducing the set of accepted solutions $\mathcal{C}(x) = \{y : R(x,y) = 1\}$.
Our focus is on problems that lie beyond the student's current ability, i.e., $\mathrm{pass}@k{=}0$ set, for which all $k$ sampled attempts fail.
Further, in \Cref{sec:theory}, we formalize the generation process using a dynamic-token sequence MDP to provide theoretical insights.

\noindent\textbf{Projection.}\label{sec:method-projection} 
The accepted set $\mathcal{C}(x)$ typically contains many solutions, but these solutions can differ substantially in likelihood under the student.
Training toward a solution the student already considers likely requires less movement of its output distribution and therefore better preserves unrelated capabilities~\citep{kirkpatrick2017ewc, kumar2022finetuning}.
We therefore choose, among the verifier-accepted solutions in $\mathcal C(x)$, the one with highest likelihood under the student.

\begin{definition}[Projection] 
\label{def:projection} 
For a student $\pi_{\theta_S}$ and a problem $x$, the \emph{projection} is the correct answer the
student is most likely to produce, $\proj(x) \in \argmax_{y'\in\mathcal{C}(x)}
\log\pi_{\theta_S}(y'\mid x)$.
\end{definition}

Practically, searching the true projection within $\mathcal{C}(x)$ is intractable. We propose to
find a close approximation with the help of the student's attempt itself. The intuition behind this
is that a wrong attempt is a sample from the student and is therefore already likely under it; a
correction that preserves most of its tokens preserves most of that likelihood, while the verifier
places the result in $\mathcal{C}(x)$. Minimal editing is thus a tractable surrogate for
\Cref{def:projection}, with the verifier supplying correctness and the retained tokens supplying
proximity. We test the surrogate directly in \Cref{sec:rq2-onpolicyness}, varying the edit size and
recording the likelihood shift
$\Lambda=\log\pi_{\theta_S}(y^+\mid x)-\log\pi_{\theta_S}(y^-\mid x)$ that each size demands.
A \emph{projector} receives the problem and the attempt in context and is prompted to make the
smallest change that corrects the answer (examples in \Cref{app:proj-examples}; prompts in
\Cref{app:prompts}).

\noindent\textbf{Learning from projection.} \label{sec:method-blackbox}
In each round we prompt the student on the hard problems and collect a batch of responses. Most of
them are wrong. We write $y^-$ for a wrong response and $y^+$ for its projection, and
$y^{\mathrm c}$ for the responses the student gets correct on its own. We have access only to the text the projector returns, not to its logits, so we cannot
match its token distribution. We learn from the pair $(y^-,y^+)$ by preference instead. The
projection edits only a small part of the attempt, so $y^-$ and $y^+$ agree almost everywhere, and a
contrast over their full length would spend most of its gradient on text that was already correct.
We therefore restrict the contrast to the tokens the edit changed. A token-diff of $y^-$ against
$y^+$ marks a \emph{mistake span} $E^-$ and a \emph{fix span} $E^+$; for instance, for
$y^-=\mathtt{ABCD}$ and $y^+=\mathtt{AXCZ}$, $E^-=\{\mathtt{B},\mathtt{D}\}$ and
$E^+=\{\mathtt{X},\mathtt{Z}\}$.
We apply a DPO
contrast~\citep{rafailov2023dpo} to those spans alone, using the frozen initialization
$\pi_{\mathrm{ref}}$ as reference and $\beta$ for its strength:
\begin{equation}
  \label{eq:proj-dpo}
  \mathcal{L}_{\mathrm{Diff\text{-}Local\text{-}DPO}}
  = -\,\mathbb{E}_{x}\,\log\sigma\!\Big(
      \beta\big[\textstyle\sum_{t\in E^+}\log\tfrac{\pi_{\theta_S}(y^+_t\mid x,\,y^+_{<t})}{\pi_{\mathrm{ref}}(y^+_t\mid x,\,y^+_{<t})}
      -\sum_{t\in E^-}\log\tfrac{\pi_{\theta_S}(y^-_t\mid x,\,y^-_{<t})}{\pi_{\mathrm{ref}}(y^-_t\mid x,\,y^-_{<t})}\big]\Big),
\end{equation}
Here $\sigma$ is the logistic function. The contrast constrains only how likely $y^+$ is relative to
$y^-$, so we add a term that supervises the student's own correct responses directly:
\begin{equation}
  \label{eq:proj-objective}
\!\!\!\mathcal{L}_{\mathrm{Projected\text{-}Distillation}} = \mathcal{L}_{\mathrm{Diff\text{-}Local\text{-}DPO}} + \mathcal{L}_{\mathrm{SFT}},
  ~\text{where}~~
  \mathcal{L}_{\mathrm{SFT}} = -\,\mathbb{E}_{x}\,\textstyle\sum_t \log\pi_{\theta_S}(y^{\mathrm c}_t\mid x, y^{\mathrm c}_{<t}).\!
\end{equation}

\begin{figure}[t]
  \centering
  \includegraphics[width=.5\linewidth]{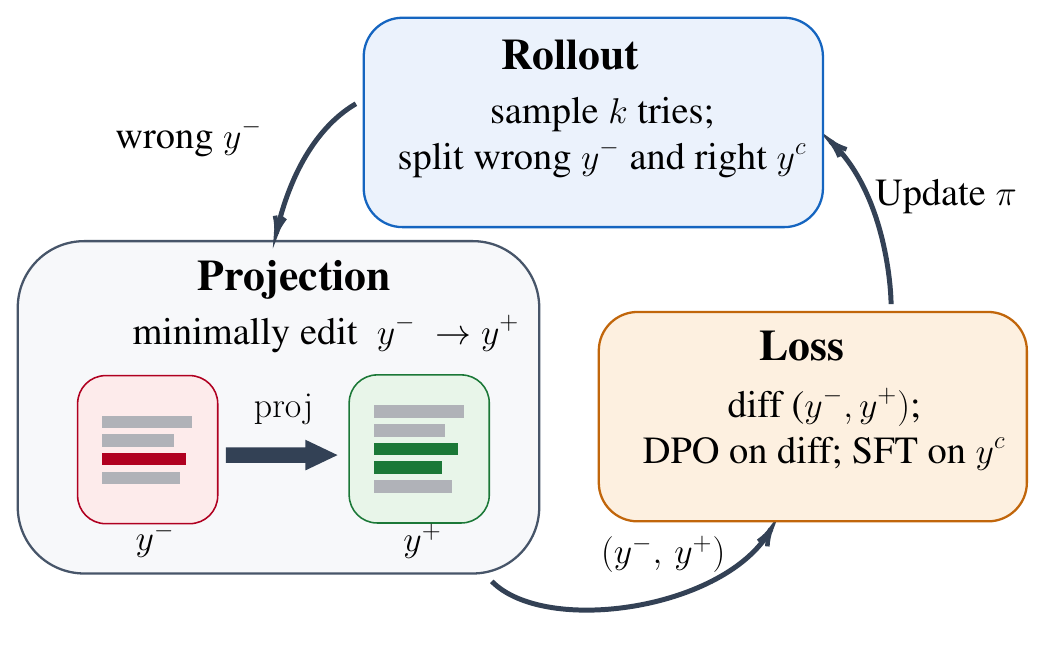}
  \caption{\textbf{One round of projected distillation.} The student answers each problem $k$
  times, and its responses split into the wrong ones $y^-$ and the correct ones $y^{\mathrm c}$. A
  projector rewrites a wrong response into a correct one $y^+$ with the smallest edit that works.
  The two are aligned token by token, and the preference term acts only on the spans that differ
  (\textcolor{wrongfg}{removed}, \textcolor{fixfg}{added}), while $y^{\mathrm c}$ is supervised
  directly. The updated student begins the next round, so its corrections are recomputed as it
  changes.}
  \label{fig:method-flow}
\end{figure}

Therefore, the wrong responses supply the contrast and the correct
ones the supervised term. The next round re-prompts the student on the same problems and projects
whatever it gets wrong then, so the targets follow the current policy. \Cref{fig:method-flow} summarizes one round of our approach. \Cref{sec:rq3-refresh} compares
rebuilding the targets each round against reusing a static correction pool, and \Cref{sec:rq4-loss} compares our loss~(\ref{eq:proj-objective}) against five alternatives on the same data.

\section{Experiments}
\label{sec:experiments}
First, we compare projected distillation with standard post-training approaches in the
hard-task regime, and ask whether what a student learns on one benchmark's hard set carries over to
the hard problems of other benchmarks (\Cref{sec:rq1-main}). The remaining sections examine the recipe itself, each
component of which could have been set differently: the correction may change more or fewer of the
attempt's tokens (\Cref{sec:rq2-onpolicyness}), it may be computed once or recomputed as the student
learns (\Cref{sec:rq3-refresh}), and several objectives may be applied to it (\Cref{sec:rq4-loss}).
We vary each in turn while holding the remainder fixed, then ask whether the projections can
be improved further by test-time scaling (\Cref{sec:rq5-teacher}), and finally whether the same model
can serve as its own projector when given privileged information (\Cref{sec:rq6-selfproj}).

\noindent\textbf{Experimental Setup:}
We study two instruction-tuned students from different families, \textbf{Qwen3-8B}
\citep{qwen2025qwen3} and \textbf{OLMo-2-7B}~\citep{olmo2_2024}, on five benchmarks
spanning math, code, tools and science: GSM8K~\citep{cobbe2021gsm8k},
MATH~\citep{hendrycks2021math}, LiveCodeBench~v6~\citep{jain2024livecodebench},
SciKnowEval~\citep{feng2024sciknoweval}, and ToolAlpaca~\citep{tang2023toolalpaca}.
For each (student, benchmark) pair we train \emph{only} on its hard set: the problems on which the
base student fails all $k{=}16$ sampled attempts at temperature $0.8$. The hard set is
mined from the benchmark's training split (\Cref{app:hardset}), and retention is measured in-domain on a disjoint held-out slice of
the same benchmark.
All methods use matched budgets for optimization updates and per-step generation. The remaining configuration is in \Cref{app:exp-details}, with prompt
templates in \Cref{app:prompts} and further training details in \Cref{app:train-extra}. Unless noted,
results are averaged over $3$ training seeds, and \Cref{app:eval-resolution} separates training-seed
from evaluation-seed spread.
\Cref{app:math-full} additionally evaluates the trained checkpoints on suites they never saw, and
\Cref{app:transfer} on the hard problems of benchmarks they never saw.

\looseness-1 \noindent\textbf{Baselines.}
We compare against four baselines, each probing a different axis of the post-training design space.
\textbf{GRPO} \citep{shao2024deepseekmath} performs on-policy RL with verifiable rewards, where
positive rewards are scarce.
\textbf{OPD} \citep{agarwal2023gkd} performs white-box distillation from the token-level
logits of a strong teacher conditioned on the student's own prefixes, using a same-family 32B
teacher (Qwen3-32B or OLMo-2-32B) for each student; our method uses these same 32B teachers as its
projection models, and both methods are given the gold answer as privileged information.
\textbf{Surgical Post-Training} (SPOT; \citealp{lin2026spot}) minimally corrects the student's
initial attempts and trains on the resulting \emph{frozen} preference data, while
\textbf{Semi-On-policy Distillation} (SODA; \citealp{chen2026soda}) fine-tunes the student on
teacher corrections and then contrasts a teacher response against a frozen copy of that fine-tuned
student.
Against SPOT and SODA every method draws its corrections from one shared black-box model, which
isolates the learning algorithm from the projector.

\noindent\textbf{Metrics.}
We report \textbf{specialization}, pass@1 on the hard set, and \textbf{retention}, pass@1 on the
held-out slice. A method can raise either at the other's expense, so neither effectively ranks methods on its
own. Our headline metric is the \textbf{area under the specialization--retention Pareto frontier}
(AUPF), the retention a method sustains as it specializes further, anchored at a shared untrained
base. AUPF is low both for a method that never specializes and for one that specializes by giving
up retention. We report it alongside three continual-learning summaries of the retention
trajectory, backward transfer \citep{lopezpaz2017gem}, forgetting \citep{chaudhry2018rwalk} and
training-averaged accuracy \citep{kemker2018measuring}, and retention-safe specialization, the
highest specialization reached while retention stays within $\varepsilon$ of base. Full definitions
are in \Cref{app:rq1-full}.

\subsection{RQ1: How does projected distillation compare with baselines?}
\label{sec:rq1-main}
We compare against the four baselines, drawing our corrections first from a same-family $32$B
teacher and then from a shared black-box projector, and report AUPF in \Cref{tab:rq1-aupf}, the
continual-learning metrics in \Cref{tab:rq1-cl}, and the pooled frontiers in
\Cref{fig:pareto-all}.

\begin{figure}[t]
  \centering
  \includegraphics[width=.8\linewidth]{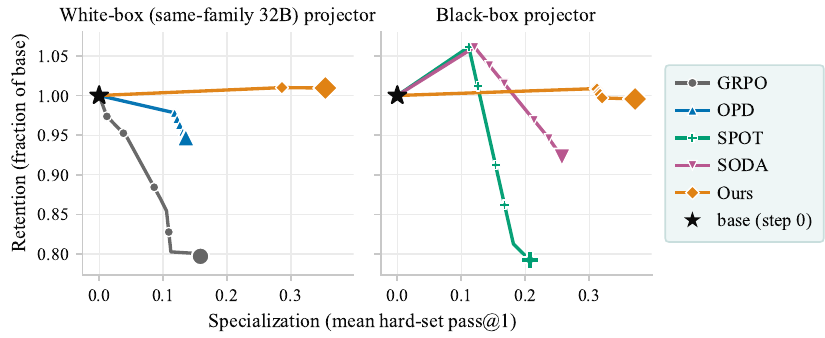}
  \caption{\textbf{Aggregate specialization--retention Pareto frontier}, one curve per method
  pooled over both students, all benchmarks and seeds, split by projector class. \textbf{Left:} same-family ($32$B) projectors. \textbf{Right:} black-box
  projector. GRPO uses no projector. $y$-axis is retention measured as held-out accuracy as a fraction of the untrained base. $x$-axis is specialization measured as average hard-set accuracy. Each curve is the monotone Pareto frontier of the pooled
  trajectory anchored at $\star=(0,1.0)$;
  \emph{up-and-to-the-right is better}. 
  \textbf{Every baseline ends below its untrained base}: GRPO and OPD decline from the outset,
  while SPOT and SODA exceed base early and then fall below it. Projected distillation holds
  retention at base across the range and attains the highest specialization.}
  \label{fig:pareto-all}
\end{figure}

\begin{table}[t]
  \centering
  \caption{\textbf{RQ1: area under the specialization--retention frontier (AUPF).} Higher is
  better (\Cref{app:rq1-full}); mean\,$\pm$\,s.e.\ over $3$ seeds. GRPO uses no projector and is
  shown for reference. The OLMo-2-7B LiveCodeBench cell is omitted, since no method moves
  specialization appreciably above zero there. \textbf{Bold} $=$ best per row within each projector
  block. Ours attains the highest AUPF on $8$ of $9$ cells with the same-family teacher and $7$ of
  $9$ with the black-box projector.}
  \label{tab:rq1-aupf}
  \footnotesize\setlength{\tabcolsep}{4.2pt}
  \begin{tabular}{@{}ll c cc ccc@{}}
    \toprule
    & & RLVR & \multicolumn{2}{c}{same-family $32$B teacher} & \multicolumn{3}{c}{black-box projector} \\
    \cmidrule(lr){3-3}\cmidrule(lr){4-5}\cmidrule(lr){6-8}
    & & GRPO & OPD & Ours & SPOT & SODA & Ours \\
    \midrule
    \multirow{5}{*}{\rotatebox[origin=c]{90}{\tiny{Qwen3-8B}}}
     & GSM8K   & 0.061{\tiny$\pm$.011} & 0.145{\tiny$\pm$.006} & \textbf{0.205{\tiny$\pm$.012}} & 0.209{\tiny$\pm$.030} & 0.231{\tiny$\pm$.016} & \textbf{0.262{\tiny$\pm$.020}} \\
     & MATH    & 0.000{\tiny$\pm$.000} & 0.003{\tiny$\pm$.001} & \textbf{0.008{\tiny$\pm$.002}} & 0.145{\tiny$\pm$.010} & \textbf{0.159{\tiny$\pm$.009}} & 0.035{\tiny$\pm$.010} \\
     & LCB     & 0.000{\tiny$\pm$.000} & 0.068{\tiny$\pm$.018} & \textbf{0.138{\tiny$\pm$.003}} & 0.095{\tiny$\pm$.005} & 0.091{\tiny$\pm$.001} & \textbf{0.123{\tiny$\pm$.006}} \\
     & SciKnow & 0.045{\tiny$\pm$.008} & 0.027{\tiny$\pm$.002} & \textbf{0.179{\tiny$\pm$.005}} & 0.124{\tiny$\pm$.008} & \textbf{0.155{\tiny$\pm$.013}} & 0.124{\tiny$\pm$.005} \\
     & Tool    & 0.248{\tiny$\pm$.152} & 0.128{\tiny$\pm$.004} & \textbf{0.676{\tiny$\pm$.013}} & 0.594{\tiny$\pm$.010} & 0.635{\tiny$\pm$.014} & \textbf{0.667{\tiny$\pm$.015}} \\
    \midrule
    \multirow{4}{*}{\rotatebox[origin=c]{90}{\tiny{OLMo-2-7B}}}
     & GSM8K   & 0.195{\tiny$\pm$.084} & \textbf{0.417{\tiny$\pm$.011}} & 0.217{\tiny$\pm$.014} & 0.186{\tiny$\pm$.007} & 0.210{\tiny$\pm$.004} & \textbf{0.336{\tiny$\pm$.067}} \\
     & MATH    & 0.009{\tiny$\pm$.003} & 0.014{\tiny$\pm$.002} & \textbf{0.017{\tiny$\pm$.002}} & 0.013{\tiny$\pm$.001} & 0.015{\tiny$\pm$.002} & \textbf{0.016{\tiny$\pm$.001}} \\
     & SciKnow & 0.198{\tiny$\pm$.022} & 0.060{\tiny$\pm$.013} & \textbf{0.236{\tiny$\pm$.010}} & 0.116{\tiny$\pm$.011} & 0.101{\tiny$\pm$.001} & \textbf{0.234{\tiny$\pm$.015}} \\
     & Tool    & 0.171{\tiny$\pm$.111} & 0.319{\tiny$\pm$.038} & \textbf{0.626{\tiny$\pm$.018}} & 0.222{\tiny$\pm$.031} & 0.488{\tiny$\pm$.029} & \textbf{0.682{\tiny$\pm$.038}} \\
    \bottomrule
  \end{tabular}
\end{table}

Projected distillation attains the highest AUPF on eight of nine cells against GRPO and OPD and on
seven of nine against SPOT and SODA (\Cref{tab:rq1-aupf}), under both projector settings. The
margins are largest on SciKnowEval and ToolAlpaca, reaching $0.179$ against OPD's $0.027$ on Qwen
SciKnowEval and $0.676$ against $0.128$ on Qwen ToolAlpaca. The baselines are not short of
specialization; SPOT and SODA both reach real hard-set accuracy (\Cref{tab:rq1-cl}). However, they retain poorly. In \Cref{fig:pareto-all} our frontier runs flat along the untrained base
across the whole specialization range while every baseline's falls away from it. Pooled across
projectors, our mean AUPF is $0.256$ with the mid-capacity same-family $32$B teacher and $0.275$
with the black-box projector, so the mid-capacity projector performs slightly worse. In
\Cref{sec:rq5-teacher}, we investigate how to improve learning with such a teacher. 

The two baseline families fall short for different reasons. GRPO has little to learn from, since
the hard set is mined at $\mathrm{pass}@k{=}0$; continued sampling still recovers occasional
correct rollouts, enough for the highest specialization of any method on OLMo SciKnowEval, but its
retention is the lowest there as well (\Cref{tab:rq1-main}), because reinforcing a rare correct
rollout sharpens the student onto that solution at the expense of what it produces elsewhere~\citep{yue2025rlvr}. SPOT and SODA hold correct targets throughout,
so they are not limited by correctness. They are rather limited by how far
the correction moves the attempt and how often it is rebuilt, as shown further in sections~\ref{sec:rq2-onpolicyness} and~\ref{sec:rq3-refresh}. Three AUPF result cells are not in our favour. On Qwen MATH, SPOT and SODA fit the hard-set answers more
closely, SODA reaching $0.136$ specialization against our $0.031$, and give up some retention for
it ($0.664$ against our $0.760$; \Cref{tab:rq1-main}). On Qwen SciKnowEval SODA leads on both
axes, in the one cell whose retention carries a standard error of $0.131$ across seeds, several
times that of any other method there. On OLMo GSM8K, OPD reaches $0.496$ specialization against
our $0.236$ at comparable retention, so token-level feedback apparently keeps its effectiveness
there.

Training on a hard set risks catastrophic forgetting due to limited model capacity, so \Cref{tab:rq1-cl} reports four
continual-learning measures, each summarizing a whole run.
Training-averaged specialization and retention give the mean of each axis over the run: ours is
highest on specialization ($\bar x = 0.264$ and $0.276$ against SODA's $0.224$) and holds retention
at base ($\bar y_{\mathrm{norm}} = 1.000$ and $0.995$ against $0.776$), so what it gains is not paid
for out of retention during training. Backward transfer is the net change in retention
from the untrained base to the final checkpoint, and ours are the only values that are not significantly negative
($0.009$ and $-0.004$). Forgetting is the drop from a run's highest retention to its
final retention, so it registers ability lost during training even when the endpoint looks
acceptable: ours are $0.084$ and $0.107$ against $0.509$ for SPOT and $0.355$ for SODA, which
specialize early and lose retention through the rest of training (\Cref{fig:learning-curves}).
Retention-safe specialization, the most a run reaches at any checkpoint whose retention is within
$\varepsilon$ of base, is highest for both our variants at either tolerance. Retention here is measured on a held-out slice of the trained benchmark;
\Cref{app:math-full} evaluates the math-trained checkpoints on suites they never saw.

\begin{table}[t]
  \centering
  \setlength{\tabcolsep}{12pt}
  \small
  \caption{\textbf{RQ1: learning one benchmark's hard set, tested on other benchmarks'.} Qwen3-8B
  trained on the MATH hard set with each black-box method and GRPO. Sampled pass@1 over $16$
  generations per problem at $T{=}0.8$; pass@16 is the share of problems solved at least once in
  those $16$. Every method is run for one training seed, evaluated every $50$ updates, and entered
  at its own best value for each column, so the checkpoint is selected on the same problems the
  number is reported on. The rule is applied identically to the baselines and to ours.}
  \label{tab:transfer-main}
  \begin{tabular}{lcccc}
  \toprule
  & \multicolumn{2}{c}{MATH held-out, hard ($127$)} & \multicolumn{2}{c}{Other benchmarks, hard ($2{,}885$)} \\
  \cmidrule(lr){2-3}\cmidrule(lr){4-5}
  Method & pass@1 & pass@16 & pass@1 & pass@16 \\
  \midrule
  Base & 0.013 & 0.150 & 0.011 & 0.114 \\
  GRPO & 0.020 & 0.205 & 0.011 & 0.108 \\
  SPOT & 0.020 & 0.197 & 0.012 & 0.114 \\
  SODA & 0.031 & 0.291 & 0.016 & 0.127 \\
  \textbf{Ours} & \textbf{0.056} & \textbf{0.354} & \textbf{0.020} & \textbf{0.154} \\
  \bottomrule
  \end{tabular}
\end{table}

\noindent\textbf{Transfer to other benchmarks' hard problems.}
We test whether performance improvement on one benchmark's hard problems extends to hard problems from other benchmarks. For this experiment, we train Qwen3-8B with each black-box method and with GRPO on $501$ MATH problems that it
fails on all $16$ attempts. We evaluate it on $127$ held-out hard MATH
problems not seen during training, and on $2{,}885$ hard problems from OlympiadBench~\citep{he2024olympiadbench},
Omni-MATH~\citep{gao2024omnimath}, AIME 1983--2023 and DAPO-Math~\citep{yu2025dapo}. More details of the setup are provided in \Cref{app:transfer}.

We report the results in \Cref{tab:transfer-main}. Projected distillation shows stronger performance than the baselines in hard-to-hard generalization. On the held-out MATH problems it raises pass@1 from $0.013$ to $0.056$ and
pass@16 from $0.150$ to $0.354$, against $0.031$ and $0.291$ for the best baseline. On the four
other benchmarks it raises pass@1 from $0.011$ to $0.020$ ($+0.009\pm0.001$) and pass@16 from
$0.114$ to $0.154$, where the best baseline gains $0.005$ in pass@1 and SPOT and GRPO gain nothing.

What the student learns on one benchmark's hard problems therefore carries to hard problems it never
trained on.
In \Cref{app:transfer}, we provide a breakdown of the gain across individual benchmarks and an analysis of a projected distillation model trained on Omni-MATH.

\subsection{RQ2: How does edit size influence learning in projected distillation?}
\label{sec:rq2-onpolicyness} 
A correction may be correct and still lie far from what the student would have produced, and this
question concerns the cost of that distance. We construct four arms that hold correctness fixed and
vary only how much the correction rewrites the attempt: a minimal edit, two progressively larger
rewrites obtained by re-prompting the projector, and \emph{gold}, the dataset's GSM8K reference
solution, which discards the attempt altogether. Each arm trains a separate Qwen3-8B student from
the same base under identical hyperparameters, so only the target $y^+$ varies; the rejected side
remains the student's own wrong rollout $y^-$. The projector rather than the prompt settles how much
text is replaced, so we identify the arms by the edit size they realize and measure the distance a
correction imposes by the likelihood shift it demands,
$\Lambda=\log\pi_{\theta}(y^+\mid x)-\log\pi_{\theta}(y^-\mid x)$, reported as $|\Lambda|$.

\begin{table}[t]
  \centering
  \footnotesize
  \setlength{\tabcolsep}{3.6pt}
  \caption{\textbf{RQ1: continual-learning metrics}, pooled over students, benchmarks and
  seeds (mean\,$\pm$\,s.e.), with arrows marking the better direction. $\bar x$ and $\bar y_{\mathrm{norm}}$ are specialization and retention averaged over
  training, BWT is backward transfer against the base, $F$ the drop from a run's retention peak,
  and $\mathrm{SafeSpec}_\varepsilon$ the most specialization reached at any checkpoint whose
  retention is within $\varepsilon$ of base (\Cref{app:rq1-full}). \textbf{Bold} $=$ best, \underline{underline} $=$
  second best per column. Our variants specialize most while holding
  retention at base, where the baselines either barely specialize or forget.}
  \label{tab:rq1-cl}
  \begin{tabular}{lcccccc}
    \toprule
    Method & $\bar x$\,$\uparrow$ & $\bar y_{\mathrm{norm}}$\,$\uparrow$ & BWT\,$\uparrow$ & $F$\,$\downarrow$ & SafeSpec$_{2\%}$\,$\uparrow$ & SafeSpec$_{5\%}$\,$\uparrow$ \\
    \midrule
    GRPO           & 0.114{\scriptsize$\pm$0.034} & 0.826{\scriptsize$\pm$0.048} & $-0.203${\scriptsize$\pm$0.058} & 0.221{\scriptsize$\pm$0.057} & 0.035{\scriptsize$\pm$0.012} & 0.056{\scriptsize$\pm$0.022} \\
    OPD            & 0.118{\scriptsize$\pm$0.022} & 0.951{\scriptsize$\pm$0.031} & $-0.062${\scriptsize$\pm$0.043} & 0.136{\scriptsize$\pm$0.041} & 0.101{\scriptsize$\pm$0.019} & 0.135{\scriptsize$\pm$0.026} \\
    SPOT           & 0.184{\scriptsize$\pm$0.031} & 0.706{\scriptsize$\pm$0.074} & $-0.380${\scriptsize$\pm$0.082} & 0.509{\scriptsize$\pm$0.064} & 0.147{\scriptsize$\pm$0.020} & 0.149{\scriptsize$\pm$0.020} \\
    SODA           & 0.224{\scriptsize$\pm$0.031} & 0.776{\scriptsize$\pm$0.053} & $-0.266${\scriptsize$\pm$0.059} & 0.355{\scriptsize$\pm$0.044} & 0.140{\scriptsize$\pm$0.021} & 0.190{\scriptsize$\pm$0.022} \\
    \midrule
    Ours (32B)     & \underline{0.264}{\scriptsize$\pm$0.042} & \textbf{1.000}{\scriptsize$\pm$0.020} & \textbf{0.009}{\scriptsize$\pm$0.028} & \textbf{0.084}{\scriptsize$\pm$0.015} & \textbf{0.202}{\scriptsize$\pm$0.040} & \textbf{0.295}{\scriptsize$\pm$0.048} \\
    Ours (bb)      & \textbf{0.276}{\scriptsize$\pm$0.042} & \underline{0.995}{\scriptsize$\pm$0.024} & \underline{$-0.004$}{\scriptsize$\pm$0.034} & \underline{0.107}{\scriptsize$\pm$0.027} & \underline{0.160}{\scriptsize$\pm$0.037} & \underline{0.235}{\scriptsize$\pm$0.040} \\
    \bottomrule
  \end{tabular}
\end{table}

\begin{figure}[t]
  \centering
  \includegraphics[width=.75\linewidth]{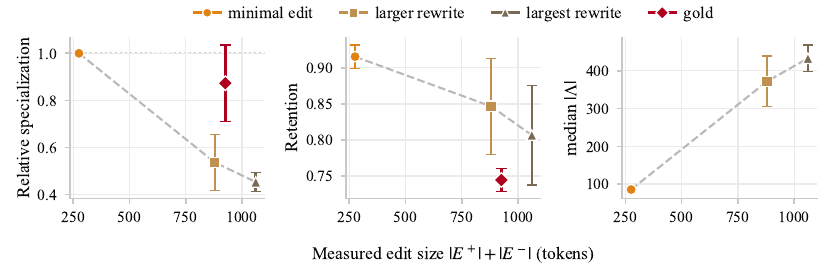}
  \caption{\textbf{RQ2: what a larger edit costs.} Qwen3-8B student per arm on the GSM8K hard
  set, varying only how much a verified-correct correction rewrites the student's own attempt
  (mean\,$\pm$\,s.e.\ over $3$ seeds). Arms sit at their \emph{measured} edit size, since the
  projector rather than the prompt sets it, and the dashed line joins the three matched arms.
  \textbf{Left to right:} hard-set pass@1 relative to the minimal-edit arm, absolute held-out
  pass@1, and the median likelihood shift $|\Lambda|$. Both behavioural axes fall as the edit grows while $|\Lambda|$ rises. \emph{Gold}
  appears on the first two panels only: it keeps none of the attempt, so its edit count understates
  how far it sits, and its $|\Lambda|$ of $3969$ is off the scale.}
  \label{fig:rq2}
\end{figure}

\Cref{fig:rq2} shows that among the three arms that differ only in edit
size, relative specialization falls to $0.54$ and then $0.45$ of the minimal-edit arm and retention from
$0.916$ to $0.807$, while $\Lambda$ falls from $-86$ to $-433$. The arms are equally correct, so what the larger rewrites cost
is distance from the student. \emph{Gold} sits outside that ordering: its edit is $926$ tokens,
between the two rewrites, yet it specializes better than either ($0.87$ of the minimal edit) and
returns the lowest retention of any arm ($0.744$). The token count understates its distance, since gold keeps none of
the attempt and counts small only because a reference solution is short; its likelihood shift does not
say the same: at $\Lambda=+3969$ it is the one arm the student finds \emph{more} likely than its own
attempt, the sign opposite to every rewrite, and that value is dominated by the length difference
rather than lying on the rewrites' scale. We therefore read edit size as the quantity
governing the trade-off among matched rewrites, with the minimal edit best balanced, and report
\emph{gold} on the behavioral axes only; \Cref{app:rq2-mechanism} gives the per-arm values and the
mechanism measurements.
\subsection{RQ3: How much does refreshing corrections online help?}
\label{sec:rq3-refresh}

Projected distillation recomputes its corrections at every round, so that each target is an edit
of an attempt the current student has just produced. Naturally, we ask what role this mechanism plays in projected distillation learning. For comparison, we gather a \emph{static
pool}: minimal-edit corrections of the initial checkpoint's responses, mined once before training
and reused for the remainder of the run. This is the same regime occupied by methods like SPOT~\citep{lin2026spot} and LeMa~\citep{an2023lema}. The projector for both static and online-refreshed arms is
the same black-box projector. We report results pooled over all model$\times$benchmark cells and three
seeds in figure~\ref{fig:rq3-agg}.
\begin{figure}[t]
  \centering
  \begin{minipage}[c]{0.5\linewidth}
    \centering
    \includegraphics[width=\linewidth]{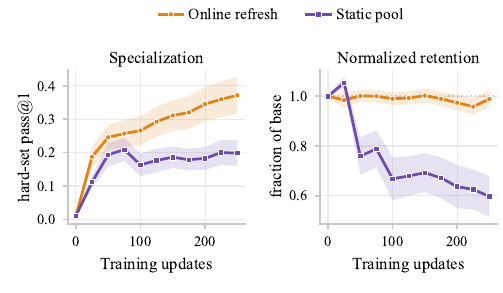}
  \end{minipage}\hfill
  \begin{minipage}[c]{0.48\linewidth}
    \captionof{figure}{\textbf{RQ3: online refresh against a static pool}, pooled over all
    model$\times$benchmark cells and $3$ seeds. Both arms take minimal edits from the same
    black-box projector; the static pool mines them once against the base checkpoint and reuses
    them for the run, the regime of mined-once correction methods. Refreshing holds retention at
    base throughout, while the pool climbs briefly above base and then falls to
    $0.60$.}
    \label{fig:rq3-agg}
  \end{minipage}
\end{figure}

Refreshing the corrections is essential. Online refresh improves specialization steadily throughout
training, reaching $0.37$ pass@1, whereas the static pool rises over the first few updates and then
flattens near $0.20$ (\Cref{fig:rq3-agg}, left). Retention separates more sharply: online refresh
holds it at base for the whole run, while the static pool climbs slightly above base over the first
$25$ steps, falls to $0.76$ by step $50$, and declines from there to $0.60$ (\Cref{fig:rq3-agg},
right). The fall sets in once the pool has been reused for several rounds. A static
pool cannot follow
the student to the failures it develops afterwards, so repeated passes fit it to a fixed set of
targets, while its generalization degrades.

\subsection{RQ4: How best to learn from projections of wrong responses?}
\label{sec:rq4-loss}
\begin{figure}[h]
  \centering
  \begin{minipage}[c]{0.4\linewidth}
    \centering
    \includegraphics[width=\linewidth]{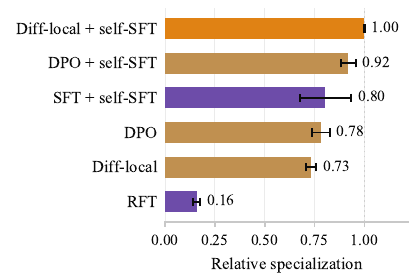}
  \end{minipage}
  \hfill
  \begin{minipage}[c]{0.58\linewidth}
    \captionof{figure}{\textbf{RQ4: loss recipes.} Specialization of each recipe as a fraction of
    diff-local DPO $+$ self-SFT, the recipe adopted in \Cref{sec:method-blackbox}. Qwen3-8B, GSM8K hard set, mean and stderr on $3$ seeds. The adopted recipe is
    \textcolor[RGB]{224,130,20}{orange}, the other contrastive arms
    \textcolor{contrastamber}{amber}, and the arms without a contrast
    \textcolor{covviolet}{violet}. Held-out retention stays between $0.88$ and $0.94$ for every
    arm, so the ordering reflects specialization and not differences in forgetting. SFT on the
    projection carries a wider spread across seeds.}
    \label{fig:rq4-loss}
  \end{minipage}
\end{figure}

In a typical round of projected distillation, we have access to wrong attempts $y^-$, their corrections $y^+$ and 
correct student responses $y^{\mathrm c}$. These three
quantities admit several learning objectives, and in this section we study an effective objective for learning. The
simplest is supervised fine-tuning on the projection, optionally combined with a term on the
student's own correct responses. Keeping only that second term is rejection fine-tuning
(\emph{RFT}; \citealp{yuan2023rft}), which learns from the student's successes alone and ignores the
projection. Further, because the wrong attempt is also retained, a contrastive
objective is available: \emph{DPO} \citep{rafailov2023dpo} over the full response, again
optionally combined with the self-SFT term. Since a minimal edit leaves most of the response
unchanged, the contrast may instead be localized to the edited spans (\eqref{eq:proj-dpo}). We
compare the six resulting recipes on the GSM8K hard set and report the outcome in \Cref{fig:rq4-loss}.

The projection is necessary: training on the student's own correct responses alone recovers $0.16$ of the
adopted recipe, since the hard set is mined at $\mathrm{pass}@k{=}0$ and the student's rollouts stay correct on
about $6\%$ of attempts through the run, so RFT has almost nothing to train on. Among objectives
that do use the projection, the contrast is the more reliable choice: SFT on the projection reaches
$0.81$ on average but carries a standard error of $0.078$ across seeds, three to sixteen times
that of any contrastive arm. The self-SFT term then settles the ordering among those arms. A
contrast constrains the likelihood of $y^+$ only relative to $y^-$, and a localized one does so
over a handful of tokens, so both gain from a term that raises the likelihood of correct text
outright, and the localized one gains more: without it the diff-localized contrast underperforms
the full-sequence one ($0.73$ against $0.78$), and with it the ordering reverses ($1.00$ against
$0.92$). We adopt the diff-localized contrast with self-SFT as the default recipe
(\Cref{sec:method-blackbox}); per-arm objectives and absolute values are in
\Cref{app:rq4-recipe}.
\subsection{RQ5: Can we improve the quality of projections for better learning?}
\label{sec:rq5-teacher}
\begin{figure}[t]
  \centering
  \begin{minipage}[c]{0.55\linewidth}
    \centering
    \includegraphics[width=\linewidth]{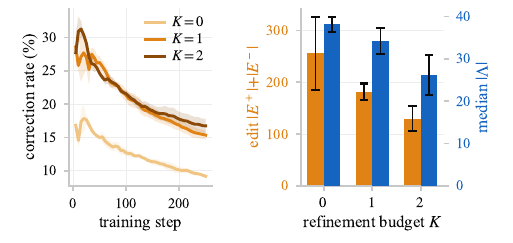}
  \end{minipage}\hfill
  \begin{minipage}[c]{0.43\linewidth}
    \captionof{figure}{\textbf{RQ5: refining the projector at test time.} A frozen Qwen3-32B
    projector iteratively refines its own correction and the Qwen3-8B student trains on the result
    (GSM8K hard set, $K\in\{0,1,2\}$, $3$ seeds), with the running correction rate on the left and
    edit size and off-policyness against the refinement budget on the right. More refinement corrects more wrong attempts
    with a smaller edit, and the projector is never finetuned.}
    \label{fig:rq5-refinement}
  \end{minipage}
\end{figure}

\looseness -1 In this section, we explore ways to make projections from a mid-capacity projector more effective. The
natural approach would be to fine-tune the projector; we instead take a computationally efficient test-time route and ask a
frozen Qwen3-32B to refine its own correction over several rounds. At each round it is re-prompted
with its previous correction and the spans it edited, and across rounds we retain the valid
correction most likely under the base student, measured by log-likelihood over the edited tokens.
We sweep the refinement budget over $K\in\{0,1,2\}$ on a Qwen3-8B student and the GSM8K hard set
with the default recipe.

Refinement improves the corrections in two ways (Fig.~\ref{fig:rq5-refinement}). More of the student's failures become
correctable, the correction rate rising from $9.2\%$ at $K{=}0$ to $15.4\%$ at $K{=}1$ and
$16.7\%$ at $K{=}2$ (\Cref{fig:rq5-refinement}, left), so a larger part of the hard set contributes a training
pair at all. At the same time the edits become smaller, from $256$ to $130$ tokens
(\Cref{fig:rq5-refinement}, right), which by the account of \Cref{sec:rq2-onpolicyness} should
make each target easier to learn from. Downstream, hard-set pass@1 rises from $0.37$ to $0.44$ while
retention stays near $0.91$--$0.93$, with most of the gain arriving in the first round. Test-time
refinement can therefore improve the projections, leading to effective performance. Additional details are in \Cref{app:rq5-refine}.

\subsection{RQ6: Can we use the same model for projection given privileged information?}
\label{sec:rq6-selfproj}

\begin{table}[t]
  \centering
  \setlength{\tabcolsep}{20pt}
  \normalsize
  \caption{\textbf{RQ6: a self-projector recovers about $40\%$ of the specialization.} Qwen3-8B on
  the GSM8K hard set, $250$ updates, mean $\pm$ s.e.\ over $3$ seeds. The self-projector is a frozen
  copy of the student's own base weights, given the verified answer. \emph{Gold} replaces the attempt
  with the dataset's reference solution. Further arms and the full setup are in
  \Cref{app:selfprojection}.}
  \label{tab:rq6-selfproj}
  \begin{tabular}{lcc}
  \toprule
  Projector & Hard set & Held-out \\
  \midrule
  none (untrained student) & 0.000 & $\approx$0.955 \\
  the student itself & 0.219{\scriptsize$\pm$0.012} & 0.924{\scriptsize$\pm$0.021} \\
  \emph{gold} solution & 0.489 & 0.744 \\
  black-box projector & \textbf{0.560} & 0.916 \\
  \bottomrule
  \end{tabular}
\end{table}

Use of the same model to provide the learning signal is a popular choice in post-training. One well-known example is on-policy self-distillation (OPSD)~\citep{zhao2026opsd}. Similarly, in this section, we test how much we can gain by replacing external projections with the model's own self-projections given the ground truth answer. We experiment with Qwen3-8B on the GSM8K dataset. The full setup is provided in \Cref{app:selfprojection}.

\Cref{tab:rq6-selfproj} reports the result. Self-projection does produce usable corrections: the
student lifts its hard set from $0.000$ to $0.219$ while retention stays at $0.924$, so a model that
is told the answer can repair enough of its own reasoning to learn from it. It does not, however,
substitute for a stronger projector. The black-box projector reaches $0.560$ on the same problems,
and even the student-agnostic \emph{gold} solution reaches $0.489$. Concurrent
work~\citep{karan2026sampling} proposes a self-projection scheme of its own, in which an MCMC
sampler has the student rewrite its response with a reference solution in view, and
\Cref{app:selfproj} compares it against black-box projection on its own math benchmark. 

\section{A Patch-MDP Analysis of Projected Distillation}
\label{sec:theory}
\noindent\textbf{Setup.}
We model generation as a deterministic \emph{patch MDP} with a token budget $H$ (\Cref{def:patch}).
A state is a prompt together with the tokens generated so far. An action is a \emph{patch}, a span
of tokens appended to that state, and every token carries a reward in $[0,1]$. A patch may span
several tokens, but the reward and every bound below decompose token by token
(\Cref{lem:patchadd}), so decisions are indexed by token position throughout and $H$ counts them. The student is a policy
$\pi$ in a class $\Pi$, with expected return $J(\pi)$ and optimum $J^\star$. Choosing action $w$ at
state $s$ forfeits some return against the best action available in $s$. We denote this \emph{optimality gap} with
$\Delta(s,w)$, which is upper-bounded by a constant $\Delta_{\max}$~(\Cref{def:gap}). Owing to
telescoping (\Cref{lem:patchadd}), a full rollout's shortfall can be accounted for one token at a
time. The regret $J^\star-J(\pi)$ then equals $H$ times the average per-token optimality gap, taken over $\pi$'s state occupancy (\Cref{lem:pdl}). We write $m$ for the \emph{edit size} (\Cref{def:edit}): the tokens
the correction changes relative to the attempt, counted over the $H$ decisions \Cref{def:patch}
gives both and so including the padding that equalizes their lengths. It relates to the token diff
$|E^+|+|E^-|$ of \Cref{sec:method}, the quantity \Cref{sec:rq2-onpolicyness} reports, by
$m=|E^+|+|E^-|+\big||E^+|-|E^-|\big|=2\max(|E^+|,|E^-|)$, the extra term being the change in answer
length. \Cref{app:setup,app:arms-proofs} give the
full treatment.\\
\noindent\textbf{Learning protocol.}
Training runs for $M$ \emph{rounds}. Each round is one batch of rollouts, the corrections made to
them, and the gradient update that follows. $\pi_i$ is the policy entering round $i$. Ideally we
would query the projector at every student mistake as the student generates a response. That would
let us minimize the gap under the student's own state occupancy. This is exactly the DAgger
\citep{ross2011dagger} approach, where an expert labels states under the learner's own occupancy,
leading to a learning regret of $O(H)$, linear in the horizon. Per-mistake labeling as an LLM
student responds is infeasible, since it demands a projector call and a forward pass at every
mistake. In projected distillation, we instead correct the rollout once and perform a single forward
pass over the response. This minimizes the optimality gap under a state occupancy drifted
away from the student's own. \\
\noindent\textbf{Fitting error.}
The projector edits the student's wrong response such that for each state in the rewritten response, the next token has zero optimality gap. We define two further quantities: (i) the policy fit error $\tilde\varepsilon$, the irreducible
error in fitting the best policy in the class to the expert projections; and (ii) the optimization
term $\tilde\gamma_M$, the student's excess disagreement over that best policy, which vanishes for a
no-regret learner (\Cref{def:regret}). Since a single disagreement costs at most $\Delta_{\max}$, the gap loss at these prefixes is bounded by
$\Delta_{\max}$ times the disagreement rate (\Cref{lem:reduction}).\\
We assume that the drift caused by earlier token edits at a given decision point has a smooth, bounded effect on the optimality gap. We denote the maximum per-token effect on the optimality gap by $\eta$.
\begin{assumption}[Prefix sensitivity]
\label{asmp:eta}
Write $e(s,s')$ for the number of insertions and deletions carrying the generated tokens of one
history into those of the other (\Cref{def:edit}). The map
$s\mapsto\mathbb{E}_{w\sim\pi(\cdot\mid s)}[\Delta(s,w)]$ is $\eta$-Lipschitz in that count:
\[
  \bigl|\mathbb{E}_{w\sim\pi(\cdot\mid s')}[\Delta(s',w)]-\mathbb{E}_{w\sim\pi(\cdot\mid s)}[\Delta(s,w)]\bigr|
  \;\le\;\eta\,e(s,s')\qquad\forall\,\pi\in\Pi,\ \forall\,s,s'.
\]
\end{assumption}

\begin{theorem}[Best-iterate suboptimality of projected distillation]
\label{thm:single}
Let \Cref{asmp:eta} hold, let the projector be as above, and let $m$ be the mean edit size over
the rounds. Then after $M$ rounds
\[
  \min_{i\le M}\big[J^\star-J(\pi_i)\big]
  \;\le\;H\,\Delta_{\max}\,\big(\tilde\varepsilon+\tilde\gamma_M\big)\;+\;\eta\,m\,H .
\]
\end{theorem}
The first term collects the fitting and optimization costs. The second is drift, which does not decrease with $M$; a smaller edit is the only lever on it.\\
\noindent\textbf{Horizon regimes.}
For a minimal edit the edit size is
$m=O(1)$, making the drift $O(\eta H)$ and hence the regret $O(H)$, linear in the horizon. A
complete rewrite of the response leads to $m=\Theta(H)$, making the drift $O(\eta H^{2})$ and
therefore the regret $O(H^{2})$, quadratic in the horizon (\Cref{cor:regimes}). In a controlled
experiment in \Cref{sec:rq2-onpolicyness} we vary $m$ (\Cref{def:edit}) and observe that the
student's performance suffers as the edit size increases. The two endpoints recover the rates of DAgger and of behavior cloning, though not those algorithms
themselves.

\section{Conclusion}
\label{sec:conclusion}

We introduced \textit{projected distillation}, a post-training recipe to enable lower-capacity LLMs to solve hard problems beyond their current reach. We demonstrated that standard on-policy post-training methods such as RLVR and OPD struggle to specialize on the hard task, while training on gold solutions leads to catastrophic forgetting. Our proposed approach avoids these issues by maintaining rich learning targets close to the student's current policy. This is done by iteratively transforming the student’s failed attempts into correct solutions through minimal edits.
Our extensive empirical analyses showed that our approach achieved a better specialization--retention Pareto frontier than previous approaches. We showed that it is important to learn from minimal edits rather than from larger edits or full rewrites, and from the mistakes of the latest checkpoint rather than only those of the base model. Theoretical analysis with a smoothness assumption on the effects of token edits described a spectrum where minimal edits lead to DAgger-like linear-in-horizon regret, while full rewrites lead to behavior-cloning-like quadratic regret. Thus, our work provides empirical and theoretical evidence on the utility of the underlying continuum of on-policyness for continual learning on hard tasks.

\noindent\textbf{Limitations and future work.} 
We instantiate the projection as a minimal token edit; 
exploring likelihood-based alternatives is also an important future direction. The uniform per-token sensitivity assumed in \Cref{asmp:eta} may be strong and needs further investigation. Our recipe assumes a cheap correctness check, leaving open the domains without a readily available verifier. We also study single-turn tasks, where the student produces one response and the projector edits that response. Multi-turn tasks, in which a failure can originate several turns before it becomes visible in the outcome, would require deciding which turn to project, and we leave this to future work.

\bibliographystyle{plainnat}
\bibliography{main}

\begin{thebibliography}{41}
\providecommand{\natexlab}[1]{#1}
\providecommand{\url}[1]{\texttt{#1}}
\expandafter\ifx\csname urlstyle\endcsname\relax
  \providecommand{\doi}[1]{doi: #1}\else
  \providecommand{\doi}{doi: \begingroup \urlstyle{rm}\Url}\fi

\bibitem[Agarwal et~al.(2023)Agarwal, Vieillard, Zhou, Stanczyk, Ramos, Geist, and Bachem]{agarwal2023gkd}
Rishabh Agarwal, Nino Vieillard, Yongchao Zhou, Piotr Stanczyk, Sabela Ramos, Matthieu Geist, and Olivier Bachem.
\newblock On-policy distillation of language models: Learning from self-generated mistakes.
\newblock \emph{International Conference on Learning Representations}, 2023.
\newblock URL \url{https://arxiv.org/abs/2306.13649}.

\bibitem[An et~al.(2023)An, Ma, Lin, Zheng, Lou, and Chen]{an2023lema}
Shengnan An, Zexiong Ma, Zeqi Lin, Nanning Zheng, Jian-Guang Lou, and Weizhu Chen.
\newblock Learning from mistakes makes {LLM} better reasoner.
\newblock \emph{arXiv preprint arXiv:2310.20689}, 2023.
\newblock URL \url{https://arxiv.org/abs/2310.20689}.

\bibitem[Chaudhry et~al.(2018)Chaudhry, Dokania, Ajanthan, and Torr]{chaudhry2018rwalk}
Arslan Chaudhry, Puneet~K. Dokania, Thalaiyasingam Ajanthan, and Philip H.~S. Torr.
\newblock Riemannian walk for incremental learning: Understanding forgetting and intransigence.
\newblock In \emph{European Conference on Computer Vision (ECCV)}, 2018.
\newblock URL \url{https://arxiv.org/abs/1801.10112}.

\bibitem[Chen et~al.(2021)Chen, Tworek, Jun, Yuan, de~Oliveira~Pinto, Kaplan, Edwards, Burda, Joseph, Brockman, Ray, Puri, Krueger, Petrov, Khlaaf, Sastry, Mishkin, Chan, Gray, Ryder, Pavlov, Power, Kaiser, Bavarian, Winter, Tillet, Such, Cummings, Plappert, Chantzis, Barnes, Herbert-Voss, Guss, Nichol, Paino, Tezak, Tang, Babuschkin, Balaji, Jain, Saunders, Hesse, Carr, Leike, Achiam, Misra, Morikawa, Radford, Knight, Brundage, Murati, Mayer, Welinder, McGrew, Amodei, McCandlish, Sutskever, and Zaremba]{chen2021codex}
Mark Chen, Jerry Tworek, Heewoo Jun, Qiming Yuan, Henrique~Ponde de~Oliveira~Pinto, Jared Kaplan, Harri Edwards, Yuri Burda, Nicholas Joseph, Greg Brockman, Alex Ray, Raul Puri, Gretchen Krueger, Michael Petrov, Heidy Khlaaf, Girish Sastry, Pamela Mishkin, Brooke Chan, Scott Gray, Nick Ryder, Mikhail Pavlov, Alethea Power, Lukasz Kaiser, Mohammad Bavarian, Clemens Winter, Philippe Tillet, Felipe~Petroski Such, Dave Cummings, Matthias Plappert, Fotios Chantzis, Elizabeth Barnes, Ariel Herbert-Voss, William~Hebgen Guss, Alex Nichol, Alex Paino, Nikolas Tezak, Jie Tang, Igor Babuschkin, Suchir Balaji, Shantanu Jain, William Saunders, Christopher Hesse, Andrew~N. Carr, Jan Leike, Josh Achiam, Vedant Misra, Evan Morikawa, Alec Radford, Matthew Knight, Miles Brundage, Mira Murati, Katie Mayer, Peter Welinder, Bob McGrew, Dario Amodei, Sam McCandlish, Ilya Sutskever, and Wojciech Zaremba.
\newblock Evaluating large language models trained on code.
\newblock \emph{arXiv preprint arXiv:2107.03374}, 2021.

\bibitem[Chen et~al.(2026)Chen, Wang, Zhu, Qiu, Dong, Deng, Sang, Wang, Geramifard, and Luo]{chen2026soda}
Xiwen Chen, Jingjing Wang, Wenhui Zhu, Peijie Qiu, Xuanzhao Dong, Yueyue Deng, Hejian Sang, Zhipeng Wang, Alborz Geramifard, and Feng Luo.
\newblock {SODA}: Semi on-policy black-box distillation for large language models.
\newblock \emph{arXiv preprint arXiv:2604.03873}, 2026.
\newblock URL \url{https://arxiv.org/abs/2604.03873}.

\bibitem[Cobbe et~al.(2021)Cobbe, Kosaraju, Bavarian, Chen, Jun, Kaiser, Plappert, Tworek, Hilton, Nakano, Hesse, and Schulman]{cobbe2021gsm8k}
Karl Cobbe, Vineet Kosaraju, Mohammad Bavarian, Mark Chen, Heewoo Jun, Lukasz Kaiser, Matthias Plappert, Jerry Tworek, Jacob Hilton, Reiichiro Nakano, Christopher Hesse, and John Schulman.
\newblock Training verifiers to solve math word problems.
\newblock \emph{arXiv preprint arXiv:2110.14168}, 2021.
\newblock URL \url{https://arxiv.org/abs/2110.14168}.

\bibitem[Feng et~al.(2024)Feng, Shen, Wang, Zhuang, Tang, Zhang, and Ding]{feng2024sciknoweval}
Kehua Feng, Xinyi Shen, Weijie Wang, Xiang Zhuang, Yuqi Tang, Qiang Zhang, and Keyan Ding.
\newblock {SciKnowEval}: Evaluating multi-level scientific knowledge of large language models.
\newblock \emph{arXiv preprint arXiv:2406.09098}, 2024.
\newblock URL \url{https://arxiv.org/abs/2406.09098}.

\bibitem[Gao et~al.(2024{\natexlab{a}})Gao, Song, Yang, Cai, Miao, Dong, Li, Ma, Chen, Xu, Tang, Wang, Zan, Quan, Zhang, Sha, Zhang, Ren, Liu, and Chang]{gao2024omnimath}
Bofei Gao, Feifan Song, Zhe Yang, Zefan Cai, Yibo Miao, Qingxiu Dong, Lei Li, Chenghao Ma, Liang Chen, Runxin Xu, Zhengyang Tang, Benyou Wang, Daoguang Zan, Shanghaoran Quan, Ge~Zhang, Lei Sha, Yichang Zhang, Xuancheng Ren, Tianyu Liu, and Baobao Chang.
\newblock {Omni-MATH}: A universal olympiad level mathematic benchmark for large language models.
\newblock \emph{arXiv preprint arXiv:2410.07985}, 2024{\natexlab{a}}.
\newblock URL \url{https://arxiv.org/abs/2410.07985}.

\bibitem[Gao et~al.(2024{\natexlab{b}})Gao, Tow, et~al.]{gao2024lmeval}
Leo Gao, Jonathan Tow, et~al.
\newblock The language model evaluation harness, 07 2024{\natexlab{b}}.
\newblock URL \url{https://zenodo.org/records/12608602}.

\bibitem[He et~al.(2024)He, Luo, Bai, Hu, Thai, Shen, Hu, Han, Huang, Zhang, Liu, Qi, Liu, and Sun]{he2024olympiadbench}
Chaoqun He, Renjie Luo, Yuzhuo Bai, Shengding Hu, Zhen~Leng Thai, Junhao Shen, Jinyi Hu, Xu~Han, Yujie Huang, Yuxiang Zhang, Jie Liu, Lei Qi, Zhiyuan Liu, and Maosong Sun.
\newblock {OlympiadBench}: A challenging benchmark for promoting {AGI} with olympiad-level bilingual multimodal scientific problems.
\newblock In \emph{Proceedings of the 62nd Annual Meeting of the Association for Computational Linguistics}, 2024.
\newblock URL \url{https://arxiv.org/abs/2402.14008}.

\bibitem[Hendrycks et~al.(2021{\natexlab{a}})Hendrycks, Burns, Basart, Zou, Mazeika, Song, and Steinhardt]{hendrycks2021mmlu}
Dan Hendrycks, Collin Burns, Steven Basart, Andy Zou, Mantas Mazeika, Dawn Song, and Jacob Steinhardt.
\newblock Measuring massive multitask language understanding.
\newblock In \emph{International Conference on Learning Representations}, 2021{\natexlab{a}}.
\newblock URL \url{https://arxiv.org/abs/2009.03300}.

\bibitem[Hendrycks et~al.(2021{\natexlab{b}})Hendrycks, Burns, Kadavath, Arora, Basart, Tang, Song, and Steinhardt]{hendrycks2021math}
Dan Hendrycks, Collin Burns, Saurav Kadavath, Akul Arora, Steven Basart, Eric Tang, Dawn Song, and Jacob Steinhardt.
\newblock Measuring mathematical problem solving with the {MATH} dataset.
\newblock In \emph{Advances in Neural Information Processing Systems (NeurIPS) Datasets and Benchmarks Track}, 2021{\natexlab{b}}.
\newblock URL \url{https://arxiv.org/abs/2103.03874}.

\bibitem[Hirschberg(1975)]{hirschberg1975lcs}
D.~S. Hirschberg.
\newblock A linear space algorithm for computing maximal common subsequences.
\newblock \emph{Communications of the ACM}, 18\penalty0 (6):\penalty0 341--343, 1975.
\newblock URL \url{https://doi.org/10.1145/360825.360861}.

\bibitem[Jain et~al.(2024)Jain, Han, Gu, Li, Yan, Zhang, Wang, Solar-Lezama, Sen, and Stoica]{jain2024livecodebench}
Naman Jain, King Han, Alex Gu, Wen-Ding Li, Fanjia Yan, Tianjun Zhang, Sida Wang, Armando Solar-Lezama, Koushik Sen, and Ion Stoica.
\newblock {LiveCodeBench}: Holistic and contamination free evaluation of large language models for code.
\newblock \emph{arXiv preprint arXiv:2403.07974}, 2024.
\newblock URL \url{https://arxiv.org/abs/2403.07974}.

\bibitem[Karan et~al.(2026)Karan, Chen, and Du]{karan2026sampling}
Aayush Karan, Sitan Chen, and Yilun Du.
\newblock Finetuning with sampling: {SFT} learns better than you think.
\newblock \emph{arXiv preprint arXiv:2610.02140}, 2026.
\newblock URL \url{https://arxiv.org/abs/2610.02140}.

\bibitem[Kemker et~al.(2018)Kemker, McClure, Abitino, Hayes, and Kanan]{kemker2018measuring}
Ronald Kemker, Marc McClure, Angelina Abitino, Tyler Hayes, and Christopher Kanan.
\newblock Measuring catastrophic forgetting in neural networks.
\newblock In \emph{AAAI Conference on Artificial Intelligence (AAAI)}, 2018.
\newblock URL \url{https://arxiv.org/abs/1708.02072}.

\bibitem[Kirkpatrick et~al.(2017)Kirkpatrick, Pascanu, Rabinowitz, Veness, Desjardins, Rusu, Milan, Quan, Ramalho, Grabska-Barwinska, Hassabis, Clopath, Kumaran, and Hadsell]{kirkpatrick2017ewc}
James Kirkpatrick, Razvan Pascanu, Neil Rabinowitz, Joel Veness, Guillaume Desjardins, Andrei~A. Rusu, Kieran Milan, John Quan, Tiago Ramalho, Agnieszka Grabska-Barwinska, Demis Hassabis, Claudia Clopath, Dharshan Kumaran, and Raia Hadsell.
\newblock Overcoming catastrophic forgetting in neural networks.
\newblock \emph{Proceedings of the National Academy of Sciences (PNAS)}, 114\penalty0 (13):\penalty0 3521--3526, 2017.
\newblock URL \url{https://arxiv.org/abs/1612.00796}.

\bibitem[Kumar et~al.(2022)Kumar, Raghunathan, Jones, Ma, and Liang]{kumar2022finetuning}
Ananya Kumar, Aditi Raghunathan, Robbie Jones, Tengyu Ma, and Percy Liang.
\newblock Fine-tuning can distort pretrained features and underperform out-of-distribution.
\newblock In \emph{International Conference on Learning Representations (ICLR)}, 2022.
\newblock URL \url{https://arxiv.org/abs/2202.10054}.

\bibitem[Kwon et~al.(2023)Kwon, Li, Zhuang, Sheng, Zheng, Yu, Gonzalez, Zhang, and Stoica]{kwon2023vllm}
Woosuk Kwon, Zhuohan Li, Siyuan Zhuang, Ying Sheng, Lianmin Zheng, Cody~Hao Yu, Joseph~E. Gonzalez, Hao Zhang, and Ion Stoica.
\newblock Efficient memory management for large language model serving with {PagedAttention}.
\newblock In \emph{Proceedings of the 29th Symposium on Operating Systems Principles}, 2023.
\newblock URL \url{https://arxiv.org/abs/2309.06180}.

\bibitem[Lightman et~al.(2023)Lightman, Kosaraju, Burda, Edwards, Baker, Lee, Leike, Schulman, Sutskever, and Cobbe]{lightman2023verify}
Hunter Lightman, Vineet Kosaraju, Yura Burda, Harri Edwards, Bowen Baker, Teddy Lee, Jan Leike, John Schulman, Ilya Sutskever, and Karl Cobbe.
\newblock Let's verify step by step.
\newblock \emph{arXiv preprint arXiv:2305.20050}, 2023.
\newblock URL \url{https://arxiv.org/abs/2305.20050}.

\bibitem[Lin and Han(2026)]{lin2026spot}
Wenye Lin and Kai Han.
\newblock Surgical post-training: Proximal on-policy distillation for reasoning with knowledge retention.
\newblock \emph{arXiv preprint arXiv:2603.01683}, 2026.
\newblock URL \url{https://arxiv.org/abs/2603.01683}.

\bibitem[Liu et~al.(2025)Liu, Farina, and Ozdaglar]{liu2025uft}
Mingyang Liu, Gabriele Farina, and Asuman Ozdaglar.
\newblock {UFT}: Unifying supervised and reinforcement fine-tuning.
\newblock \emph{arXiv preprint arXiv:2505.16984}, 2025.
\newblock URL \url{https://arxiv.org/abs/2505.16984}.

\bibitem[Lopez-Paz and Ranzato(2017)]{lopezpaz2017gem}
David Lopez-Paz and Marc'Aurelio Ranzato.
\newblock Gradient episodic memory for continual learning.
\newblock In \emph{Advances in Neural Information Processing Systems (NeurIPS)}, 2017.
\newblock URL \url{https://arxiv.org/abs/1706.08840}.

\bibitem[{OLMo Team}(2024)]{olmo2_2024}
{OLMo Team}.
\newblock 2 {OLMo} 2 furious.
\newblock \emph{arXiv preprint arXiv:2501.00656}, 2024.
\newblock URL \url{https://arxiv.org/abs/2501.00656}.

\bibitem[{Qwen Team}(2025)]{qwen2025qwen3}
{Qwen Team}.
\newblock Qwen3 technical report.
\newblock \emph{arXiv preprint arXiv:2505.09388}, 2025.
\newblock URL \url{https://arxiv.org/abs/2505.09388}.

\bibitem[Rafailov et~al.(2023)Rafailov, Sharma, Mitchell, Ermon, Manning, and Finn]{rafailov2023dpo}
Rafael Rafailov, Archit Sharma, Eric Mitchell, Stefano Ermon, Christopher~D. Manning, and Chelsea Finn.
\newblock Direct preference optimization: Your language model is secretly a reward model.
\newblock \emph{Advances in Neural Information Processing Systems}, 2023.
\newblock URL \url{https://arxiv.org/abs/2305.18290}.

\bibitem[Ross et~al.(2011)Ross, Gordon, and Bagnell]{ross2011dagger}
St\'ephane Ross, Geoffrey~J. Gordon, and J.~Andrew Bagnell.
\newblock A reduction of imitation learning and structured prediction to no-regret online learning.
\newblock In \emph{Proceedings of the 14th International Conference on Artificial Intelligence and Statistics (AISTATS)}, 2011.
\newblock URL \url{https://arxiv.org/abs/1011.0686}.

\bibitem[Schulman et~al.(2017)Schulman, Wolski, Dhariwal, Radford, and Klimov]{schulman2017ppo}
John Schulman, Filip Wolski, Prafulla Dhariwal, Alec Radford, and Oleg Klimov.
\newblock Proximal policy optimization algorithms.
\newblock \emph{arXiv preprint arXiv:1707.06347}, 2017.
\newblock URL \url{https://arxiv.org/abs/1707.06347}.

\bibitem[Shao et~al.(2024)Shao, Wang, Zhu, Xu, Song, Bi, Zhang, Zhang, Li, Wu, and Guo]{shao2024deepseekmath}
Zhihong Shao, Peiyi Wang, Qihao Zhu, Runxin Xu, Junxiao Song, Xiao Bi, Haowei Zhang, Mingchuan Zhang, Y.~K. Li, Y.~Wu, and Daya Guo.
\newblock {DeepSeekMath}: Pushing the limits of mathematical reasoning in open language models.
\newblock \emph{arXiv preprint arXiv:2402.03300}, 2024.
\newblock URL \url{https://arxiv.org/abs/2402.03300}.

\bibitem[Sheng et~al.(2024)Sheng, Zhang, Ye, Wu, Zhang, Zhang, Peng, Lin, and Wu]{sheng2024verl}
Guangming Sheng, Chi Zhang, Zilingfeng Ye, Xibin Wu, Wang Zhang, Ru~Zhang, Yanghua Peng, Haibin Lin, and Chuan Wu.
\newblock {HybridFlow}: A flexible and efficient {RLHF} framework.
\newblock \emph{arXiv preprint arXiv:2409.19256}, 2024.
\newblock URL \url{https://arxiv.org/abs/2409.19256}.

\bibitem[Tang et~al.(2023)Tang, Deng, Lin, Han, Liang, Cao, and Sun]{tang2023toolalpaca}
Qiaoyu Tang, Ziliang Deng, Hongyu Lin, Xianpei Han, Qiao Liang, Boxi Cao, and Le~Sun.
\newblock {ToolAlpaca}: Generalized tool learning for language models with 3000 simulated cases.
\newblock \emph{arXiv preprint arXiv:2306.05301}, 2023.
\newblock URL \url{https://arxiv.org/abs/2306.05301}.

\bibitem[von Werra et~al.(2020)von Werra, Belkada, Tunstall, Beeching, Thrush, Lambert, Huang, Rasul, and Gallou{\'e}dec]{vonwerra2020trl}
Leandro von Werra, Younes Belkada, Lewis Tunstall, Edward Beeching, Tristan Thrush, Nathan Lambert, Shengyi Huang, Kashif Rasul, and Quentin Gallou{\'e}dec.
\newblock {TRL}: Transformers reinforcement learning.
\newblock \url{https://github.com/huggingface/trl}, 2020.

\bibitem[Wagner and Fischer(1974)]{wagner1974string}
Robert~A. Wagner and Michael~J. Fischer.
\newblock The string-to-string correction problem.
\newblock \emph{Journal of the ACM}, 21\penalty0 (1):\penalty0 168--173, 1974.
\newblock URL \url{https://doi.org/10.1145/321796.321811}.

\bibitem[Wei et~al.(2022)Wei, Bosma, Zhao, Guu, Yu, Lester, Du, Dai, and Le]{wei2022flan}
Jason Wei, Maarten Bosma, Vincent~Y. Zhao, Kelvin Guu, Adams~Wei Yu, Brian Lester, Nan Du, Andrew~M. Dai, and Quoc~V. Le.
\newblock Finetuned language models are zero-shot learners.
\newblock In \emph{International Conference on Learning Representations (ICLR)}, 2022.
\newblock URL \url{https://arxiv.org/abs/2109.01652}.

\bibitem[Yang et~al.(2025)Yang, Yu, Zhang, Xu, Gonzalez, Cui, and Yan]{yang2025supercorrect}
Ling Yang, Zhaochen Yu, Tianjun Zhang, Minkai Xu, Joseph~E Gonzalez, Bin Cui, and Shuicheng Yan.
\newblock {SuperCorrect}: Advancing small {LLM} reasoning with thought template distillation and self-correction.
\newblock In \emph{International Conference on Learning Representations (ICLR)}, 2025.
\newblock URL \url{https://arxiv.org/abs/2410.09008}.

\bibitem[Yang et~al.(2024)Yang, Pang, Feng, Wang, Chen, Zhu, and Liu]{yang2024sdft}
Zhaorui Yang, Tianyu Pang, Haozhe Feng, Han Wang, Wei Chen, Minfeng Zhu, and Qian Liu.
\newblock Self-distillation bridges distribution gap in language model fine-tuning.
\newblock In \emph{Proceedings of the 62nd Annual Meeting of the Association for Computational Linguistics (ACL)}, 2024.
\newblock URL \url{https://arxiv.org/abs/2402.13669}.

\bibitem[Yu et~al.(2025)Yu, Zhang, Zhu, Yuan, Zuo, Yue, Dai, Fan, Liu, Liu, Liu, Lin, Lin, Ma, Sheng, Tong, Zhang, Zhang, Zhang, Zhu, Zhu, Chen, Chen, Wang, Yu, Song, Wei, Zhou, Liu, Ma, Zhang, Yan, Qiao, Wu, and Wang]{yu2025dapo}
Qiying Yu, Zheng Zhang, Ruofei Zhu, Yufeng Yuan, Xiaochen Zuo, Yu~Yue, Weinan Dai, Tiantian Fan, Gaohong Liu, Lingjun Liu, Xin Liu, Haibin Lin, Zhiqi Lin, Bole Ma, Guangming Sheng, Yuxuan Tong, Chi Zhang, Mofan Zhang, Wang Zhang, Hang Zhu, Jinhua Zhu, Jiaze Chen, Jiangjie Chen, Chengyi Wang, Hongli Yu, Yuxuan Song, Xiangpeng Wei, Hao Zhou, Jingjing Liu, Wei-Ying Ma, Ya-Qin Zhang, Lin Yan, Mu~Qiao, Yonghui Wu, and Mingxuan Wang.
\newblock {DAPO}: An open-source {LLM} reinforcement learning system at scale.
\newblock \emph{arXiv preprint arXiv:2503.14476}, 2025.
\newblock URL \url{https://arxiv.org/abs/2503.14476}.

\bibitem[Yuan et~al.(2023)Yuan, Yuan, Li, Dong, Lu, Tan, Zhou, and Zhou]{yuan2023rft}
Zheng Yuan, Hongyi Yuan, Chengpeng Li, Guanting Dong, Keming Lu, Chuanqi Tan, Chang Zhou, and Jingren Zhou.
\newblock Scaling relationship on learning mathematical reasoning with large language models.
\newblock \emph{arXiv preprint arXiv:2308.01825}, 2023.
\newblock URL \url{https://arxiv.org/abs/2308.01825}.

\bibitem[Yue et~al.(2025)Yue, Chen, Lu, Zhao, Wang, Song, and Huang]{yue2025rlvr}
Yang Yue, Zhiqi Chen, Rui Lu, Andrew Zhao, Zhaokai Wang, Shiji Song, and Gao Huang.
\newblock Does reinforcement learning really incentivize reasoning capacity in {LLMs} beyond the base model?
\newblock \emph{arXiv preprint arXiv:2504.13837}, 2025.
\newblock URL \url{https://arxiv.org/abs/2504.13837}.

\bibitem[Zhang et~al.(2026)Zhang, Chai, Fu, Tu, Wang, Lin, Yin, Zhang, Zhu, and Zhao]{zhang2026full}
Yaocheng Zhang, Jiajun Chai, Yuqian Fu, Songjun Tu, Xiaohan Wang, Wei Lin, Guojun Yin, Qichao Zhang, Yuanheng Zhu, and Dongbin Zhao.
\newblock Are full rollouts necessary for on-policy distillation?
\newblock \emph{arXiv preprint arXiv:2605.31490}, 2026.

\bibitem[Zhao et~al.(2026)Zhao, Xie, Liu, Huang, Pang, Chen, and Grover]{zhao2026opsd}
Siyan Zhao, Zhihui Xie, Mengchen Liu, Jing Huang, Guan Pang, Feiyu Chen, and Aditya Grover.
\newblock Self-distilled reasoner: On-policy self-distillation for large language models.
\newblock \emph{arXiv preprint arXiv:2601.18734}, 2026.
\newblock URL \url{https://arxiv.org/abs/2601.18734}.

\end{thebibliography}

\newpage

\appendix
\crefalias{section}{appendix}
\crefalias{subsection}{appendix}
\crefalias{subsubsection}{appendix}
\section{Notation}
\label{app:notation}

\Cref{tab:notation} consolidates the notation used throughout the paper, with each symbol's
meaning and where it is defined or used.

\begin{table}[H]
\centering
\caption{Notation reference.}
\label{tab:notation}
\footnotesize
\renewcommand{\arraystretch}{1.1}
\begin{tabular}{@{}ll@{}}
\toprule
\multicolumn{2}{@{}l}{\emph{Problem, student, verifier (\Cref{sec:prelim})}} \\
$x$ & problem / prompt (drawn from a problem distribution) \\
$y\in\mathcal V^{\le H}$ & answer: a token sequence of length $\le H$ \\
$\pi_{\theta_S}$ & student policy (autoregressive, parameters $\theta_S$) \\
$\pi_{\mathrm{ref}}=\pi_{\theta_{S,0}}$ & frozen initial student, used as the DPO reference \\
$R(x,y)\in\{0,1\}$ & verifier (marks an answer correct or not) \\
$\mathcal C(x)$ & set of correct answers $\{y:R(x,y)=1\}$ \\
$\mathrm{pass}@k{=}0$ & problems on which all $k$ sampled attempts are wrong \\
\midrule
\multicolumn{2}{@{}l}{\emph{Projection and objective (\Cref{def:projection}, \Cref{eq:proj-dpo})}} \\
$\Lambda$ & likelihood shift the correction demands of the student \\
$\proj(x)$ & projection: the most likely correct answer (\Cref{def:projection}) \\
$y^-=y$ & student's wrong attempt (rejected) \\
$y^+$ & the projector's minimal-edit correction (chosen) \\
$y^{\mathrm c}$ & student's own correct rollout (self-SFT target) \\
$E^-,\ E^+$ & mistake span (unlearn) and fix span (learn) from the token-diff \\
$\beta$ & DPO contrast strength (temperature) \\
$\mathcal L_{\mathrm{Diff\text{-}Local\text{-}DPO}},\ \mathcal L_{\mathrm{SFT}},\ \mathcal L$ & diff-local contrast; self-SFT; full objective \\
\midrule
\multicolumn{2}{@{}l}{\emph{Patch MDP (\Cref{def:patch}, \Cref{app:setup})}} \\
$\mathcal V,\ H$ & vocabulary; horizon, the token budget of a rollout \\
$s_t=(x,y_{1:t})$ & state before decision $t$: prompt and the first $t$ generated tokens \\
$w$ & action: a variable-length patch (a continuation that fits the budget) \\
$\mathcal A(s)$ & admissible actions $\{w\in\mathcal V^h:1\le h\le H-|s|\}$ \\
$p:\ s'=s\oplus w$ & deterministic transition (append the patch) \\
$r_{\mathrm{tok}}\in[0,1]$ & per-token reward; $r(s,w)$ sums it over the patch \\
$\rho$ & initial state distribution: the distribution over problems \\
$\gamma=1$ & no discounting \\
$\pi,\ \pi_\theta$ & policy; underlying token model \\
$V^\star,\ Q^\star$ & optimal value and action-value \\
$J(\pi),\ J^\star$ & expected return of $\pi$; its optimum, \eqref{eq:return} \\
$d^{\pi},\ \tilde d_i$ & occupancy induced by $\pi$; the corrected prefixes of round $i$ \\
$\Delta(s,w)$ & optimality gap $V^\star(s)-Q^\star(s,w)\in[0,\Delta_{\max}]$ \\
\midrule
\multicolumn{2}{@{}l}{\emph{The regret analysis (\Cref{sec:theory}, \Cref{thm:single})}} \\
$M,\ \pi_i$ & rounds of training; the learner entering round $i$ \\
$\Pi,\ \Theta$ & policy class and its parameter set (\Cref{app:setup}) \\
$w^\star(s)$ & the projector's token at $s$, assumed optimal there \\
$m$ & edit size of a correction (\Cref{def:edit}); in \Cref{thm:single}, its mean over rounds \\
$e(s,s')$ & edit count between two histories (\Cref{def:edit}) \\
$e_t,\ \bar e_i$ & edit count at decision $t$; its mean over round $i$ (\Cref{lem:transfer}) \\
$\eta$ & per-edited-token drift (\Cref{asmp:eta}) \\
$\Delta_{\max}$ & largest return one decision can forfeit (\Cref{def:gap}) \\
$L(d,\pi)$ & mean gap of $\pi$ under a state distribution $d$, \eqref{eq:loss} \\
$\tilde\varepsilon_i(\pi)$ & rate at which $\pi$ disagrees with $w^\star$ at those prefixes \\
$\tilde\varepsilon$ & approximation error: the disagreement of the best policy in the class \\
$\tilde\gamma_M$ & optimization error: the learner's excess disagreement over it \\
$D,\ G$ & parameter-set diameter; gradient bound (\Cref{prop:rate}) \\
\bottomrule
\end{tabular}
\end{table}

\section{Projection examples}
\label{app:proj-examples}

This appendix shows complete wrong$\to$projected pairs, so the minimal-edit property
can be read at the level of full solutions. In each pair the projector reuses the
student's derivation verbatim and changes only what is needed to make the answer
correct; the changed spans are the token-diff $E^-$ (mistake) and $E^+$ (fix) that the
diff-local DPO loss \eqref{eq:proj-dpo} contrasts. All student answers are from Qwen3-8B
(base); all projections are from the black-box projector. \Cref{app:ex-cindy} works through the
single-edit example of \Cref{fig:minedit-example} in full, and \Cref{app:ex-john} a multi-edit one.

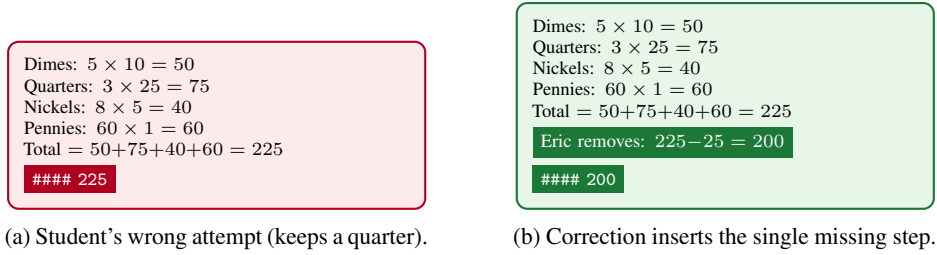
\begin{figure}[H]
  \centering
  \footnotesize
  \begin{subfigure}[t]{0.4\linewidth}
    \centering
    \begin{tikzpicture}
      \node[rectangle, rounded corners, draw=wrongfg, thick, fill=wrongbg,
            text width=\dimexpr\linewidth-14pt\relax, align=left, inner sep=6pt,
            font=\scriptsize] {%
        Dimes: $5\times10 = 50$\\
        Quarters: $3\times25 = 75$\\
        Nickels: $8\times5 = 40$\\
        Pennies: $60\times1 = 60$\\
        Total $= 50{+}75{+}40{+}60 = 225$\\[2pt]
        \colorbox{wrongfg}{\color{white}\texttt{\#\#\#\# 225}}};
    \end{tikzpicture}
    \caption{Student's wrong attempt (keeps a quarter).}
  \end{subfigure}
  \hspace{1cm}
  \begin{subfigure}[t]{0.4\linewidth}
    \centering
    \begin{tikzpicture}
      \node[rectangle, rounded corners, draw=fixfg, thick, fill=fixbg,
            text width=\dimexpr\linewidth-14pt\relax, align=left, inner sep=6pt,
            font=\scriptsize] {%
        Dimes: $5\times10 = 50$\\
        Quarters: $3\times25 = 75$\\
        Nickels: $8\times5 = 40$\\
        Pennies: $60\times1 = 60$\\
        Total $= 50{+}75{+}40{+}60 = 225$\\[2pt]
        \colorbox{fixfg}{\color{white}Eric removes: $225{-}25 = 200$}\\[2pt]
        \colorbox{fixfg}{\color{white}\texttt{\#\#\#\# 200}}};
    \end{tikzpicture}
    \caption{Correction inserts the single missing step.}
  \end{subfigure}
  \caption{\textbf{A minimal-edit projection} on GSM8K (\emph{Cindy's wishing pond}, gold answer
  $200$). The student's wrong attempt (a) and the projector's correction (b) are token-identical
  except for one inserted step and the final answer line, which changes from $225$ to $200$. Only the changed spans $E^-\!\to\!E^+$ enter the contrast in
  \eqref{eq:proj-dpo}.}
  \label{fig:minedit-example}
\end{figure}

\subsection{A single edit: Cindy's wishing pond}
\label{app:ex-cindy}

\textbf{Problem.} Cindy tosses 5 dimes into the wishing pond. Eric flips 3 quarters into
the pond. Garrick throws in 8 nickels. Ivy then drops 60 pennies in. \emph{If Eric dips
his hands into the water and pulls out a quarter}, how much money, in cents, did they put
into the pond? \textbf{(Gold: 200.)}

\medskip
\wrongbox{$y^-$ (rejected, wrong --- \#\#\#\# 225)}{%
  Dimes: $5\times10 = 50$. Quarters: $3\times25 = 75$. Nickels: $8\times5 = 40$.
  Pennies: $60\times1 = 60$. Total $= 50+75+40+60 = 225$. \texttt{\#\#\#\# 225}}

\medskip
\fixbox{$y^+$ (chosen, minimal edit --- \#\#\#\# 200)}{%
  Dimes: $5\times10 = 50$. Quarters: $3\times25 = 75$. Nickels: $8\times5 = 40$.
  Pennies: $60\times1 = 60$. Total $= 50+75+40+60 = 225$.
  \textcolor{fixfg}{\textbf{Eric removes a quarter: $225-25 = 200$.}} \texttt{\#\#\#\# 200}}

\medskip
The entire count is reused verbatim; the projector inserts the single step the student
omitted (subtracting the withdrawn quarter). The token-diff is a one-span edit:
$E^- = \{\texttt{\#\#\#\# 225}\}$, $E^+ = \{$``Eric removes\ldots $= 200$''$, \texttt{\#\#\#\# 200}\}$.

\subsection{Several edits: John's computer}
\label{app:ex-john}

\textbf{Problem.} John's new computer cost \$1500. He needs a monitor and other
peripherals that cost 1/5 that much. He also wants to upgrade the \$300 video card in the base computer
with one that costs \emph{twice as much}. How much did he spend? \textbf{(Gold: 2100.)}

\medskip
\wrongbox{$y^-$ (rejected, wrong --- \#\#\#\# 2400)}{%
  Computer: \$1500. Peripherals: $\tfrac15\times1500 = 300$. Upgraded video card:
  $2\times300 = 600$. Total: $1500+300+600 = 2400$. \texttt{\#\#\#\# 2400}}

\medskip
\fixbox{$y^+$ (chosen, minimal edit --- \#\#\#\# 2100)}{%
  Computer: \$1500. Peripherals: $\tfrac15\times1500 = 300$. Upgraded video card:
  $2\times300 = 600$. \textcolor{fixfg}{\textbf{Since the base computer already includes
  the \$300 card, he pays only the extra: $600-300 = 300$.}} Total:
  $1500+300+\textcolor{fixfg}{\textbf{300}} = 2100$. \texttt{\#\#\#\# 2100}}

\medskip
Here two spans change---the video-card line and the total---while the computer and
peripheral steps are reused unchanged. This is a multi-edit projection: the correction
threads the fix (charging only the upgrade difference) through to the final sum.

\section{Experimental setup}

\subsection{Reproducibility Details}
\label{app:exp-details}

\noindent\textbf{Model families and projectors.}
The two students are open-weight instruction-tuned models from different families,
\textbf{Qwen3-8B}~\citep{qwen2025qwen3} and \textbf{OLMo-2-7B}~\citep{olmo2_2024}, chosen so
the effect is tested as a property of the method rather than of one pretraining recipe. The
same-family projector shares the student's tokenizer, as OPD requires, and unless noted it is
\textbf{Qwen3-32B} for the Qwen student and \textbf{OLMo-2-32B} for the OLMo student. The
black-box path (\Cref{sec:method-blackbox}) needs only the correction text, so it is
tokenizer-free and draws its corrections from a single black-box model shared across students.
The exact checkpoints are:
students \texttt{Qwen/Qwen3-8B} and \texttt{allenai/OLMo-2-1124-7B-Instruct}; same-family
$32$B teachers \texttt{Qwen/Qwen3-32B} and \texttt{allenai/OLMo-2-0325-32B-Instruct}.

\noindent\textbf{Benchmarks.}
Five benchmarks spanning math, code, tools, and science, each contributing a mined hard set
for training and a held-out slice for evaluation:
\begin{itemize}[leftmargin=1.4em,itemsep=1pt,topsep=1pt]
  \item \textbf{GSM8K}~\citep{cobbe2021gsm8k}: grade-school math, numeric answers graded by
    exact match. Mine from the 7{,}473-example \texttt{train} split; evaluate on a $200$-problem
    held-out slice of the \texttt{test} split.
  \item \textbf{MATH}~\citep{hendrycks2021math}: competition math,
    \texttt{\textbackslash boxed\{\}} answers graded by symbolic equivalence. Mine from the
    12k \texttt{train} split; evaluate on the MATH-500 slice~\citep{lightman2023verify}.
  \item \textbf{ToolAlpaca}~\citep{tang2023toolalpaca}: tool use and API calling. Mine from the
    $468$-API, ${\sim}4{,}255$-instance release. The model must emit the first tool call for a
    request, and a call counts as correct when the action is a valid function of that API, its
    arguments parse as JSON with all required parameters present, and the action matches the
    reference call.
  \item \textbf{LiveCodeBench~v6}~\citep{jain2024livecodebench}: contamination-resistant
    competitive-programming problems, graded by unit-test execution (pass@1). Mine hard
    problems from the v6 release; evaluate on a held-out slice.
  \item \textbf{SciKnowEval}~\citep{feng2024sciknoweval}: multi-level scientific knowledge. We
    use the L3 four-choice subset ($4{,}135$ items) of the $28{,}392$-item release, graded by a
    letter match against the reference option.
\end{itemize}

\noindent\textbf{Evaluation protocol.} Specialization and retention are evaluated identically at
step $0$ and every $25$ steps thereafter, so each student and benchmark serves as its own
reference; the continual-learning summaries built from these two series are defined in
\Cref{app:rq1-full}. Every benchmark is scored by the grader listed above, always as
sampled pass@1 on its held-out slice, a $200$-problem slice for GSM8K. The
full-suite evaluations of \Cref{app:math-full} instead score the final checkpoints on complete
test splits. Answer extraction is per benchmark: the \texttt{\#\#\#\#}-delimited final
number for GSM8K, with a last-number fallback for free-form generations; the
\texttt{\textbackslash boxed\{\}} expression for MATH, compared by symbolic equivalence; the
final \texttt{Answer:} letter for SciKnowEval; the first \texttt{Action} and
\texttt{Action Input} pair for ToolAlpaca; and the extracted program for LiveCodeBench, executed
against the unit tests.

\noindent\textbf{Hardware and software.} The main experiments ran on a single pod of $8\times$
NVIDIA A100-SXM4 (40\,GB) GPUs, which was the compute available to us. The stack is PyTorch~2.8.0 (CUDA~12.8), TRL~1.6.0~\citep{vonwerra2020trl},
Transformers~5.12.1, and Accelerate~1.14.0. Multi-GPU training uses fully-sharded
data parallelism (FSDP2, full-shard) with gradient checkpointing. Final checkpoints are
gathered and consolidated from the sharded
(DTensor) state.

\noindent\textbf{Training configuration.} Unless noted, all methods train for $250$ update steps at a
constant learning rate with the Adam optimizer in bf16, computing loss on assistant response
tokens only. GRPO and OPD use $5\times10^{-6}$, and SPOT, SODA, and projected distillation use
$1\times10^{-5}$. The per-step generation budget is matched at $8$ rollouts: the
distillation methods (OPD, SPOT, SODA, and ours) draw $8$ problems per step, while GRPO draws a
group of $8$ samples on a single problem. The DPO-based objectives use $\beta{=}0.1$; SODA runs
a brief SFT warmup on $y^+$ (first $25$ steps) before its preference phase. GRPO keeps TRL's
default KL coefficient of zero toward the reference policy.

\subsection{Hard-Set Construction and Statistics}
\label{app:hardset}

\noindent\textbf{Mining rule.} Training targets problems the student cannot already solve. For each
candidate problem we draw $k{=}16$ samples from the \emph{untrained} student at temperature $0.8$
(seed $0$) and keep the problem only if all $16$ are graded wrong by that benchmark's grader, i.e.\
pass@$16 = 0$. The retained problems form the frozen hard set; the $16$ failed samples are stored
with each record, which is what allows an offline arm to be built without re-generating.

Because the criterion is defined against a specific student, \textbf{hard sets are per student}: the
same benchmark yields a different hard set for Qwen3-8B than for OLMo-2-7B, and a checkpoint is only
ever scored against the hard set mined for its own base model.

\Cref{tab:hardset-qwen} reports the mining statistics for Qwen3-8B and \Cref{tab:hardset-olmo} the
split sizes for OLMo-2-7B. Hard-set size is a property of the \emph{student}, not of the benchmark:
GSM8K is hard for only $1.63\%$ of Qwen3-8B's training pool, whereas MATH yields roughly twice as
many hard problems for OLMo-2-7B ($4{,}046$) as for Qwen3-8B ($1{,}991$).

\begin{table}[H]
  \centering
  \small
  \caption{\textbf{Hard-set mining for Qwen3-8B.} A problem enters the hard set when all $16$
  samples drawn from the untrained student at $T{=}0.8$ are wrong. ``Base sample accuracy'' is the
  fraction of all $16\times$(screened) samples that were correct, i.e.\ the student's accuracy over
  the pool before any training. The evaluation slice is disjoint from the hard set and comes from
  the benchmark's \texttt{test} split.}
  \label{tab:hardset-qwen}
  \begin{tabular}{lrrrrr}
  \toprule
  Benchmark & Screened & Hard set & Hard \% & Base sample acc. & Eval slice \\
  \midrule
  GSM8K         & 7{,}473  & 122    & 1.63  & 0.949 & 200 \\
  MATH          & 12{,}000 & 1{,}991 & 16.59 & 0.703 & 500 \\
  SciKnowEval   & 3{,}721  & 1{,}003 & 26.96 & 0.323 & 400 \\
  ToolAlpaca    & 3{,}584  & 467    & 13.03 & 0.845 & 300 \\
  LiveCodeBench & ---      & 100    & ---   & ---   & 50  \\
  \bottomrule
  \end{tabular}
\end{table}

\begin{table}[H]
  \centering
  \small
  \caption{\textbf{Hard-set sizes for OLMo-2-7B.} The mining summary files were not retained for
  these runs, so the screened pool size and base sample accuracy are unavailable; the split sizes
  are taken directly from the mined files.}
  \label{tab:hardset-olmo}
  \begin{tabular}{lrr}
  \toprule
  Benchmark & Hard set & Eval slice \\
  \midrule
  GSM8K         & 230    & 200 \\
  MATH          & 4{,}046 & 200 \\
  SciKnowEval   & 493    & 200 \\
  ToolAlpaca    & 471    & 200 \\
  LiveCodeBench & 810    & 131 \\
  \bottomrule
  \end{tabular}
\end{table}

\noindent\textbf{Train/evaluation disjointness.} We verify rather than assume separation. Comparing
question text lower-cased with \LaTeX{} commands stripped and non-alphanumeric characters removed
--- so that a problem reformatted between two source releases is still caught, which a byte-level
comparison would miss --- gives zero overlap between hard set and evaluation slice for both math
benchmarks: $1{,}991$ versus $500$ for MATH and $122$ versus $200$ for GSM8K. AMC23,
MMLU~\citep{hendrycks2021mmlu} and
HumanEval are evaluation-only and were never trained on by any method. Four problems appear
twice inside the MATH hard set ($1{,}991$ records, $1{,}987$ unique) and therefore carry double
weight during training.

\subsection{Prompt Templates}
\label{app:prompts}

All prompts are single user turns passed through the model's own chat template. For Qwen3 the
thinking mode is disabled in mining, training and evaluation alike, so the same decoding regime
holds throughout. The instruction appended to a question is selected by the benchmark's answer type,
which fixes the output format the grader expects; the grader always follows the \emph{evaluation}
benchmark, not the benchmark a checkpoint was trained on.

\subsubsection{Student prompts}

The student prompt is the question followed by one of the four instructions below.

\noindent\textbf{Numeric} (GSM8K):
\begin{quote}\small\ttfamily
Solve the problem by showing your full step-by-step reasoning. After the reasoning, write the final
answer on a new line by itself in the exact format "\#\#\#\# N", where N is just the final number
(no units, words, or commas). Always include this \#\#\#\# line as the last line.
\end{quote}

\noindent\textbf{Boxed} (MATH, MATH-500, AMC23):
\begin{quote}\small\ttfamily
Solve the problem by showing your full step-by-step reasoning. Put your final answer inside
\textbackslash boxed\{...\} at the end. The content of the box must be exactly the final answer (a
number or a simplified expression), with no surrounding words.
\end{quote}

\noindent\textbf{Multiple choice} (SciKnowEval, MMLU):
\begin{quote}\small\ttfamily
Answer the multiple-choice question. Think step by step, then write the final answer on a new line
by itself as "Answer: X" where X is A, B, C, or D.
\end{quote}

\noindent\textbf{Code} (LiveCodeBench, HumanEval):
\begin{quote}\small\ttfamily
Complete the Python function. Return the full function definition inside a \texttt{\textasciigrave\textasciigrave\textasciigrave python} code
block. Do not include tests or explanation.
\end{quote}

\subsubsection{Projection prompt}

The projector is instructed to \emph{edit} the student's wrong attempt rather than to solve the
problem afresh; this is what keeps the target near the student's own distribution and makes the
token-level diff small. The role is set per domain (``expert math tutor'' for numeric and boxed
answers, ``expert science tutor'' for multiple choice, ``expert assistant for tool/function
calling'' for tool use).

\noindent\textbf{System message:}
\begin{quote}\small\ttfamily
You are an \{role\}. A student answered a problem incorrectly. PROJECT their attempt onto a correct
solution: keep the student's overall approach, wording, and structure as much as possible, changing
ONLY what is needed to make it correct. Do NOT rewrite from scratch in your own style --- stay as
close to how the student writes as possible. \{end-format requirement\}
\end{quote}

\noindent\textbf{User message:}
\begin{quote}\small\ttfamily
Problem:\\
\{problem\}\\[2pt]
Student's incorrect attempt:\\
\{wrong attempt\}\\[2pt]
Minimally edit the student's attempt into a correct solution, staying as close to their style as
possible. \{end-format requirement\}
\end{quote}

The end-format requirement names the graded target exactly: \texttt{End with the correct final
answer, \{gold\}.} for math, the correct option for multiple choice, and the required
\texttt{Action} / \texttt{Action Input} pair for tool calling. A projection that does not end in the
graded format is \textbf{rejected by the verifier}, so a correction cannot be accepted on the
strength of its prose alone; every accepted projection has been re-graded by the same grader used
for evaluation.

\noindent\textbf{Code is different in kind.} For programming problems the correctness signal given
to the projector is \emph{execution feedback on the student's own attempt} --- which tests failed,
with input, expected and actual values --- and not the hidden test suite:
\begin{quote}\small\ttfamily
You are an expert competitive programmer helping a student. The student wrote a solution that FAILS
some tests. PROJECT their attempt onto a correct one: keep the student's overall approach,
structure, and variable names as much as possible, changing ONLY what is needed to make it pass. Do
NOT rewrite from scratch in your own style --- stay as close to how the student writes as possible.
Return the corrected solution as a single code block.
\end{quote}

\subsection{Training Configuration: Additional Details}
\label{app:train-extra}

\Cref{app:exp-details} gives the models, learning rates, optimizer, per-step generation budget and
hardware. This section records the details that a re-implementation additionally needs.

\noindent\textbf{Generation budget per domain, and why it varies.} The rollout budget must be at
least the budget the hard set was mined with. A problem is in the hard set precisely because the
student failed it under the mining budget, so giving the student a \emph{smaller} budget during
training truncates its derivation mid-proof and the problem cannot be solved even in principle.
Competition MATH attempts run a median of roughly $805$ tokens, with $61\%$ needing at least $690$,
so a $700$-token budget both pinned hard-set accuracy near zero and trained the model on truncated
derivations, collapsing held-out accuracy from $0.39$ to $0.16$. The experiments of
\Cref{sec:experiments} therefore use a uniform $2{,}048$-token rollout budget, across both students
and all five benchmarks, which is above the mining budget everywhere. Rollouts are sampled at
$T{=}0.8$. This is also why the mean deleted span reported for the \emph{gold} arm in
\Cref{tab:rq2-onpol} exceeds $700$ tokens. The transfer runs of \Cref{app:transfer} and the
replication of \Cref{app:selfproj} use their own budgets, stated in those sections.

\noindent\textbf{Projector configuration.} The black-box projector is Claude Opus 4.8, accessed
through a hosted API and run off-GPU, which is what allows the whole pod to be spent on the
student. Each wrong sample gets at most two projection attempts and the first verified-correct
projection is kept; if neither verifies, the sample contributes no preference pair for that step.
The same-family projector variant is served locally.

\noindent\textbf{GPU allocation.} Of the eight GPUs in the pod, seven hold FSDP2 shards of the
student and one is reserved for the in-loop vLLM~\citep{kwon2023vllm} generator that produces
on-policy rollouts and runs
the periodic evaluations. Sharing the generator GPU with a training rank starves the generator and
stalls the step, so the split is fixed rather than dynamic.

\noindent\textbf{Loss variants.} The arms and their objectives are given in
\Cref{app:rq4-recipe}. Two further switches are fixed throughout: the adopted default uses
unnormalized rather than token-mean-normalized DPO log-probabilities, and the auxiliary SFT term
acts on the student's own correct rollouts $y^{\mathrm c}$, as in \eqref{eq:proj-objective}, not on
the projection $y^+$.

\subsection{Evaluation: Statistical Resolution and Seed Variance}
\label{app:eval-resolution}

\noindent\textbf{Two kinds of seed.} Two sources of variation are easy to conflate and differ by
roughly a factor of fifty in magnitude.

\begin{itemize}[leftmargin=1.4em,itemsep=1pt,topsep=1pt]
  \item A \textbf{training seed} indexes an independently trained checkpoint. Different weights.
    The spread is large: the three math-trained projected-distillation seeds score $0.749$, $0.311$
    and $0.688$ on MATH-500.
  \item An \textbf{evaluation seed} indexes which generations were sampled from \emph{one fixed}
    checkpoint. The model is identical across replicates. The spread is small, typically
    $\pm 0.006$ on the large suites.
\end{itemize}

The two are reported in different places. The $\pm$ in \Cref{tab:math-seed0,tab:math-seed0-delta}
is the standard error over evaluation seeds of a single training-seed-$0$ checkpoint: it certifies
that a measurement would reproduce under re-sampling, and says nothing about whether a different
training run would land near it. Everywhere else, including \Cref{tab:rq1-main,tab:rq1-bwt,tab:rq1-forget},
the $\pm$ is the standard error over the three training seeds. Claims about
a \emph{method} rather than about a particular checkpoint have to be read against training-seed
spread, which is the dominant term.

Greedy decoding cannot be averaged this way. It is deterministic, so replicates return the same
number and a standard deviation of exactly zero, which reports nothing about how well the mean is
pinned down. Sampling one generation per problem and repeating over independent seeds gives an error
bar that means something. For pass@$k$ estimated from $n>k$ samples we use the unbiased estimator of
\citet{chen2021codex}, $\text{pass@}k = 1 - \binom{n-c}{k}\big/\binom{n}{k}$. The pass@$16$ figures
in this paper are not estimated that way: they are computed from $n{=}16$ draws as the share of
problems solved at least once. That indicator is unbiased for pass@$16$ when $n=k$, but it carries
more variance than an estimate from more draws would.

\noindent\textbf{Resolution limits of the small suites.} Suite size sets the granularity of any
difference that can be claimed. AMC23 has only $40$ problems, so one problem is $0.025$ and the
base standard error alone is $\pm 0.030$; differences below roughly $0.06$ there should not be
argued. GSM8K ($1{,}319$) and MMLU ($14{,}042$) are the suites where small differences resolve.

\noindent\textbf{A note on greedy reproducibility.} Repeating a greedy evaluation with identical
prompts and identical suite size did not return bit-identical scores. The cause is
continuous-batching non-determinism in the inference server, which changes numerics enough to flip
borderline argmax tokens. Greedy numbers therefore carry a small amount of unreported jitter of
their own, which is a further reason to report a sampled mean with an explicit error bar.

\section{RQ1: Comparison with baselines}
\label{app:rq1-full}

\subsection{Metrics}

\noindent\textbf{Area under the Pareto frontier (AUPF).}
For a single run, evaluation at checkpoints $i=1,\dots,n$ yields points $(x_i,y_i)\in[0,1]^2$,
where $x_i$ is specialization (hard-set pass@1) and $y_i$ is retention (held-out pass@1). A
point is \emph{dominated} if another point is at least as good on both axes, and the Pareto
frontier is the non-dominated set
\begin{equation}
  \mathcal{F}=\bigl\{\,i:\ \nexists\, j\ \text{with}\ x_j\ge x_i,\ y_j\ge y_i,\
  (x_j,y_j)\neq(x_i,y_i)\,\bigr\}.
\end{equation}
We anchor the frontier at a \emph{shared untrained base} $B=(0,H_0)$: specialization $0$ and
retention $H_0$, the untrained held-out accuracy, both estimated as the mean step-$0$ value over
the methods that start untrained. Specialization at the anchor is measured at step $0$ rather than
inferred from the mining criterion; it averages $0.012$ across those methods, which we round to zero. $B$ is a \emph{forced} left anchor (it is never subject to
domination); using a shared base rather than each run's own first checkpoint prevents methods
that begin with an SFT warmup (e.g.\ SODA) from being scored against an already-elevated
starting specialization. Indexing $\{B\}\cup\mathcal{F}$ in order of increasing specialization,
$0=\tilde x_0<\tilde x_1<\dots<\tilde x_K$, AUPF is the area under the piecewise-linear frontier
$y_{\mathcal F}(\cdot)$ by the trapezoidal rule:
\begin{equation}
  \mathrm{AUPF}\;=\;\int_{0}^{\tilde x_K} y_{\mathcal F}(x)\,dx
  \;=\;\sum_{k=1}^{K}\tfrac{1}{2}\,(\tilde y_{k-1}+\tilde y_k)\,(\tilde x_k-\tilde x_{k-1}).
\end{equation}
We compute AUPF per seed and report the mean\,$\pm$\,s.e.\ across seeds. Because retention is
integrated over specialization from the base at $x{=}0$, a method that never specializes
($\tilde x_K\!\approx\!0$) has $\mathrm{AUPF}\!\approx\!0$; a higher AUPF means pushing
specialization further while holding retention.

\noindent\textbf{Continual-learning metrics (\Cref{tab:rq1-cl}).}
For a run with per-step specialization $x_t$ and retention $y_t$ ($t=1,\dots,T$) and the
shared untrained base $(0,H_0)$, define
\begin{align*}
  \bar x &= \tfrac{1}{T}\textstyle\sum_{t=1}^{T} x_t, &
  \bar y &= \tfrac{1}{T}\textstyle\sum_{t=1}^{T} y_t, \\
  \mathrm{BWT} &= y_T - H_0, &
  F &= \max\nolimits_{\tau\le T} y_\tau - y_T \ \text{with}\ y_0{:=}H_0, \\
  \mathrm{SafeSpec}_\varepsilon &= \max\!\bigl\{x_t: y_t\ge H_0-\varepsilon\bigr\}, &
  &\text{$0$ if the set is empty.}
\end{align*}
$\bar x,\bar y$ are training-averaged accuracies in the spirit of the AUC metrics of
\citet{kemker2018measuring}; $\mathrm{BWT}$ is the backward transfer of \citet{lopezpaz2017gem}
and $F$ the forgetting measure of \citet{chaudhry2018rwalk}. $\mathrm{SafeSpec}_\varepsilon$ is the highest specialization
reached at any checkpoint whose retention is within $\varepsilon$ of the untrained base. For the aggregate
row in \Cref{tab:rq1-cl}, retention is expressed as a fraction of base so panels with
different retention scales are commensurate: $\bar y_{\mathrm{norm}}=\bar y/H_0$,
$F_{\mathrm{norm}}=F/H_0$, $\mathrm{BWT}_{\mathrm{norm}}=(y_T-H_0)/H_0$, and pooled
$\mathrm{SafeSpec}$ uses $y_t/H_0\ge1-\varepsilon_{\mathrm{rel}}$ with
$\varepsilon_{\mathrm{rel}}\in\{0.02,0.05\}$. Each metric is computed per seed and averaged
over $3$ seeds; the aggregate pools further over $($family, benchmark$)$, with GRPO on
Qwen-LiveCodeBench entered as no learning.

\subsection{Per-cell results}

\Cref{tab:rq1-main} gives the exact final-checkpoint accuracies behind \Cref{fig:pareto-all}:
hard-set pass@1 (specialization) and held-out pass@1 (retention), as mean\,$\pm$\,standard error
across seeds at $250$ steps. \Cref{fig:pareto-per-benchmark} breaks that figure down per benchmark,
\Cref{fig:learning-curves} gives the pooled trajectories over training, and
\Cref{tab:rq1-bwt,tab:rq1-forget} break the two continual-learning summaries of \Cref{tab:rq1-cl}
down per cell.

\begin{figure}[H]
  \centering
  \includegraphics[width=\linewidth]{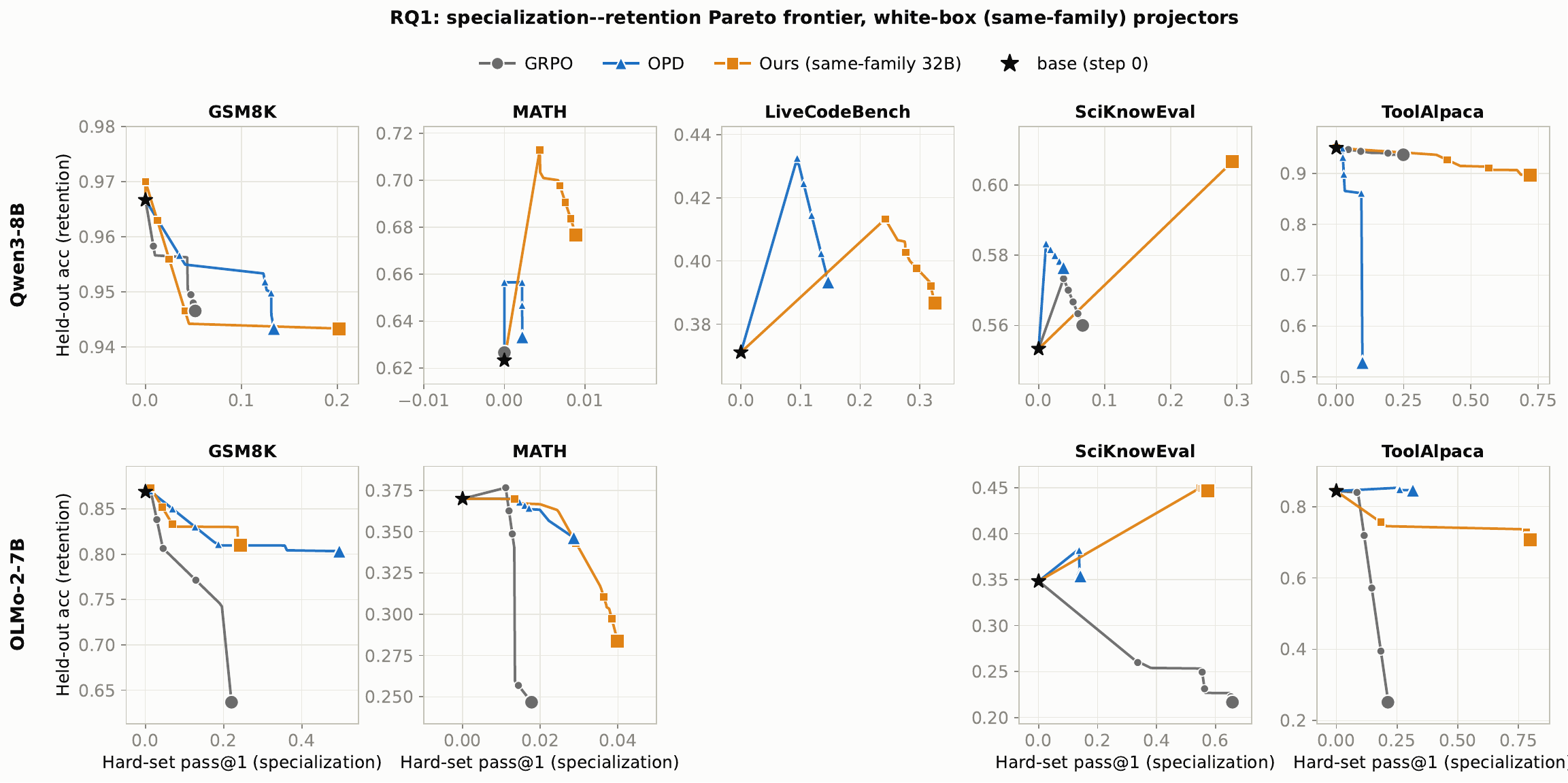}\\[0.6em]
  \includegraphics[width=\linewidth]{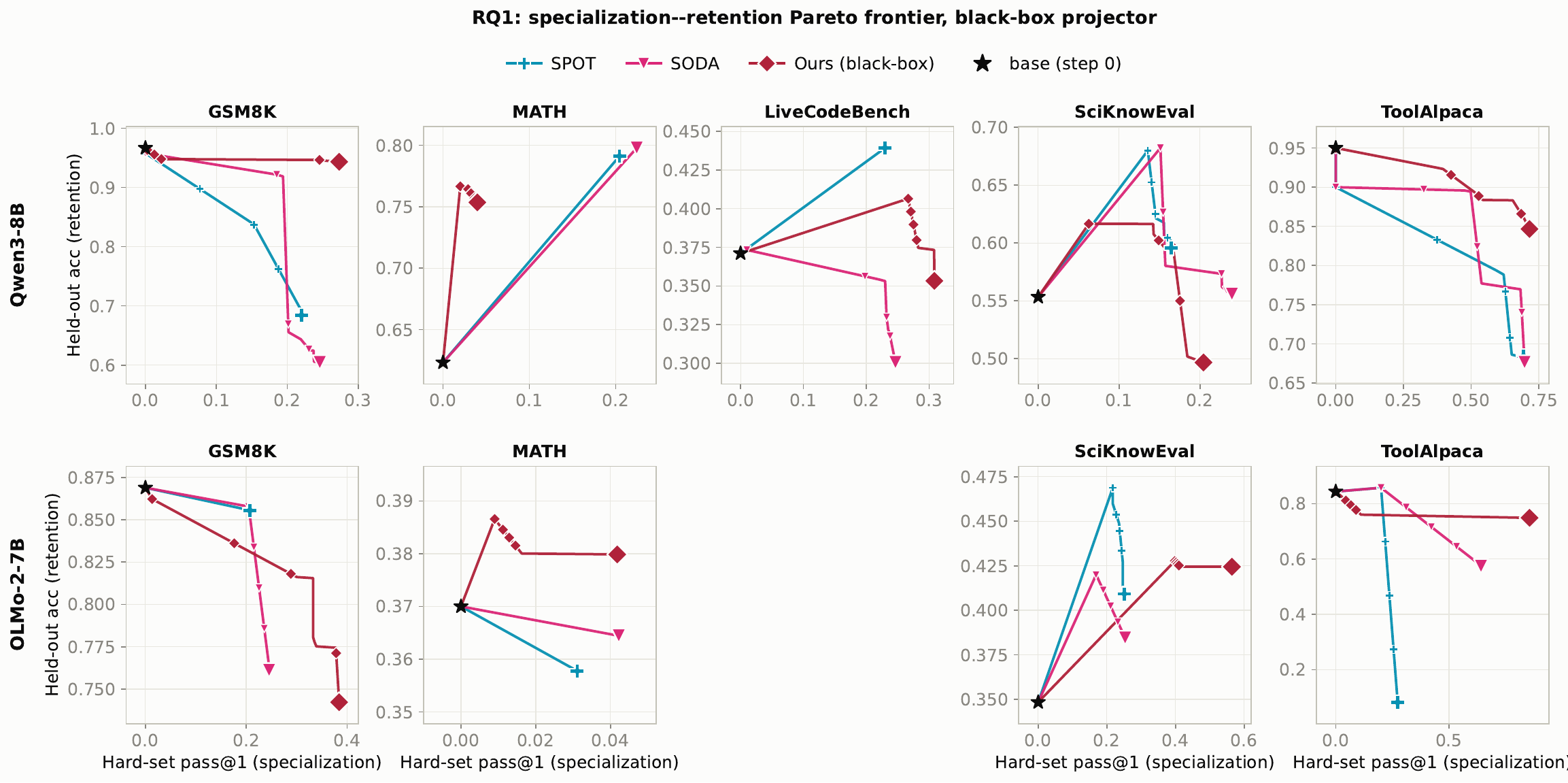}
  \caption{\textbf{Per-benchmark specialization--retention Pareto frontier} (detailed
  breakdown of \Cref{fig:pareto-all}). \textbf{Top:} same-family ($32$B) projectors;
  \textbf{bottom:} black-box projector. Each block has two rows (Qwen3-8B, OLMo-2-7B) and one
  column per benchmark. Axes: hard-set pass@1 ($x$, specialization) and held-out pass@1 ($y$,
  retention). Bold line $=$ monotone Pareto frontier of the seed-averaged trajectory (larger
  marker $=$ final checkpoint) anchored at the untrained base $\star$ (specialization $0$);
  axes auto-scale per panel. GRPO on Qwen-LiveCodeBench is entered as no learning.}
  \label{fig:pareto-per-benchmark}
\end{figure}

\begin{figure}[H]
  \centering
  \includegraphics[width=.7\linewidth]{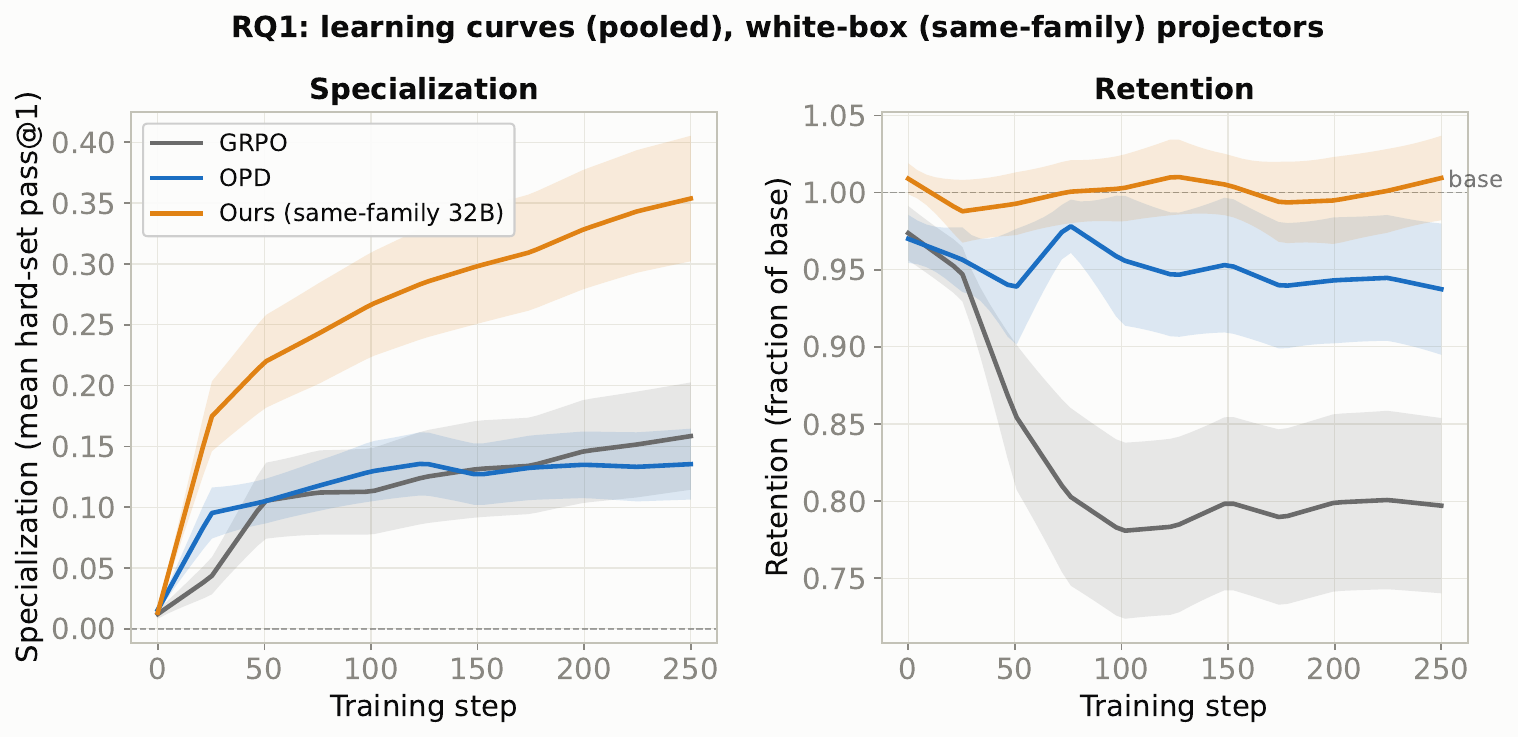}\\[0.5em]
  \includegraphics[width=.7\linewidth]{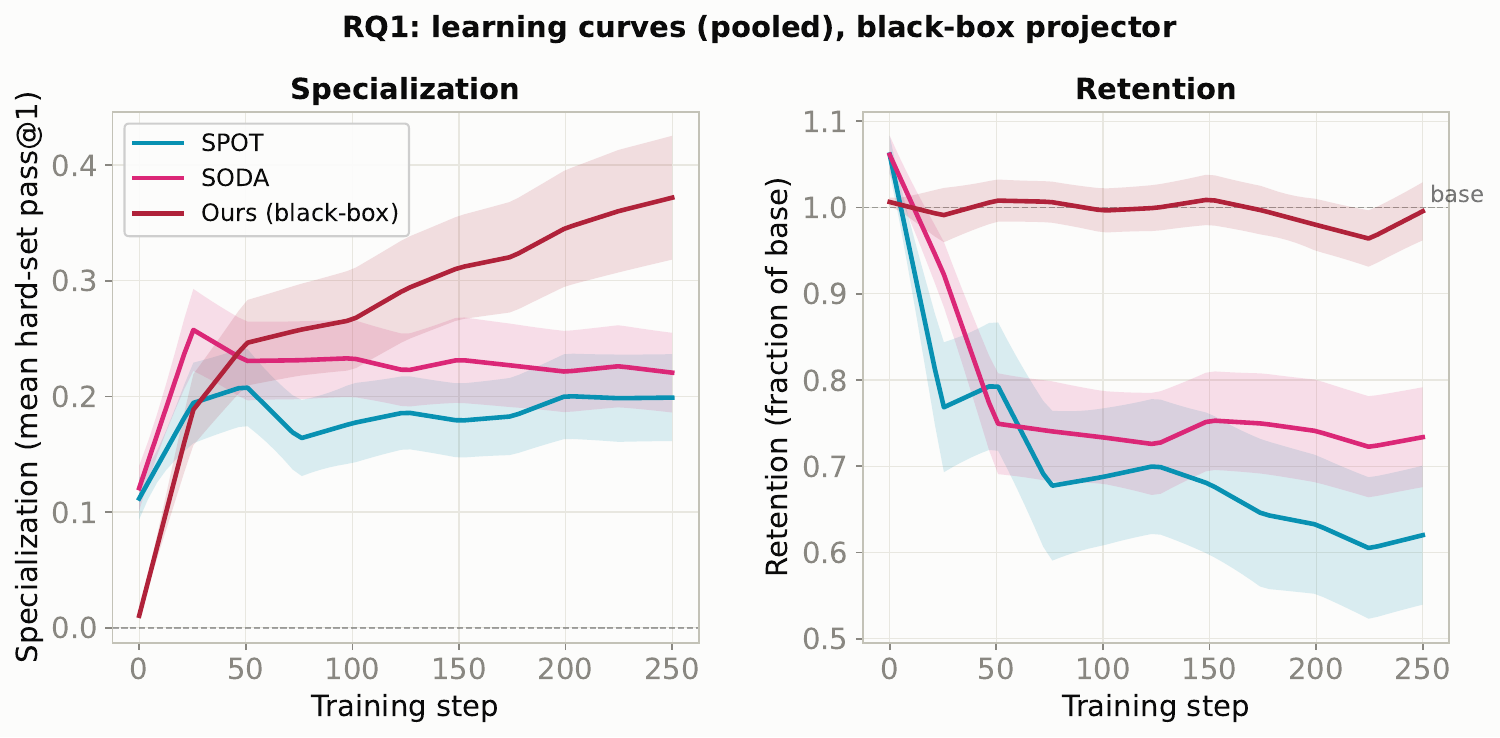}
  \caption{\textbf{RQ1 pooled learning curves.} \textbf{Top:} same-family ($32$B)
  projectors; \textbf{bottom:} black-box projector. \emph{Left panels:} specialization
  (mean hard-set pass@1) over training. \emph{Right panels:} retention as a fraction of the
  untrained base (base $=1.0$ dashed). Curves pool over students, benchmarks, and seeds (same
  pooling as \Cref{fig:pareto-all}); shaded bands are s.e.\ over the pooled runs. The
  early specialization and subsequent retention loss of SPOT and SODA (retention falls by step $50$) and the
  retention preservation of both Ours variants are visible directly. GRPO on Qwen-LiveCodeBench
  entered as no learning.}
  \label{fig:learning-curves}
\end{figure}

\begin{table}[H]
  \centering
  \footnotesize
  \setlength{\tabcolsep}{2pt}
  \caption{\textbf{RQ1 per-cell final accuracies} at the $250$-step checkpoint,
  mean\,$\pm$\,s.e.\ across seeds. Hard $=$ hard-set pass@1 (specialization); Held $=$ held-out
  pass@1 (in-domain retention), measured on a disjoint slice of the same benchmark's training
  split, with sizes in \Cref{tab:hardset-qwen,tab:hardset-olmo}. These slices are not the public
  test suites used in \Cref{tab:math-seed0}, so the two are not comparable cell by cell. Higher is better on both, and a method is preferable where it
  raises Hard without lowering Held. \textbf{Bold} $=$ best on each axis within a (student,
  benchmark) block. Our two variants take the best Hard in $6$ of $9$ blocks and the best Held in
  $5$ of $9$, so the AUPF ordering in \Cref{tab:rq1-aupf} rests on the two axes together rather
  than on either alone. $^{\dagger}$ This is the one cell reported from a single training seed:
  OLMo-2-7B/MATH GRPO seeds $0$ and $1$ stopped with CUDA out-of-memory before the $250$-step
  checkpoint, so only seed $2$ is reported and the $\pm$ there is the standard error of one value,
  not of three. Every other cell is three seeds.}
  \label{tab:rq1-main}
  \begin{tabular}{cllcc}
    \toprule
    Student & Benchmark & Method & Hard $\uparrow$ & Held $\uparrow$ \\
    \midrule
    \multirow{30}{*}{\rotatebox[origin=c]{90}{Qwen3-8B}}
      & \multirow{6}{*}{GSM8K} & GRPO & 0.044\,$\pm$\,0.012 & 0.950\,$\pm$\,0.006 \\
      &  & OPD & 0.123\,$\pm$\,0.013 & \textbf{0.953}\,$\pm$\,0.007 \\
      &  & Ours (32B) & 0.197\,$\pm$\,0.019 & 0.943\,$\pm$\,0.019 \\
      &  & SPOT & 0.090\,$\pm$\,0.005 & 0.216\,$\pm$\,0.039 \\
      &  & SODA & 0.246\,$\pm$\,0.021 & 0.604\,$\pm$\,0.140 \\
      &  & Ours (bb) & \textbf{0.273}\,$\pm$\,0.022 & 0.943\,$\pm$\,0.009 \\
      \cmidrule(l){2-5}
      & \multirow{6}{*}{MATH} & GRPO & 0.000\,$\pm$\,0.000 & 0.627\,$\pm$\,0.003 \\
      &  & OPD & 0.002\,$\pm$\,0.002 & 0.633\,$\pm$\,0.017 \\
      &  & Ours (32B) & 0.007\,$\pm$\,0.004 & 0.700\,$\pm$\,0.015 \\
      &  & SPOT & 0.087\,$\pm$\,0.008 & 0.258\,$\pm$\,0.028 \\
      &  & SODA & \textbf{0.136}\,$\pm$\,0.023 & 0.664\,$\pm$\,0.002 \\
      &  & Ours (bb) & 0.031\,$\pm$\,0.004 & \textbf{0.760}\,$\pm$\,0.006 \\
      \cmidrule(l){2-5}
      & \multirow{6}{*}{LiveCodeBench} & GRPO & 0.000\,$\pm$\,0.000 & 0.371\,$\pm$\,0.000 \\
      &  & OPD & 0.093\,$\pm$\,0.032 & 0.367\,$\pm$\,0.018 \\
      &  & Ours (32B) & \textbf{0.325}\,$\pm$\,0.014 & 0.387\,$\pm$\,0.013 \\
      &  & SPOT & 0.217\,$\pm$\,0.003 & \textbf{0.413}\,$\pm$\,0.027 \\
      &  & SODA & 0.237\,$\pm$\,0.009 & 0.313\,$\pm$\,0.041 \\
      &  & Ours (bb) & 0.308\,$\pm$\,0.008 & 0.353\,$\pm$\,0.007 \\
      \cmidrule(l){2-5}
      & \multirow{6}{*}{SciKnowEval} & GRPO & 0.067\,$\pm$\,0.024 & 0.560\,$\pm$\,0.000 \\
      &  & OPD & 0.027\,$\pm$\,0.004 & 0.567\,$\pm$\,0.026 \\
      &  & Ours (32B) & \textbf{0.293}\,$\pm$\,0.015 & \textbf{0.607}\,$\pm$\,0.022 \\
      &  & SPOT & 0.164\,$\pm$\,0.026 & 0.596\,$\pm$\,0.004 \\
      &  & SODA & 0.227\,$\pm$\,0.033 & 0.573\,$\pm$\,0.037 \\
      &  & Ours (bb) & 0.204\,$\pm$\,0.019 & 0.497\,$\pm$\,0.131 \\
      \cmidrule(l){2-5}
      & \multirow{6}{*}{ToolAlpaca} & GRPO & 0.249\,$\pm$\,0.154 & \textbf{0.937}\,$\pm$\,0.009 \\
      &  & OPD & 0.078\,$\pm$\,0.039 & 0.463\,$\pm$\,0.274 \\
      &  & Ours (32B) & \textbf{0.720}\,$\pm$\,0.010 & 0.897\,$\pm$\,0.015 \\
      &  & SPOT & 0.691\,$\pm$\,0.006 & 0.678\,$\pm$\,0.073 \\
      &  & SODA & 0.667\,$\pm$\,0.023 & 0.691\,$\pm$\,0.067 \\
      &  & Ours (bb) & 0.716\,$\pm$\,0.008 & 0.847\,$\pm$\,0.023 \\
    \midrule
    \multirow{24}{*}{\rotatebox[origin=c]{90}{OLMo-2-7B}}
      & \multirow{6}{*}{GSM8K} & GRPO & 0.220\,$\pm$\,0.070 & 0.637\,$\pm$\,0.048 \\
      &  & OPD & \textbf{0.496}\,$\pm$\,0.016 & 0.803\,$\pm$\,0.015 \\
      &  & Ours (32B) & 0.236\,$\pm$\,0.008 & \textbf{0.830}\,$\pm$\,0.006 \\
      &  & SPOT & 0.115\,$\pm$\,0.045 & 0.313\,$\pm$\,0.153 \\
      &  & SODA & 0.153\,$\pm$\,0.007 & 0.527\,$\pm$\,0.020 \\
      &  & Ours (bb) & 0.378\,$\pm$\,0.091 & 0.774\,$\pm$\,0.018 \\
      \cmidrule(l){2-5}
      & \multirow{6}{*}{MATH} & GRPO$^{\dagger}$ & 0.020\,$\pm$\,0.000 & 0.330\,$\pm$\,0.000 \\
      &  & OPD & 0.013\,$\pm$\,0.004 & 0.370\,$\pm$\,0.025 \\
      &  & Ours (32B) & \textbf{0.036}\,$\pm$\,0.008 & 0.307\,$\pm$\,0.009 \\
      &  & SPOT & 0.009\,$\pm$\,0.006 & 0.133\,$\pm$\,0.020 \\
      &  & SODA & 0.011\,$\pm$\,0.002 & 0.151\,$\pm$\,0.008 \\
      &  & Ours (bb) & 0.016\,$\pm$\,0.006 & \textbf{0.380}\,$\pm$\,0.010 \\
      \cmidrule(l){2-5}
      & \multirow{6}{*}{SciKnowEval} & GRPO & \textbf{0.658}\,$\pm$\,0.037 & 0.217\,$\pm$\,0.047 \\
      &  & OPD & 0.113\,$\pm$\,0.024 & 0.353\,$\pm$\,0.033 \\
      &  & Ours (32B) & 0.573\,$\pm$\,0.023 & \textbf{0.447}\,$\pm$\,0.003 \\
      &  & SPOT & 0.224\,$\pm$\,0.008 & 0.431\,$\pm$\,0.019 \\
      &  & SODA & 0.089\,$\pm$\,0.016 & 0.371\,$\pm$\,0.002 \\
      &  & Ours (bb) & 0.564\,$\pm$\,0.035 & 0.424\,$\pm$\,0.024 \\
      \cmidrule(l){2-5}
      & \multirow{6}{*}{ToolAlpaca} & GRPO & 0.171\,$\pm$\,0.171 & 0.140\,$\pm$\,0.140 \\
      &  & OPD & 0.273\,$\pm$\,0.015 & \textbf{0.840}\,$\pm$\,0.015 \\
      &  & Ours (32B) & 0.798\,$\pm$\,0.006 & 0.707\,$\pm$\,0.047 \\
      &  & SPOT & 0.193\,$\pm$\,0.112 & 0.069\,$\pm$\,0.049 \\
      &  & SODA & 0.220\,$\pm$\,0.070 & 0.189\,$\pm$\,0.059 \\
      &  & Ours (bb) & \textbf{0.856}\,$\pm$\,0.026 & 0.749\,$\pm$\,0.036 \\
    \bottomrule
  \end{tabular}
\end{table}

\begin{table}[H]
  \centering
  \scriptsize
  \setlength{\tabcolsep}{2pt}
  \caption{\textbf{RQ1 per-cell backward transfer}, $\mathrm{BWT}=y_T-H_0$
  \citep{lopezpaz2017gem}, mean\,$\pm$\,s.e.\ across seeds in raw retention points. Higher is
  better: a positive value means the run ended above the untrained base, a negative one that it
  ended below, and zero that it returned to base. \textbf{Bold} $=$ highest per row. Our variants
  take the highest value in $5$ of $9$ cells and SPOT the lowest in $5$. Every method ends below
  base on GSM8K and ToolAlpaca for both students, so those two benchmarks carry the retention cost
  regardless of how the correction is obtained. GRPO on Qwen-LiveCodeBench is entered as no
  learning.}
  \label{tab:rq1-bwt}
  \setlength{\tabcolsep}{2pt}\begin{tabular}{ll c cc ccc}
    \toprule
    & & RLVR & \multicolumn{2}{c}{same-family $32$B teacher} & \multicolumn{3}{c}{black-box projector} \\
    \cmidrule(lr){3-3}\cmidrule(lr){4-5}\cmidrule(lr){6-8}
    Student & Benchmark & GRPO & OPD & Ours & SPOT & SODA & Ours \\
    \midrule
    \multirow{5}{*}{Qwen3-8B}
      & GSM8K          & $-0.017${\scriptsize$\pm$0.006} & $\mathbf{-0.013}${\scriptsize$\pm$0.007} & $-0.023${\scriptsize$\pm$0.019} & $-0.751${\scriptsize$\pm$0.039} & $-0.362${\scriptsize$\pm$0.140} & $-0.023${\scriptsize$\pm$0.009} \\
      & MATH           & $+0.003${\scriptsize$\pm$0.003} & $+0.010${\scriptsize$\pm$0.017} & $+0.077${\scriptsize$\pm$0.015} & $-0.366${\scriptsize$\pm$0.028} & $+0.041${\scriptsize$\pm$0.002} & $\mathbf{+0.137}${\scriptsize$\pm$0.006} \\
      & LiveCodeBench  & $+0.000${\scriptsize$\pm$0.000} & $-0.004${\scriptsize$\pm$0.018} & $+0.016${\scriptsize$\pm$0.013} & $\mathbf{+0.042}${\scriptsize$\pm$0.027} & $-0.058${\scriptsize$\pm$0.041} & $-0.018${\scriptsize$\pm$0.007} \\
      & SciKnowEval    & $+0.007${\scriptsize$\pm$0.000} & $+0.013${\scriptsize$\pm$0.026} & $\mathbf{+0.053}${\scriptsize$\pm$0.022} & $+0.042${\scriptsize$\pm$0.004} & $+0.020${\scriptsize$\pm$0.037} & $-0.057${\scriptsize$\pm$0.131} \\
      & ToolAlpaca     & $\mathbf{-0.013}${\scriptsize$\pm$0.009} & $-0.487${\scriptsize$\pm$0.274} & $-0.053${\scriptsize$\pm$0.015} & $-0.272${\scriptsize$\pm$0.073} & $-0.259${\scriptsize$\pm$0.067} & $-0.103${\scriptsize$\pm$0.023} \\
    \midrule
    \multirow{4}{*}{OLMo-2-7B}
      & GSM8K          & $-0.232${\scriptsize$\pm$0.048} & $-0.066${\scriptsize$\pm$0.015} & $\mathbf{-0.039}${\scriptsize$\pm$0.006} & $-0.556${\scriptsize$\pm$0.153} & $-0.342${\scriptsize$\pm$0.020} & $-0.094${\scriptsize$\pm$0.018} \\
      & MATH           & $-0.123${\scriptsize$\pm$0.049} & $-0.000${\scriptsize$\pm$0.025} & $-0.063${\scriptsize$\pm$0.009} & $-0.237${\scriptsize$\pm$0.020} & $-0.219${\scriptsize$\pm$0.008} & $\mathbf{+0.010}${\scriptsize$\pm$0.010} \\
      & SciKnowEval    & $-0.132${\scriptsize$\pm$0.047} & $+0.005${\scriptsize$\pm$0.033} & $\mathbf{+0.098}${\scriptsize$\pm$0.003} & $+0.083${\scriptsize$\pm$0.019} & $+0.023${\scriptsize$\pm$0.002} & $+0.076${\scriptsize$\pm$0.024} \\
      & ToolAlpaca     & $-0.704${\scriptsize$\pm$0.140} & $\mathbf{-0.004}${\scriptsize$\pm$0.015} & $-0.137${\scriptsize$\pm$0.047} & $-0.775${\scriptsize$\pm$0.049} & $-0.655${\scriptsize$\pm$0.059} & $-0.095${\scriptsize$\pm$0.036} \\
    \bottomrule
  \end{tabular}
\end{table}

\begin{table}[H]
  \centering
  \scriptsize
  \setlength{\tabcolsep}{2pt}
  \caption{\textbf{RQ1 per-cell forgetting}, $F=\max_{\tau\le T} y_\tau - y_T$ with $H_0$
  included in the max \citep{chaudhry2018rwalk}, mean\,$\pm$\,s.e.\ across seeds in raw retention
  points. Smaller is better: $F$ is the retention a run gave up between its best checkpoint and its
  last, so a small value means retention was held throughout rather than lost and partly regained.
  \textbf{Bold} $=$ smallest per row. SPOT is largest in $5$ of $9$ cells. The smallest is split
  evenly between OPD, GRPO and our same-family variant, so $F$ alone does not separate methods and
  should be read alongside specialization, since a run that moves the student little also has
  little to give up. GRPO on Qwen-LiveCodeBench is entered as no learning.}
  \label{tab:rq1-forget}
  \setlength{\tabcolsep}{2pt}\begin{tabular}{ll c cc ccc}
    \toprule
    & & RLVR & \multicolumn{2}{c}{same-family $32$B teacher} & \multicolumn{3}{c}{black-box projector} \\
    \cmidrule(lr){3-3}\cmidrule(lr){4-5}\cmidrule(lr){6-8}
    Student & Benchmark & GRPO & OPD & Ours & SPOT & SODA & Ours \\
    \midrule
    \multirow{5}{*}{Qwen3-8B}
      & GSM8K          & 0.020{\scriptsize$\pm$0.006} & \textbf{0.016}{\scriptsize$\pm$0.007} & 0.030{\scriptsize$\pm$0.015} & 0.751{\scriptsize$\pm$0.039} & 0.362{\scriptsize$\pm$0.140} & 0.024{\scriptsize$\pm$0.009} \\
      & MATH           & \textbf{0.001}{\scriptsize$\pm$0.001} & 0.030{\scriptsize$\pm$0.012} & 0.027{\scriptsize$\pm$0.015} & 0.533{\scriptsize$\pm$0.030} & 0.133{\scriptsize$\pm$0.004} & 0.010{\scriptsize$\pm$0.006} \\
      & LiveCodeBench  & \textbf{0.000}{\scriptsize$\pm$0.000} & 0.073{\scriptsize$\pm$0.029} & 0.047{\scriptsize$\pm$0.018} & 0.040{\scriptsize$\pm$0.020} & 0.070{\scriptsize$\pm$0.038} & 0.053{\scriptsize$\pm$0.007} \\
      & SciKnowEval    & 0.040{\scriptsize$\pm$0.010} & 0.047{\scriptsize$\pm$0.026} & \textbf{0.023}{\scriptsize$\pm$0.012} & 0.084{\scriptsize$\pm$0.009} & 0.109{\scriptsize$\pm$0.046} & 0.153{\scriptsize$\pm$0.129} \\
      & ToolAlpaca     & \textbf{0.013}{\scriptsize$\pm$0.009} & 0.490{\scriptsize$\pm$0.272} & 0.053{\scriptsize$\pm$0.015} & 0.272{\scriptsize$\pm$0.073} & 0.259{\scriptsize$\pm$0.067} & 0.103{\scriptsize$\pm$0.023} \\
    \midrule
    \multirow{4}{*}{OLMo-2-7B}
      & GSM8K          & 0.233{\scriptsize$\pm$0.048} & 0.070{\scriptsize$\pm$0.016} & \textbf{0.043}{\scriptsize$\pm$0.007} & 0.556{\scriptsize$\pm$0.153} & 0.342{\scriptsize$\pm$0.020} & 0.095{\scriptsize$\pm$0.018} \\
      & MATH           & 0.147{\scriptsize$\pm$0.038} & \textbf{0.023}{\scriptsize$\pm$0.015} & 0.077{\scriptsize$\pm$0.009} & 0.240{\scriptsize$\pm$0.022} & 0.219{\scriptsize$\pm$0.008} & 0.027{\scriptsize$\pm$0.015} \\
      & SciKnowEval    & 0.132{\scriptsize$\pm$0.047} & 0.063{\scriptsize$\pm$0.032} & \textbf{0.013}{\scriptsize$\pm$0.009} & 0.062{\scriptsize$\pm$0.006} & 0.060{\scriptsize$\pm$0.014} & 0.033{\scriptsize$\pm$0.018} \\
      & ToolAlpaca     & 0.704{\scriptsize$\pm$0.140} & \textbf{0.030}{\scriptsize$\pm$0.006} & 0.139{\scriptsize$\pm$0.048} & 0.789{\scriptsize$\pm$0.048} & 0.669{\scriptsize$\pm$0.058} & 0.098{\scriptsize$\pm$0.037} \\
    \bottomrule
  \end{tabular}
\end{table}

\subsection{Math-domain evaluation: in-domain and out-of-domain behavior}
\label{app:math-full}

\Cref{sec:rq1-main} measures retention on a held-out slice of the benchmark a method was trained
on. Here we ask what training on one domain does to abilities elsewhere. For each math-trained
Qwen3-8B method we evaluate its final checkpoint on five full test suites: GSM8K ($1{,}319$ problems)
and MATH-500 ($500$) in the trained domain, and AMC23 ($40$), MMLU ($14{,}042$) and HumanEval
($164$), none of which any method was trained on. The metric is sampled pass@1, one generation per
problem at $T{=}0.8$ and top-$p$ $0.95$. Each evaluation is repeated with three \emph{evaluation}
seeds and averaged, with the checkpoint held fixed: an evaluation seed changes only which
generations are sampled from that one model. The reported $\pm$ is the standard error of that mean
across the three replicates. Training and evaluation data are disjoint, with zero overlap between
the $1{,}991$-problem math hard set and the $500$-problem evaluation slice on normalized,
LaTeX-stripped text. \Cref{tab:math-seed0} reports the absolute scores and
\Cref{tab:math-seed0-delta} the change relative to the untrained student.

\begin{table}[H]
  \centering
  \small
  \caption{\textbf{Training on math corrections suggests retention is held outside the trained
  domain.} Final-checkpoint evaluation of the \emph{training-seed-$0$} math-trained Qwen3-8B
  methods, sampled pass@1 at
  $T{=}0.8$ averaged over three evaluation seeds, $\pm$ the standard error of that mean, on full
  public test splits. These are distinct from the in-domain held-out slices that
  \Cref{tab:rq1-main} uses for retention. GSM8K and MATH-500 are in the trained domain; AMC23, MMLU and HumanEval were never
  trained on. The black-box variant is the only arm whose score rises on the held-out split of
  the domain it trained on ($0.709 \to 0.749$ on MATH-500, where the same recipe with a $32$B
  projector reaches $0.615$, SPOT $0.513$ and SODA $0.590$), and it reaches $0.765$ on MMLU while
  remaining within a few points of base on GSM8K and HumanEval. \textbf{Bold} $=$ best and \underline{underline} $=$ second best per column among the
  six trained methods; the untrained base is a reference row and is not ranked. GRPO ranks high on
  several columns because it moves the student least, as \Cref{tab:math-seed0-delta} makes plain,
  so these absolute scores are best read alongside that table. These are single checkpoints on one
  domain, so the pattern is suggestive rather than conclusive, and AMC23 has only $40$ problems,
  where small differences should not be read as meaningful.}
  \label{tab:math-seed0}
  \begin{tabular}{lccccc}
  \toprule
  Method & GSM8K & MATH-500 & AMC23 & MMLU & HumanEval \\
  \midrule
  Base (Qwen3-8B) & 0.934\,$\pm$\,0.003 & 0.709\,$\pm$\,0.005 & 0.433\,$\pm$\,0.030 & 0.702\,$\pm$\,0.002 & 0.821\,$\pm$\,0.002 \\
  \midrule
  \textsc{Ours} & 0.915\,$\pm$\,0.001 & \textbf{0.749}\,$\pm$\,0.003 & \textbf{0.492}\,$\pm$\,0.036 & \underline{0.765}\,$\pm$\,0.001 & 0.787\,$\pm$\,0.000 \\
  \textsc{Ours}-32B & 0.917\,$\pm$\,0.001 & 0.615\,$\pm$\,0.005 & 0.400\,$\pm$\,0.050 & 0.724\,$\pm$\,0.001 & \textbf{0.827}\,$\pm$\,0.005 \\
  \textsc{Soda} & 0.882\,$\pm$\,0.004 & 0.590\,$\pm$\,0.019 & 0.342\,$\pm$\,0.022 & \textbf{0.766}\,$\pm$\,0.002 & 0.797\,$\pm$\,0.005 \\
  \textsc{Spot} & 0.916\,$\pm$\,0.002 & 0.513\,$\pm$\,0.011 & 0.325\,$\pm$\,0.029 & 0.750\,$\pm$\,0.002 & 0.813\,$\pm$\,0.014 \\
  GRPO & \textbf{0.932}\,$\pm$\,0.004 & \underline{0.700}\,$\pm$\,0.008 & \underline{0.425}\,$\pm$\,0.029 & 0.702\,$\pm$\,0.002 & \underline{0.821}\,$\pm$\,0.002 \\
  OPD & \underline{0.927}\,$\pm$\,0.002 & 0.692\,$\pm$\,0.005 & 0.408\,$\pm$\,0.008 & 0.677\,$\pm$\,0.002 & 0.815\,$\pm$\,0.005 \\
  \bottomrule
  \end{tabular}
\end{table}

\begin{table}[H]
  \centering
  \small
  \caption{\textbf{Only the black-box variant improves in the trained domain.} Change relative to
  the untrained Qwen3-8B, from the same evaluation as \Cref{tab:math-seed0}. Ours is alone in moving
  up on MATH-500 ($+0.041$) and on AMC23 ($+0.058$), and gains $+0.063$ on MMLU. The same recipe
  with a $32$B projector loses $0.093$ on MATH-500, and the mined-once and static-snapshot
  corrections lose $0.119$ and $0.195$. GRPO is flat throughout and OPD moves down. \textbf{Bold} $=$ largest change and \underline{underline} $=$ second largest per
  column; GRPO is near zero everywhere, so where it places second it has simply moved the least.}
  \label{tab:math-seed0-delta}
  \begin{tabular}{lccccc}
  \toprule
  Method & GSM8K & MATH-500 & AMC23 & MMLU & HumanEval \\
  \midrule
  \textsc{Ours} & $-0.019$ & $\mathbf{+0.041}$ & $\mathbf{+0.058}$ & \underline{$+0.063$} & $-0.035$ \\
  \textsc{Ours}-32B & $-0.017$ & $-0.093$ & $-0.033$ & $+0.022$ & $\mathbf{+0.006}$ \\
  \textsc{Soda} & $-0.052$ & $-0.119$ & $-0.092$ & $\mathbf{+0.064}$ & $-0.024$ \\
  \textsc{Spot} & $-0.018$ & $-0.195$ & $-0.108$ & $+0.047$ & $-0.008$ \\
  GRPO & $\mathbf{-0.002}$ & \underline{$-0.009$} & \underline{$-0.008$} & $-0.001$ & \underline{$+0.000$} \\
  OPD & \underline{$-0.006$} & $-0.017$ & $-0.025$ & $-0.025$ & $-0.006$ \\
  \bottomrule
  \end{tabular}
\end{table}

\subsection{Discussion}

\Cref{sec:rq1-main} names three cells that go to a baseline; the full rows are in
\Cref{tab:rq1-main}. On Qwen MATH, SODA reaches $0.136$ specialization at $0.664$ retention and
SPOT $0.087$ at $0.258$, against our $0.031$ at $0.760$: both fit the hard set further than we do,
and SPOT gives up most of its held-out accuracy doing so. On Qwen SciKnowEval, SODA reaches $0.227$
at $0.573$ against our $0.204$ at $0.497$, leading on both axes; our retention there carries a
standard error of $0.131$ across seeds, against $0.037$ for SODA and $0.004$ for SPOT, so that cell
is the least stable in the panel. On OLMo GSM8K, OPD reaches $0.496$ at $0.803$ against our $0.236$
at $0.830$.

The OLMo-2-7B LiveCodeBench cell is omitted from every RQ1 table and figure. No method moves
specialization appreciably above zero there, so a frontier anchored at the base is degenerate and
the cell contributes an area near zero for every method, our own included.

\subsection{Transfer to other benchmarks' hard problems}
\label{app:transfer}

This subsection gives the setup and the full results of the transfer experiment summarized in
\Cref{sec:rq1-main} and \Cref{tab:transfer-main}.

\noindent\textbf{Hard sets.} For every problem of every pool we draw $16$ samples from the
untrained Qwen3-8B at $T{=}0.8$ (seed $0$, thinking mode disabled, $3{,}072$ new tokens) and grade
them with the MATH grader. A problem is \emph{hard} when all $16$ are wrong, as in
\Cref{app:hardset}. The budget is set so that a hard problem is one the student cannot solve rather
than one it cannot finish (\Cref{app:train-extra}). This leaves $628$ of the $11{,}998$ MATH
training problems hard. We split the $628$ into $501$ training and
$127$ held-out problems, stratified by subject and level. Every other pool is used only for testing.
A test problem is removed when a MATH training problem covers at least
$60\%$ of its $13$-word phrases with the same answer, which removes $410$ AIME and $474$ DAPO-Math
problems.

\noindent\textbf{Training and evaluation.} Each method trains Qwen3-8B on the $501$ problems for
$250$ updates (about four passes) with the configuration of \Cref{app:exp-details} and a
$3{,}072$-token rollout budget, saving a checkpoint every $50$ updates, under one training seed. Ours,
SPOT and SODA draw their corrections from the black-box projector, which corrects $1{,}798$ of our
$1{,}844$ wrong attempts ($97.5\%$). The base model and every checkpoint draw $16$ samples per problem at
$T{=}0.8$ and top-$p$ $0.95$ (seed $1$, so the base scores slightly above zero on its own hard set);
pass@1 is the mean over the samples and pass@16 the share of problems solved at least once. A
change against the base is marked
when it is at least twice its paired standard error. The four other benchmarks contribute
$2{,}885$ hard problems in all, by benchmark in \Cref{tab:transfer-bench}, out of $6{,}874$ that
survive filtering and decontamination.

\noindent\textbf{Which benchmarks transfer.} \Cref{tab:transfer-bench} splits the cross-benchmark
gain of our best checkpoint by benchmark. It concentrates in Omni-MATH and DAPO-Math, which
together hold $2{,}536$ of the $2{,}885$ hard problems, and it holds under sampling and under greedy
decoding alike. OlympiadBench moves within its standard error and the AIME problems do not move.

\begin{table}[H]
  \centering
  \small
  \caption{\textbf{Cross-benchmark gain of Ours at its best checkpoint, by benchmark}, on each
  benchmark's hard problems. Sampled pass@1 over $16$ generations and one greedy decode; changes
  against the base $\pm$ paired standard error.}
  \label{tab:transfer-bench}
  \begin{tabular}{lrcccc}
  \toprule
  & & \multicolumn{2}{c}{Sampled pass@1} & & Greedy \\
  \cmidrule(lr){3-4}
  Benchmark & Hard & Base $\to$ Ours & Change & pass@16 & change \\
  \midrule
  OlympiadBench & 150 & $0.011\to0.017$ & $+0.006${\scriptsize$\pm$0.004} & $0.133\to0.147$ & $+0.000${\scriptsize$\pm$0.013} \\
  Omni-MATH & 2{,}178 & $0.010\to0.019$ & $\mathbf{+0.009}${\scriptsize$\pm$0.001} & $0.096\to0.132$ & $\mathbf{+0.015}${\scriptsize$\pm$0.003} \\
  AIME 1983--2023 & 199 & $0.016\to0.013$ & $-0.003${\scriptsize$\pm$0.003} & $0.146\to0.126$ & $-0.010${\scriptsize$\pm$0.010} \\
  DAPO-Math & 358 & $0.018\to0.033$ & $\mathbf{+0.016}${\scriptsize$\pm$0.004} & $0.201\to0.307$ & $\mathbf{+0.034}${\scriptsize$\pm$0.010} \\
  \midrule
  All four & 2{,}885 & $0.011\to0.020$ & $\mathbf{+0.009}${\scriptsize$\pm$0.001} & $0.114\to0.154$ & $\mathbf{+0.015}${\scriptsize$\pm$0.003} \\
  \bottomrule
  \end{tabular}
\end{table}

\noindent\textbf{A second training benchmark.} To check that the transfer is not specific to MATH,
we repeat the experiment for our method with Omni-MATH as the training benchmark: $501$ of its
$2{,}178$ hard problems, drawn in proportion to its difficulty levels and topics, and otherwise the
same recipe and protocol. The projector corrects $86$ to $96\%$ of the wrong attempts. The two runs
share most test problems, so \Cref{tab:transfer-omni} compares them on identical sets: the $127$
held-out MATH hard problems, the $1{,}677$ Omni-MATH hard problems that neither run trained on, and
the three benchmarks that neither run trained on (OlympiadBench, AIME 1983--2023 and DAPO-Math).
Transfer runs in both directions and to a similar degree. Training on Omni-MATH lifts the held-out
MATH hard problems about as much as training on MATH does (pass@16 $0.362$ against $0.354$), and
training on MATH lifts the Omni-MATH hard problems nearly as much as training on Omni-MATH does
(pass@1 $0.018$ against $0.022$). The Omni-MATH run transfers more to the three olympiad-style
benchmarks ($0.030$ against $0.024$). The direction of training therefore does not decide where the gain
appears.

\begin{table}[H]
  \centering
  \footnotesize
  \setlength{\tabcolsep}{3.2pt}
  \caption{\textbf{Training on MATH against training on Omni-MATH, on identical problems} (our
  method, $1$ training seed each, each at its own best checkpoint). ``$3$ others'' $=$
  OlympiadBench, AIME 1983--2023 and DAPO-Math, which neither run trained on, pooled over their
  hard problems. Sampled pass@1 unless marked. Each run's $^{*}$ is against its own evaluation of the base model; the two base
  evaluations agree to within $0.001$ in pass@1 and $0.008$ in pass@16.}
  \label{tab:transfer-omni}
  \begin{tabular}{lccccc}
  \toprule
  & \multicolumn{2}{c}{MATH held-out, hard} & \multicolumn{2}{c}{Omni-MATH held-out, hard} & $3$ others \\
  \cmidrule(lr){2-3}\cmidrule(lr){4-5}
  Trained on & pass@1 & pass@16 & pass@1 & pass@16 & hard \\
  \midrule
  Base & 0.013 & 0.150 & 0.009 & 0.092 & 0.016 \\
  \midrule
  MATH & \textbf{0.056}$^{*}$ & 0.354 & 0.018$^{*}$ & 0.128 & 0.024$^{*}$ \\
  Omni-MATH & 0.045$^{*}$ & \textbf{0.362} & \textbf{0.022}$^{*}$ & \textbf{0.146} & \textbf{0.030}$^{*}$ \\
  \bottomrule
  \end{tabular}
\end{table}

\section{RQ2: Edit size}
\label{app:rq2-mechanism}

\subsection{Setup and measurement}

Qwen3-8B student, GSM8K hard set (mined pass@$k$$=$$0$, $122$ problems), black-box projector,
one objective shared by every arm, lr $1\mathrm{e}{-}5$, $\beta{=}0.1$, $250$ steps, $3$ seeds. The axis is realized by \emph{edit magnitude}: the minimal edit, then re-prompting
the projector for progressively larger rewrites of the same attempt, plus the
student-agnostic \emph{gold} anchor. Because the projector rather than the prompt determines
how much text it actually replaces, arms are identified by their measured edit size
(\Cref{tab:rq2-onpol}); the two larger arms realize $3.2\times$ and $3.9\times$ the minimal
edit. Off-policyness is the log-likelihood difference between the correction and the attempt
under the current student $\pi_\theta$,
\begin{equation}
  \label{eq:rq2-onpol}
  \Lambda \;=\; \log\pi_{\theta}(y^+ \mid x)\;-\;\log\pi_{\theta}(y^- \mid x),
\end{equation}
and a larger $|\Lambda|$ means the correction sits farther from what the student would itself
produce. We summarize $\Lambda$ by its \emph{median} over training steps rather than its mean:
the per-step quantity is heavy-tailed, since a rollout is scored against a policy that has
advanced by up to five optimizer steps since it was sampled, and occasional large excursions
dominate an average. The median is stable across seeds
(\Cref{tab:rq2-onpol}) whereas the mean is not. For \emph{gold}, which replaces a long attempt
with a short reference solution, $\Lambda$ moves chiefly through length; \Cref{tab:rq2-onpol}
reports its value with that caveat, and it is omitted from the $|\Lambda|$ panel of
\Cref{fig:rq2}.

\subsection{Ablation arms}

\Cref{tab:rq2-onpol} lists the per-arm endpoint accuracies with the off-policyness scalar
$|\Lambda|$ and the mean edit-span sizes $|E^+|,|E^-|$ that define the arms. Both spans grow
together across the edit arms, so the two are not separable here. \emph{gold} is the exception in
kind rather than in degree, inserting the fewest tokens of any arm while deleting the entire
attempt.

\begin{table}[H]
  \centering
  \small
  \caption{RQ2 edit-size ablation (Qwen3-8B, GSM8K hard set, $3$ seeds, mean\,$\pm$\,s.e.).
  \emph{hard} and \emph{held} are final sampled pass@1; $\Lambda$ is the median signed
  likelihood shift over training under the current student (\eqref{eq:rq2-onpol}); a negative value
  means the correction is less likely than the attempt. $|E^+|,|E^-|$
  are mean edit-span sizes in tokens, and $|E^+|+|E^-|$ is the measured edit size by which the arms
  are identified. In the three outcome columns \textbf{bold} $=$ best and \underline{underline}
  $=$ second best, with $\Lambda$ read as nearer zero is better; the two edit-size columns describe
  the arms and are not ranked. $^{\dagger}$\emph{gold} is the one arm whose shift is positive: it replaces a
  long attempt with a short reference solution, which the student finds \emph{more} likely, not
  less. Its value moves chiefly through length and is not on the scale of the rewrites; its standard error is $62\%$ of the value against $5$ to $18\%$ for the edit
  arms.}
  \label{tab:rq2-onpol}
  \begin{tabular}{lccccc}
    \toprule
    Arm & $|E^+|\!+\!|E^-|$ & $|E^+|$ / $|E^-|$ & hard & held & $\Lambda$ \\
    \midrule
    minimal (ours) & $276$  & $160$ / $116$ & $\mathbf{0.560}${\scriptsize$\pm$0.015} & $\mathbf{0.916}${\scriptsize$\pm$0.016} & $\mathbf{-86}${\scriptsize$\pm$4} \\
    $3.2\times$             & $879$  & $308$ / $572$ & $0.301${\scriptsize$\pm$0.065}          & $\underline{0.847}${\scriptsize$\pm$0.067}          & $\underline{-372}${\scriptsize$\pm$67} \\
    $3.9\times$             & $1061$ & $359$ / $702$ & $0.254${\scriptsize$\pm$0.021}          & $0.807${\scriptsize$\pm$0.069}          & $-433${\scriptsize$\pm$35} \\
    gold (anchor)           & $926$  & $85$ / $841$  & $\underline{0.489}${\scriptsize$\pm$0.090}          & $0.744${\scriptsize$\pm$0.016}          & $+3969${\scriptsize$\pm$2471}$^{\dagger}$ \\
    \bottomrule
  \end{tabular}
\end{table}

\subsection{Mechanism}

\noindent\textbf{Full-utterance and diff-localized shift.}
\Cref{fig:rq2-fullvslocal} compares $\Lambda$ measured over the whole utterance against the
same quantity restricted to the changed spans $E^-\!/E^+$. The two agree at $r=0.9999$ across
all $2{,}249$ logged steps of the edit arms, with the localized quantity retaining $94$ to
$95\%$ of the magnitude, because the tokens before the first edit are both identical and identically
conditioned, so their contributions cancel exactly. Shared tokens \emph{after} an edit do not: they
are scored on prefixes that have already diverged, and they are what the remaining $5$ to $6\%$
measures. Localizing the measurement, as the objective of \eqref{eq:proj-dpo} does, therefore keeps
almost all of the signal.

\begin{figure}[H]
  \centering
  \includegraphics[width=.7\linewidth]{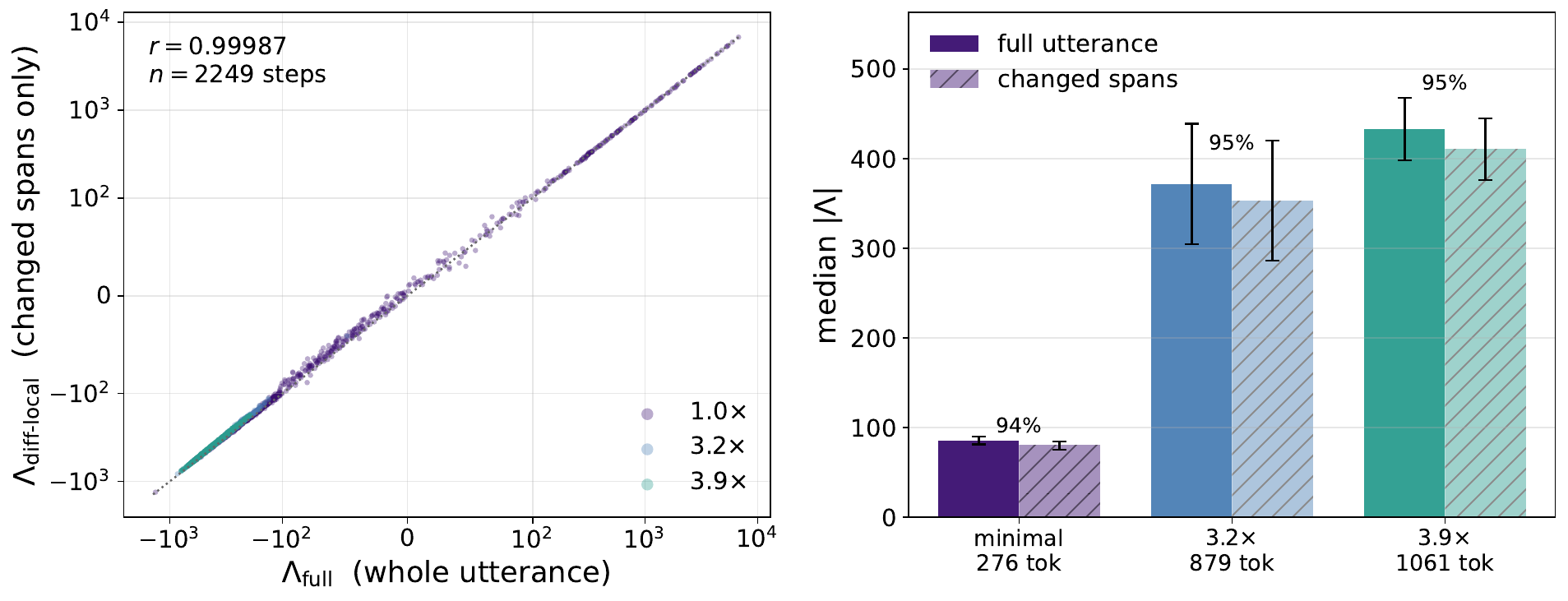}
  \caption{Full-utterance against diff-localized likelihood shift. \textbf{Left:} every logged
  step of the three edit arms, on symmetric log axes with the identity line dotted.
  \textbf{Right:} median $|\Lambda|$ per arm under both measurements, with the localized value
  as a percentage of the full one.}
  \label{fig:rq2-fullvslocal}
\end{figure}

\noindent\textbf{Edit size and the likelihood shift, step by step.}
\Cref{fig:rq2-delta-vs-edit} plots every logged step of the edit arms, with the edit size realized
at that step against the likelihood shift it demands. The relation is strongly monotone but not
linear. Pooling the arms, the median $|\Lambda|$ rises across every edit-size decile, from $82$ in
the smallest to $570$ in the largest (Spearman $\rho=0.83$), while the Pearson correlation is only
$0.28$. Both concavity and a heavy tail depress it, and the tail dominates: dropping the largest
$2\%$ of $|\Lambda|$ raises Pearson to $0.71$, whereas linearizing the concavity alone, by
correlating against $\log|\Lambda|$, reaches $0.50$. The effect also diminishes with the size of the edit: the
per-step Theil--Sen slope is $0.72$ [$0.63$, $0.82$] at the minimal edit against $0.17$
[$0.16$, $0.18$] and $0.14$ [$0.13$, $0.15$] at the two larger rewrites, and the relation holds
within each arm on its own ($\rho=0.52$, $0.82$, $0.78$). We therefore read edit size as a
monotone indicator of the likelihood shift rather than a proportional one.

\begin{figure}[H]
  \centering
  \includegraphics[width=.35\linewidth]{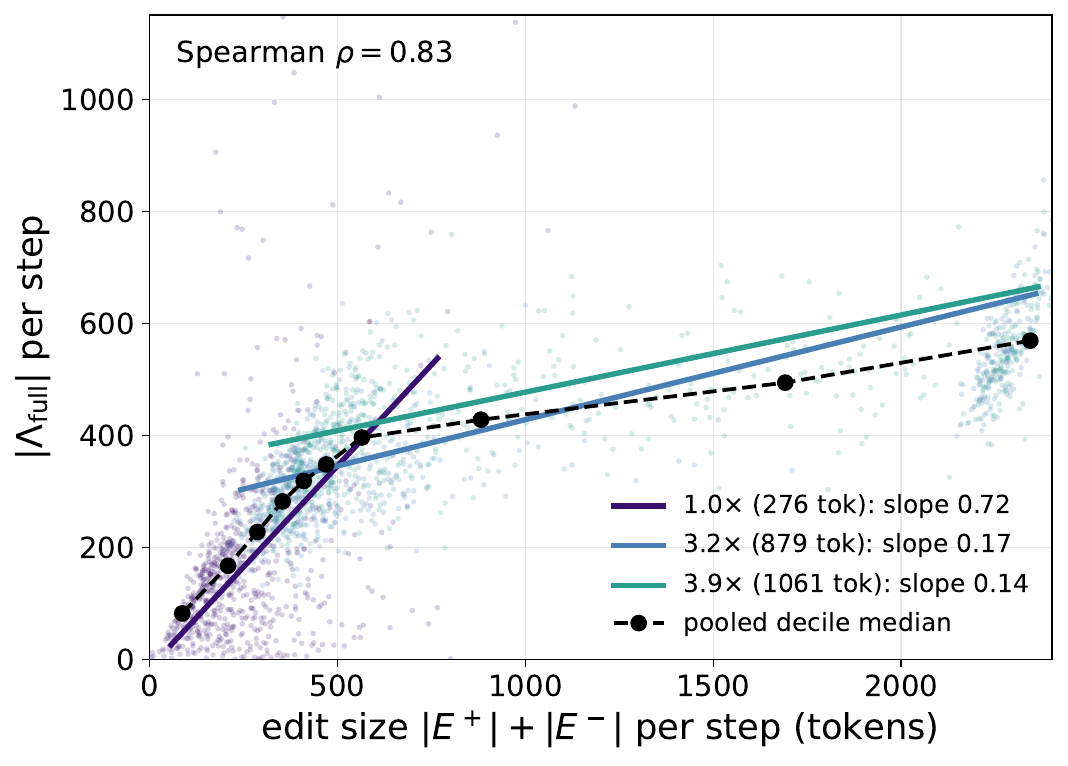}
  \caption{Per-step edit size against the likelihood shift it demands, for the three edit arms
  (one point per logged step, $n=2249$). Colored lines are per-arm Theil--Sen fits with their
  slopes; the dashed black line joins the pooled decile medians. Axes are clipped at the
  $98$th/$99.5$th percentile, so a few staleness-affected steps fall outside the view.}
  \label{fig:rq2-delta-vs-edit}
\end{figure}

\noindent\textbf{Insertion and deletion across the arms.}
\Cref{fig:rq2-eplus-eminus} reports the per-step ratio $|E^+|/|E^-|$. It falls monotonically along
the axis, from $2.46$ at the minimal edit through $1.33$ and $0.99$ at the larger rewrites to
$0.13$ for \emph{gold}, so what changes across the arms is not only the amount rewritten but
whether the correction predominantly adds to the attempt or replaces it.

\noindent\textbf{The two edit-size measures.} That asymmetry is also what separates the quantity
\Cref{sec:rq2-onpolicyness} reports from the one \Cref{thm:single} uses.
\Cref{tab:rq2-editmeasures} gives both for every arm.

\begin{table}[H]
  \centering
  \small
  \caption{\textbf{Token diff against the edit size of \Cref{def:edit}.} The token diff
  $|E^+|+|E^-|$ is what \Cref{sec:rq2-onpolicyness} reports. The edit size $m$ counts over the $H$
  decisions \Cref{def:patch} gives both rollouts, so it also carries the padding that equalizes
  their lengths: $m=|E^+|+|E^-|+\big||E^+|-|E^-|\big|$. Means over training steps and $3$ seeds,
  rounded. The two measures order the arms the same way except for \emph{gold}, which is
  intermediate under the token diff and the largest edit of any arm under $m$.}
  \label{tab:rq2-editmeasures}
  \begin{tabular}{lccccc}
  \toprule
  Arm & $|E^+|$ & $|E^-|$ & token diff & length change & $m$ \\
  \midrule
  minimal (ours) & $160$ & $116$ & $276$  & $44$  & $319$ \\
  $3.2\times$    & $308$ & $572$ & $879$  & $264$ & $1144$ \\
  $3.9\times$    & $359$ & $702$ & $1061$ & $343$ & $1405$ \\
  \emph{gold}    & $85$  & $841$ & $926$  & $755$ & $1681$ \\
  \bottomrule
  \end{tabular}
\end{table}

\begin{figure}[H]
  \centering
  \includegraphics[width=.35\linewidth]{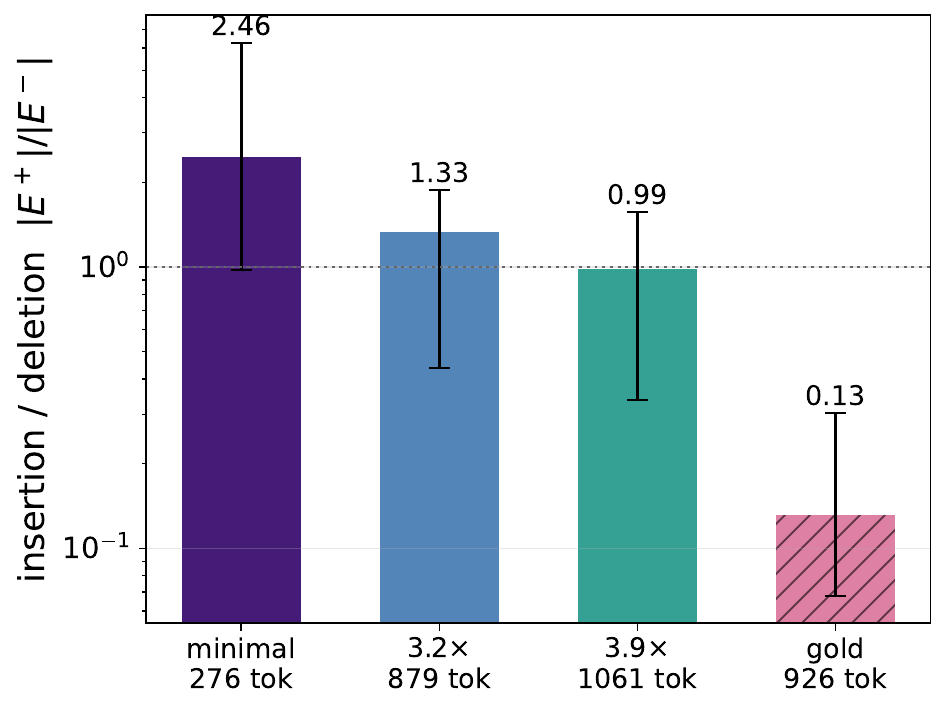}
  \caption{Insertion-to-deletion balance per arm, per-step median with interquartile range on a
  logarithmic axis. The dotted line marks equal insertion and deletion.}
  \label{fig:rq2-eplus-eminus}
\end{figure}

\subsection{Additional results}

\noindent\textbf{Specialization--retention plane.}
\Cref{fig:rq2-pareto} plots the two axes against each other. The minimal edit occupies the
upper-right region, dominating both larger rewrites on both axes, while \emph{gold} attains
intermediate specialization at the lowest retention of any arm.

\begin{figure}[H]
  \centering
  \includegraphics[width=.35\linewidth]{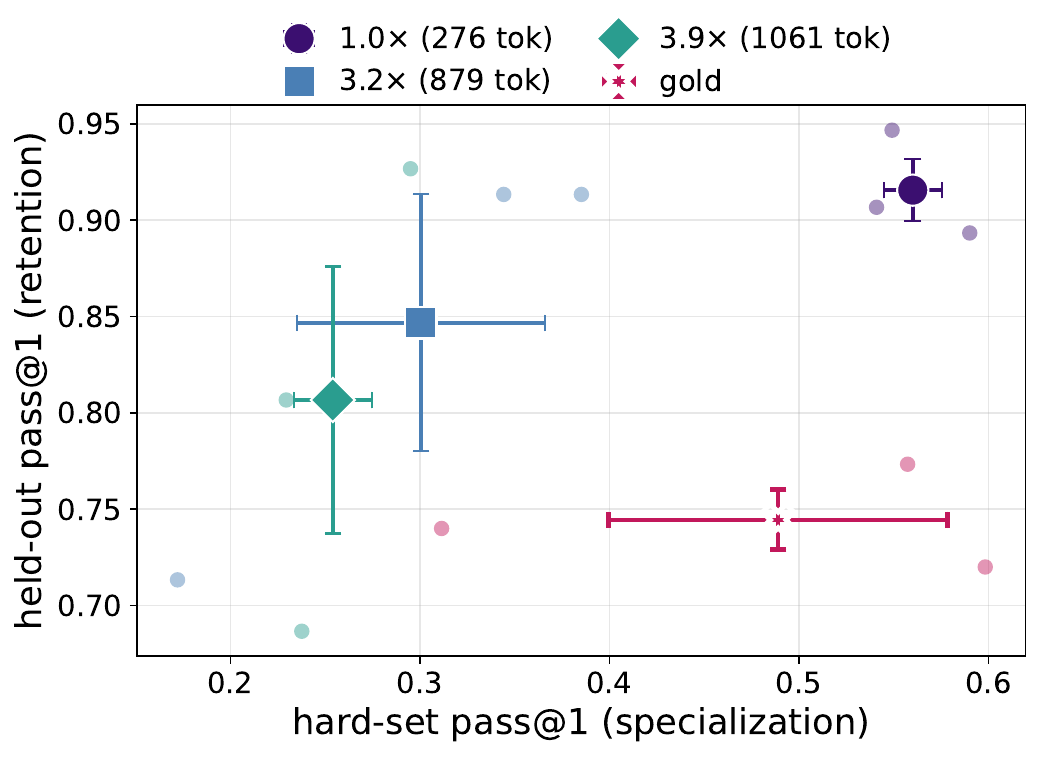}
  \caption{RQ2 specialization--retention plane ($3$ seeds; faint markers are individual seeds).
  The minimal edit dominates both larger rewrites on both axes; \emph{gold} is dominated on
  retention.}
  \label{fig:rq2-pareto}
\end{figure}

\noindent\textbf{Held-out trajectory.}
\Cref{fig:rq2-traj} tracks held-out accuracy over training. The edit arms separate in
edit-size order and decline smoothly, the minimal edit holding near its base value throughout.
\emph{gold} instead destabilizes early, falling to $0.54$ by step $75$ and remaining volatile
for the rest of the run before settling at $0.744$.

\begin{figure}[H]
  \centering
  \includegraphics[width=.35\linewidth]{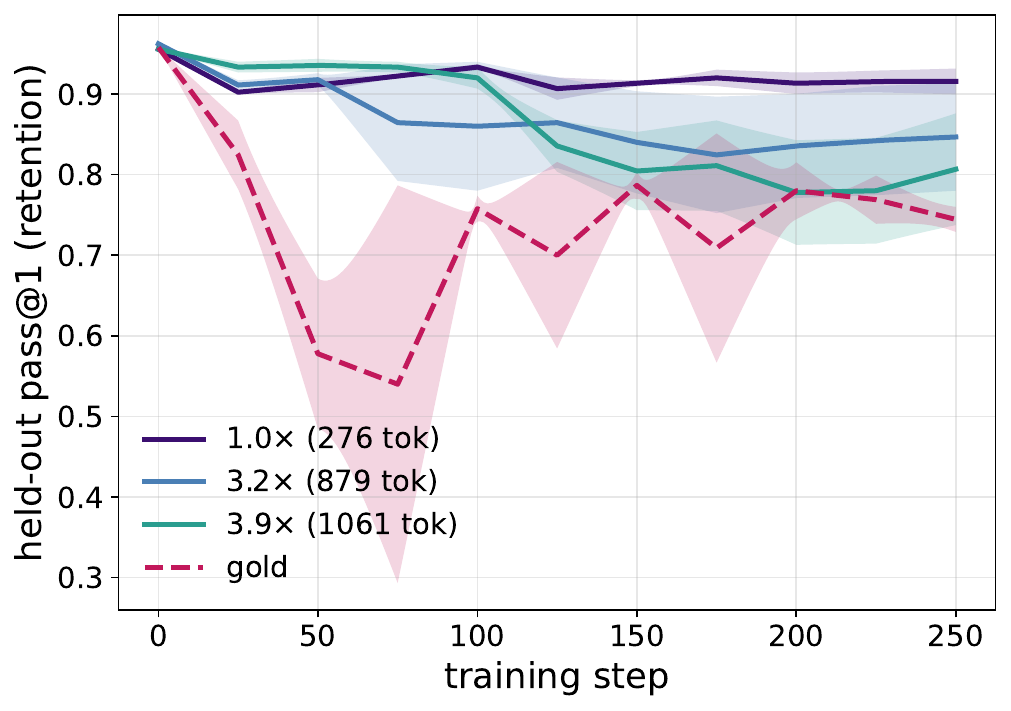}
  \caption{Held-out accuracy over training (mean $\pm$ s.e.\ over $3$ seeds). The edit arms
  decline smoothly and in edit-size order; only \emph{gold} destabilizes.}
  \label{fig:rq2-traj}
\end{figure}

\section{RQ4: Loss-recipe details}
\label{app:rq4-recipe}

\subsection{Objectives}

\noindent\textbf{Rollout terminology.} Each training step, the student rolls out attempts on the
hard problems and the verifier splits them: a \emph{correct rollout} $y^{\mathrm c}$ is one
it solves on its own ($R(x,y^{\mathrm c})=1$); a \emph{wrong rollout} $y^-$ is one it fails
($R=0$). For each wrong rollout the frozen projector produces the minimal-edit
\emph{projection} $y^+$ (verified-correct, \Cref{sec:method}), and a token-diff aligns
$(y^-,y^+)$ into the mistake span $E^-$ and fix span $E^+$ used by \eqref{eq:proj-dpo}. The
\emph{full-sequence} DPO variant replaces $E^+,E^-$ in \eqref{eq:proj-dpo} with the full
sequences $y^+,y^-$.

\noindent\textbf{Per-arm objectives.} The four contrastive arms optimize
$\mathcal{L}^{E}_{\mathrm{DPO}}$ (diff-local, on the changed spans $E^\pm$) or
$\mathcal{L}^{\mathrm{seq}}_{\mathrm{DPO}}$ (full-sequence, on $y^\pm$), each with or without
an added self-SFT term $\mathcal{L}_{\mathrm{SFT}}(y^{\mathrm c})$ on the student's own
correct rollouts. \emph{RFT} optimizes $\mathcal{L}_{\mathrm{SFT}}(y^{\mathrm c})$ alone (no
projected target, no contrast); \emph{SFT on projection $+$ self-SFT} optimizes
$\mathcal{L}_{\mathrm{SFT}}(y^+)+\mathcal{L}_{\mathrm{SFT}}(y^{\mathrm c})$ (SFT on the
projected correction $y^+$ plus self-SFT on the student's own correct rollouts, no contrast).

\subsection{Results}

Compute (batch size), number of epochs and learning rate are matched across arms.
\Cref{tab:rq4-arms} gives the absolute numbers behind \Cref{fig:rq4-loss}.

\begin{table}[H]
  \centering
  \small
  \caption{\textbf{RQ4 per-arm results} on the GSM8K hard set with Qwen3-8B, mean\,$\pm$\,s.e.\
  across $3$ seeds. \emph{hard} is hard-set pass@1 and \emph{rel.} the same quantity as a fraction
  of the adopted recipe, which is also the best arm; \emph{held} is held-out pass@1.
  \textbf{Bold} $=$ best per column. SFT on the projection carries the widest spread on both axes.}
  \label{tab:rq4-arms}
  \begin{tabular}{lccc}
    \toprule
    Arm & hard\,$\uparrow$ & rel.\,$\uparrow$ & held\,$\uparrow$ \\
    \midrule
    Diff-local DPO $+$ self-SFT & $\mathbf{0.601}${\scriptsize$\pm$0.005} & $\mathbf{1.000}$ & $0.938${\scriptsize$\pm$0.002} \\
    Full-sequence DPO $+$ self-SFT & $0.552${\scriptsize$\pm$0.022} & $0.918$ & $0.911${\scriptsize$\pm$0.014} \\
    SFT on projection $+$ self-SFT & $0.484${\scriptsize$\pm$0.078} & $0.805$ & $0.877${\scriptsize$\pm$0.054} \\
    Full-sequence DPO & $0.470${\scriptsize$\pm$0.026} & $0.782$ & $0.902${\scriptsize$\pm$0.002} \\
    Diff-local DPO & $0.440${\scriptsize$\pm$0.015} & $0.732$ & $0.920${\scriptsize$\pm$0.014} \\
    RFT & $0.096${\scriptsize$\pm$0.011} & $0.159$ & $\mathbf{0.944}${\scriptsize$\pm$0.002} \\
    \bottomrule
  \end{tabular}
\end{table}

Three things the table does not show. RFT is limited by how rarely the student succeeds unaided:
the hard set is mined at $\mathrm{pass}@k{=}0$, and its rollouts stay correct on about $6\%$ of
attempts through the run, so it trains on very little. The instability of SFT on the projection
runs across both axes, its standard error of $0.078$ on specialization and $0.054$ on retention
being three to sixteen and four to twenty-seven times those of any other arm. And the self-SFT
term is worth $+0.082$ to the full-sequence contrast and $+0.161$ to the diff-local one in absolute
pass@1: because the projection edits few tokens, a full-sequence contrast is spread over spans that
are already correct while a diff-local one acts only where the edit fell, and that concentration
leaves it with less absolute signal of its own, so it gains more from a term that raises the
likelihood of correct text outright.

\section{RQ5: Refinement details}
\label{app:rq5-refine}

\subsection{Refinement protocol}

A correction is \emph{valid} only if a reasoning check
confirms that its derivation process actually yields the answer, so a projection that reaches the
right answer by an unsound route is discarded. Round $0$ is a single correction pass. In each subsequent round
the frozen projector is re-prompted with its own previous correction together with the spans it
edited in that correction, and returns a refined correction. Across rounds we keep the valid
correction that the student itself is most likely to produce, scored by the base student's
log-likelihood on the edited tokens, which selects for corrections the student can absorb rather
than for the ones the projector prefers. The projector is never updated. All arms train the
Qwen3-8B student for $250$ steps with the default recipe of \Cref{sec:method-blackbox} over $3$
seeds.

\subsection{Results}

\begin{table}[H]
  \centering
  \small
  \caption{\textbf{RQ5 per-budget results} on the GSM8K hard set with Qwen3-8B,
  mean\,$\pm$\,s.e.\ across $3$ seeds. The first three columns describe the corrections the frozen
  projector supplies, averaged over training; the last two the student that learns from them, at
  the final checkpoint. \emph{corr.} is the running correction rate
  $\sum_t n^{\mathrm{corr}}_t/\sum_t n^{\mathrm{wrong}}_t$ plotted in
  \Cref{fig:rq5-refinement} (left), and \emph{edit} and median $|\Lambda|$ are the two bars in
  \Cref{fig:rq5-refinement} (right). \textbf{Bold} $=$ best per column. Every column describing the
  corrections improves monotonically with the budget; downstream, the specialization gain arrives
  in the first round.}
  \label{tab:rq5-arms}
  \begin{tabular}{lccccc}
    \toprule
    & \multicolumn{3}{c}{Corrections} & \multicolumn{2}{c}{Student} \\
    \cmidrule(lr){2-4}\cmidrule(lr){5-6}
    $K$ & corr.\ (\%)\,$\uparrow$ & edit $|E^+|{+}|E^-|$\,$\downarrow$ & median $|\Lambda|$\,$\downarrow$ & hard\,$\uparrow$ & held\,$\uparrow$ \\
    \midrule
    $0$ & $9.2${\scriptsize$\pm$0.2} & $256${\scriptsize$\pm$70} & $38.1${\scriptsize$\pm$1.8} & $0.366${\scriptsize$\pm$0.023} & $0.911${\scriptsize$\pm$0.014} \\
    $1$ & $15.4${\scriptsize$\pm$0.7} & $182${\scriptsize$\pm$16} & $34.2${\scriptsize$\pm$3.0} & $\mathbf{0.440}${\scriptsize$\pm$0.010} & $0.909${\scriptsize$\pm$0.010} \\
    $2$ & $\mathbf{16.7}${\scriptsize$\pm$1.1} & $\mathbf{130}${\scriptsize$\pm$24} & $\mathbf{26.2}${\scriptsize$\pm$4.8} & $\mathbf{0.440}${\scriptsize$\pm$0.021} & $\mathbf{0.929}${\scriptsize$\pm$0.006} \\
    \bottomrule
  \end{tabular}
\end{table}

\noindent\textbf{Projection quality.} Refinement improves the corrections on three counts at once,
and the three are linked. Re-prompting the projector with the spans it has already edited pushes it
to consolidate rather than rewrite, so the correction it returns reaches the right answer more
often and departs from the student's own attempt less: the median $|\Lambda|$ falls from $38.1$ to
$26.2$ over the two rounds, placing the refined targets nearer the student in the sense
\Cref{sec:rq2-onpolicyness} measures, and the edit size falls with it. The $K{=}0$ edit size
carries a standard error of $70$ tokens across seeds, so the level of that row is the least settled
number in the table; the ordering across budgets is what the comparison rests on.

\noindent\textbf{Downstream trade-off.} \Cref{fig:rq5-pareto} traces the attainable frontier over
training for each budget. Refinement extends it: the furthest specialization reached on the
frontier rises from $0.393$ at $K{=}0$ to $0.459$ at $K{=}1$ and $0.475$ at $K{=}2$. The $K{=}2$
frontier also lies above the other two over the whole sweep, and is the only one that rises above
the shared base in retention, reaching $0.960$ against the base's $0.955$, while $K{=}0$ and
$K{=}1$ only decline from it; the extension in specialization is therefore not paid for out of
retention. At the final checkpoint (\Cref{tab:rq5-arms}) the two budgets that refine are level on
specialization and $K{=}2$ holds retention slightly higher; against standard errors of
$0.010$--$0.021$ neither difference is resolved, and what separates the three budgets is the first
round.

\begin{figure}[H]
  \centering
  \includegraphics[width=.35\linewidth]{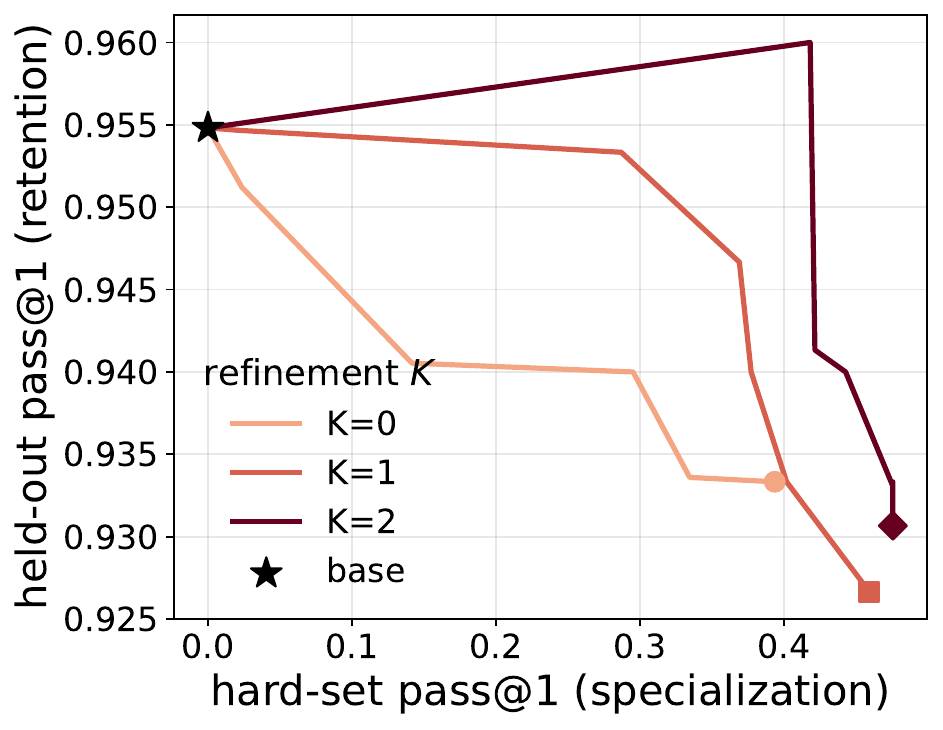}
  \caption{Specialization--retention frontiers over training for $K\in\{0,1,2\}$ ($3$ seeds
  pooled, frontier over the pooled points), gold-known, with the shared base marked $\star$. More
  refinement extends the frontier to higher specialization, and the $K{=}2$ frontier dominates the
  other two: it is the only one that rises above the base in retention. Note the narrow retention
  axis, which spans $0.925$--$0.960$.}
  \label{fig:rq5-pareto}
\end{figure}

\noindent\textbf{Correction rate over training.} \Cref{fig:rq5-corr-perstep} shows the
\emph{per-step} correction rate for $K\in\{0,1,2\}$. For every $K$ the rate rises early and then
declines as training proceeds, because the student learns the easily-correctable problems and the
wrong attempts that remain in later training are the hardest ones, which even the refined
projector rarely corrects. The main-text \Cref{fig:rq5-refinement} reports the running (cumulative)
version of this quantity.

\begin{figure}[H]
  \centering
  \includegraphics[width=.35\linewidth]{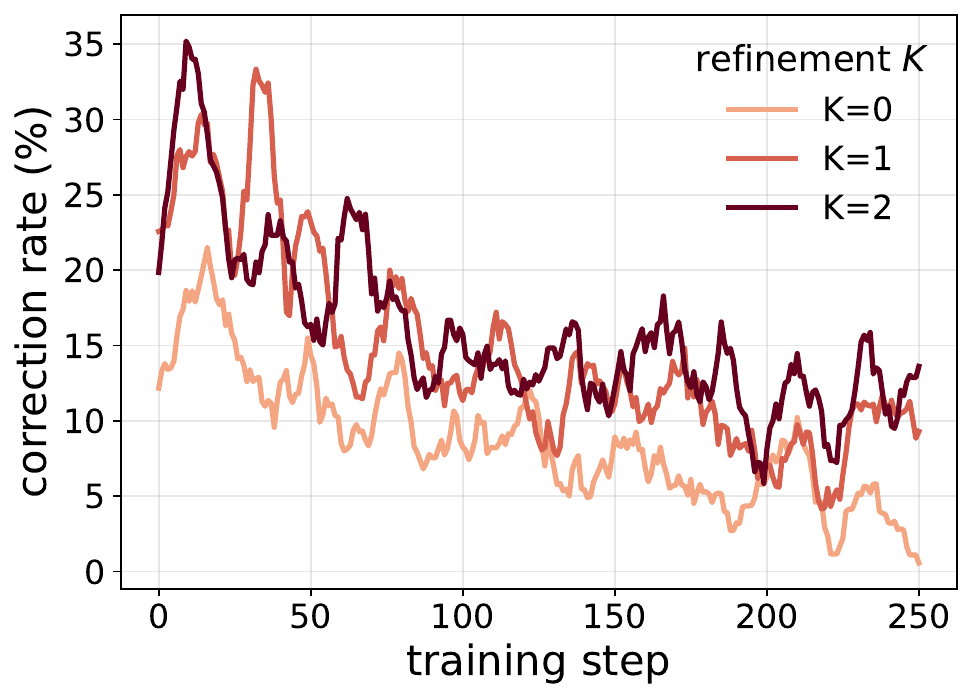}
  \caption{Per-step correction rate over training (pooled over $3$ seeds, $11$-step moving
  average), gold-known, $K\in\{0,1,2\}$. The rate peaks early and declines as only the hardest
  problems remain uncorrected; more refinement holds it higher throughout.}
  \label{fig:rq5-corr-perstep}
\end{figure}

\section{RQ6: Self-projection details}
\label{app:selfprojection}

Everywhere else in this paper the projector is a separate and stronger model. Here we make the same
model write the correction. We give the student its own wrong attempt together with privileged
information, the verified final answer, and ask it for a correct solution that stays close to what
it wrote. That output is then used exactly as a projector's correction would be, with the same
objective and the same training loop. Nothing else changes, so the run measures how much of the
method depends on the projector being stronger than the student.

The student is Qwen3-8B and the hard set is the $122$ GSM8K problems it fails on all $16$ attempts.
The teacher is a frozen copy of the student's own base weights, so it does not improve as the
student does. We train for $250$ updates over $3$ seeds, with the same data, budget and evaluation
as the edit-size runs of \Cref{sec:rq2-onpolicyness}, from which the reference rows are taken. The answer is supplied with an instruction to derive it rather than assert it;
without that, the model appends the answer line to unchanged wrong reasoning and the grader accepts
it. We also run the same idea at the logit level, where the frozen base model scores the student's
own generation under a prompt carrying the answer and the student trains on that token distribution
instead of on edited text.

\begin{table}[H]
  \centering
  \small
  \caption{\textbf{Self-projection recovers about $40\%$ of the specialization.} Qwen3-8B on the
  GSM8K hard set, $250$ updates, mean $\pm$ s.e. over $3$ seeds. The two middle rows use the
  student's own frozen base weights as the teacher. The reference rows come from
  \Cref{app:rq2-mechanism}, on the same hard set and budget.}
  \label{tab:selfprojection}
  \begin{tabular}{llcc}
  \toprule
  Teacher & Correction enters as & Hard set & Held-out \\
  \midrule
  none (base student) & -- & 0.000 & $\approx$0.955 \\
  \midrule
  the student itself & edited text & 0.219{\scriptsize$\pm$0.012} & 0.924{\scriptsize$\pm$0.021} \\
  the student itself & token distribution & 0.046{\scriptsize$\pm$0.010} & 0.937{\scriptsize$\pm$0.003} \\
  \midrule
  gold solution & full replacement & 0.489 & 0.744 \\
  black-box projector & minimal edit & \textbf{0.560} & 0.916 \\
  \bottomrule
  \end{tabular}
\end{table}

Self-projection works, but it does not match a stronger projector. It lifts the hard set from
$0.000$ to $0.219$ while retention stays at $0.924$, so an 8B model that is told the answer can
repair its own reasoning often enough to learn from it. A black-box projector reaches $0.560$ on
the same problems, and even the gold solution reaches $0.489$. Self-projection recovers about
$40\%$ of the specialization, which says that the projector's strength, and not the privileged
answer by itself, is what the method mostly relies on.

The logit version gives almost nothing, $0.046$ after $250$ updates. The answer does reach the
teacher, and the student never sees it, so the two token distributions differ from the first step;
what is missing is not information but capability. The teacher is the student's own base weights,
and being told which answer is correct does not let those weights produce a correct derivation, so
the distribution it supplies over the student's own tokens carries little the student can use. The
measured loss stays between $0.003$ and $0.006$ throughout, and making the target denser, a pure
forward {KL} against the teacher's top-$20$ tokens rather than its top-$1$, roughly doubled the
result from a very small base.

Two points help in reading the table. The text-level and logit-level arms inherit different
learning rates from their respective families in the main experiments ($10^{-5}$ and
$5\times10^{-6}$), so each is comparable to its own column rather than to the other; the comparison
against the black-box projector is at a matched $10^{-5}$. The 8B also does not edit minimally. It
appends a fresh derivation rather than repairing the attempt, which is part of why it falls short:
producing a small, targeted edit is precisely what the weaker projector cannot do.

\section{Comparison with concurrent self-projection}
\label{app:selfproj}

Concurrent work~\citep{karan2026sampling} also moves the supervision toward the student, but does so
by \emph{self-projection}: an MCMC sampler in which the student itself, shown the problem, a reference
solution and a partial response, regenerates the response block by block ($32$ blocks of $58$
tokens, $10$ proposals per block), accepting a proposal when it raises the student's mean
log-likelihood. The student is then fine-tuned on the resulting traces (\emph{Sampling SFT}),
optionally followed by GRPO (\emph{Sampling SFT + RL}). We reproduce the math setting of that work's
Table~1 with its released code, data, training recipes and evaluation scripts, and change only where
the correction comes from: the black-box projector minimally edits each wrong attempt given the
reference answer (\Cref{app:prompts}), and that work's grader must accept the edit.

\noindent\textbf{Setup.} The student is Qwen2.5-3B, trained on that work's $7{,}962$ MATH training
problems. It answers each once with that work's prompt ($T{=}0.6$, $1{,}856$ tokens) and solves
$3{,}675$; the projector corrects $4{,}265$ of the other $4{,}287$ attempts within two tries ($4{,}317$
calls, $11.2$M tokens). We train four arms from these data, one training seed each:
\begin{itemize}[leftmargin=1.4em,itemsep=1pt,topsep=1pt]
  \item \textbf{Black-box projection + SFT}: that work's SFT recipe (learning rate $5\times10^{-5}$,
    $2$ epochs, batch $16$, cosine schedule with $10$ warm-up steps, sequences up to $2{,}048$
    tokens) on the $3{,}675$ correct answers and the $4{,}265$ projections, in place of its MCMC
    traces.
  \item \textbf{Black-box projection + SFT + RL}: that work's RL stage on the previous arm (GRPO for
    $500$ updates of $256$ problems $\times$ $4$ rollouts, learning rate $10^{-6}$, KL coefficient
    $0.001$).
  \item \textbf{Ours, static pool}: the objective of \eqref{eq:proj-objective} on the same data, the
    $4{,}265$ pairs $(y^-,y^+)$ supplying the contrast and the $3{,}675$ correct answers the self-SFT
    term, for $1{,}985$ updates of $8$ records ($2$ passes, as for SFT) at learning rate $10^{-5}$. As
    in \Cref{sec:rq3-refresh}, the corrections are mined once against the untrained student.
  \item \textbf{Ours, online refresh}: the full recipe of \Cref{sec:method}, in which every update
    samples the current student on $8$ problems and projects its wrong answers, for $250$ updates.
\end{itemize}
Evaluation follows that work's protocol: one sample per problem at $T{=}0.6$ with up to $1{,}856$
tokens, on MATH levels $3$--$5$ ($1{,}024$ problems), AMC ($83$), MATH-500 ($500$) and GSM8K
($1{,}320$) for the trained skill, and on SciKnowEval chemistry ($600$) and MMLU
(lm-evaluation-harness~\citep{gao2024lmeval}, $0$-shot) for prior capabilities. Our SFT ran in TRL
rather than verl~\citep{sheng2024verl}, with
the same settings.

\begin{table}[H]
  \centering
  \footnotesize
  \setlength{\tabcolsep}{3.6pt}
  \caption{\textbf{Black-box projection matches self-projection} in the math setting of
  \citet{karan2026sampling} (Qwen2.5-3B). Upper block as reported there, not rerun; lower block
  ours, evaluated with its released scripts and protocol (one sample per problem). Avg.\ is the mean
  of the four math columns. $^{*}$ $=$ paired change of at least two standard errors against our own
  evaluation of the untrained student (math and chemistry columns).}
  \label{tab:selfproj}
  \begin{tabular}{lccccc|cc}
  \toprule
  & MATH(3,4,5) & AMC & MATH-500 & GSM8K & Avg. & Chemistry & MMLU \\
  \midrule
  Base (reported) & 0.315 & 0.133 & 0.245 & 0.579 & 0.318 & 0.283 & 0.651 \\
  SFT & 0.243 & 0.100 & 0.168 & 0.455 & 0.241 & 0.222 & 0.648 \\
  OPSD~\citep{zhao2026opsd} & 0.267 & 0.084 & 0.332 & 0.524 & 0.302 & 0.242 & 0.652 \\
  GRPO & 0.457 & 0.249 & 0.313 & 0.808 & 0.457 & 0.278 & 0.652 \\
  UFT~\citep{liu2025uft} & 0.470 & 0.293 & 0.297 & 0.746 & 0.452 & 0.283 & 0.653 \\
  Sampling SFT & 0.495 & 0.277 & 0.582 & 0.782 & 0.534 & 0.266 & 0.651 \\
  Sampling SFT + RL & 0.545 & 0.241 & 0.652 & 0.830 & 0.567 & 0.285 & 0.652 \\
  \midrule
  Black-box projection + SFT & 0.514$^{*}$ & 0.277 & 0.578 & 0.812$^{*}$ & 0.545 & 0.287 & 0.649 \\
  Black-box projection + SFT + RL & 0.561$^{*}$ & 0.301 & 0.622$^{*}$ & 0.832$^{*}$ & 0.579 & 0.268 & 0.645 \\
  Ours, static pool & 0.553$^{*}$ & 0.277 & 0.586 & 0.812$^{*}$ & 0.557 & 0.262 & 0.651 \\
  Ours, online refresh & 0.488$^{*}$ & 0.277 & 0.536 & 0.805$^{*}$ & 0.527 & 0.250 & 0.654 \\
  \bottomrule
  \end{tabular}
\end{table}

\noindent\textbf{Findings.} SFT on black-box projections reaches a math average of $0.545$ against
$0.534$ for Sampling SFT, and $0.579$ against $0.567$ once the same RL stage is added
(\Cref{tab:selfproj}). Within our own runs, the objective of
\eqref{eq:proj-objective} on the static pool beats SFT on the same data ($0.557$ against $0.545$;
$+0.039\pm0.014$ on MATH(3,4,5)) and matches SFT + RL on MATH(3,4,5) to within noise, at
$1/24$ of its compute (\Cref{tab:selfproj-compute}). Prior capabilities
hold for every arm: MMLU stays within $0.006$ of base, and no change in Chemistry is significant.

Two features of this setting bear on the online arm. The static pool covers $7{,}962$ problems and
is passed over only twice, so it never reaches the regime in which a pool has been reused for
several rounds, which is where the static arm of \Cref{sec:rq3-refresh} falls; and the online run
trains on an eighth as many records. Under these conditions the static pool scores higher ($0.557$
against $0.527$).

\begin{table}[H]
  \centering
  \small
  \caption{\textbf{Training compute} of each arm. Student FLOPs use $2N$ per forward token, $6N$ per
  training token and $12N$ per RL token ($N{=}3.09$B; attention's context term omitted, about
  $5\%$ at these lengths). The MCMC sampler is counted \emph{with} prefix caching, which is
  the accounting most favourable to it; counted from its released implementation without caching it
  costs $1.3\times10^{19}$ and the ratios become $46\times$, $24\times$ and $184\times$.}
  \label{tab:selfproj-compute}
  \begin{tabular}{lcc}
  \toprule
  Arm & Student FLOPs & Against Sampling SFT \\
  \midrule
  Sampling SFT & $8.4\times10^{18}$ & $1\times$ \\
  Sampling SFT + RL & $2.1\times10^{19}$ & -- \\
  Black-box projection + SFT & $2.8\times10^{17}$ & $30\times$ less \\
  Black-box projection + SFT + RL & $1.3\times10^{19}$ & -- \\
  Ours, static pool & $5.3\times10^{17}$ & $16\times$ less \\
  Ours, online refresh & $6.9\times10^{16}$ & $121\times$ less \\
  \bottomrule
  \end{tabular}
\end{table}

\noindent\textbf{Compute.} The MCMC sampler accounts for $98\%$ of Sampling SFT's compute: each
proposal re-reads a $521$-token prompt, regenerates part of the response and is re-scored, about
$2.5\times10^{5}$ token passes per training problem. Black-box projection uses at
least $30\times$ less compute on the student's hardware. The RL stage costs about as much as the
whole MCMC step, so it dominates both arms that include it.

\section{Theoretical analysis}

\subsection{The Patch MDP: Model and Assumption}
\label{app:setup}

This appendix gives the model and the assumption behind \Cref{sec:theory}; \Cref{app:arms-proofs}
gives the proofs.

\begin{definition}[Patch MDP]
\label{def:patch}
We define a \emph{patch MDP} as the tuple $\mathcal M=(\mathcal S,\mathcal A,p,r,\rho,\gamma)$ over a
vocabulary $\mathcal V$ with token budget $H$:
\begin{itemize}[leftmargin=1.9em,itemsep=1pt,topsep=2pt]
  \item $\mathcal S$: a state is a full history $s=(x,y')$, a problem together with the tokens $y'$
        generated so far; $s_t=(x,y_{1:t})$ is the state before decision $t$, and $|s|$ counts the
        generated tokens.
  \item $\mathcal A$: an action is a \emph{patch}, a token span fitting the remaining budget,
        $\mathcal A(s)=\{w\in\mathcal V^{h}:1\le h\le H-|s|\}$.
  \item $p$: transitions are deterministic and append the patch, $p(s\oplus w\mid s,w)=1$.
  \item $r$: the reward is token-additive,
        $r(s,w)=\sum_{j=1}^{|w|}r_{\mathrm{tok}}(s\oplus w_{<j},w_j)$ with
        $r_{\mathrm{tok}}\in[0,1]$, so a rollout returns at most $H$.
  \item $\rho$: the initial state distribution is the distribution over problems,
        $s_0=(x,\varnothing)$ with $x\sim\rho$.
  \item $\gamma=1$: no discounting. $\mathcal V$ contains a terminal symbol, and a rollout that emits
        it moves to an absorbing state of zero reward and zero value.
\end{itemize}
Each decision emits one token, so a rollout takes exactly $H$ of them, indexed $t=0,\dots,H-1$ with
$s_H$ absorbing. Then
$d^{\pi}(s)=\tfrac1H\sum_{t=0}^{H-1}\Pr_{\pi}[s_t=s]$ is a probability distribution over states, the
\emph{normalized occupancy} of $\pi$.
\end{definition}

Write $V^\star$ and $Q^\star$ for the optimal value and action-value of $\mathcal M$. The return of a
policy and its optimum are
\begin{equation}
\label{eq:return}
  J(\pi)\;=\;\mathbb{E}_{s_0\sim\rho}\,\mathbb{E}_{\pi}\Big[\sum_{t=0}^{H-1}r_{\mathrm{tok}}(s_t,w_t)\Big],
  \qquad
  J^\star\;=\;\max_{\pi}J(\pi)\;=\;\mathbb{E}_{s_0\sim\rho}\big[V^\star(s_0)\big],
\end{equation}
the second equality because an optimal policy is optimal from every state.

\begin{definition}[Optimality gap]
\label{def:gap}
$\Delta(s,w)\coloneqq V^\star(s)-Q^\star(s,w)$ is the return an action forfeits against the best one
available at $s$, and $\Delta_{\max}\coloneqq\sup_{s,w}\Delta(s,w)$, finite because returns lie in
$[0,H]$. It is a constant of $\mathcal M$, and plays the role of the cost-to-go constant $u$ of
\citet{ross2011dagger}.
\end{definition}

\begin{definition}[Edit count and edit size]
\label{def:edit}
For histories $s,s'$,
\[
  e(s,s')\;\coloneqq\;|s|+|s'|-2\,\big|\mathrm{LCS}(s,s')\big|,
\]
where $\mathrm{LCS}(s,s')$ is a longest common subsequence~\citep{hirschberg1975lcs} of their
generated tokens. This counts the insertions and deletions carrying one into the other, the
string-to-string correction distance of \citet{wagner1974string} restricted to those two operations.
The \emph{edit size} of a correction is $m\coloneqq e(y^-,y^+)$, the two rollouts compared at the $H$
decisions \Cref{def:patch} gives them, absorbing padding included. It exceeds the token-diff count
$|E^-|+|E^+|$ of \Cref{sec:method}, which \Cref{sec:rq2-onpolicyness} measures, by exactly the change
in answer length: $m=|E^-|+|E^+|+\big|\,|y^-|-|y^+|\,\big|$.
\end{definition}

\paragraph{The learning problem.}
Training runs for $M$ rounds. In round $i$ the learner $\pi_i$ samples rollouts, a frozen projector
rewrites the wrong ones, and the learner is scored on those rewrites and updated once. The learner
ranges over a policy class $\Pi=\{\pi_\theta:\theta\in\Theta\}$ with $\Theta\subseteq\mathbb R^{d}$
convex of Euclidean diameter $D$, so minimizing over $\pi\in\Pi$ and over $\theta\in\Theta$ are the
same operation; only \Cref{prop:rate} uses this.

\paragraph{Where the loss is read.}
The learner corrects a whole rollout once and scores the correction in a single forward pass, so the
loss is evaluated at \emph{corrected} prefixes, not at the states the learner visits. Write
$\tilde d_i(\tilde s)=\tfrac1H\sum_{t=0}^{H-1}\Pr[\tilde s_t=\tilde s]$ for the occupancy of those
corrected prefixes in round $i$, the state-visitation distribution the loss is read under and the
counterpart of $d^{\pi_i}$. For any distribution $d$ over states, let
\begin{equation}
\label{eq:loss}
  L(d,\pi)\;\coloneqq\;\mathbb{E}_{s\sim d}\,\mathbb{E}_{w\sim\pi(\cdot\mid s)}\big[\Delta(s,w)\big]
\end{equation}
be the mean gap of $\pi$ under $d$. Performance is controlled by $L(d^{\pi_i},\pi_i)$, but the
algorithm can only reduce $L(\tilde d_i,\pi)$.

\paragraph{The projector.}
The projector edits a wrong attempt into a correct one, and $w^\star(s)$ denotes its token at $s$. We
idealize it as an expert: its token is optimal wherever we read it, so $\Delta(s,w^\star(s))=0$.
\Cref{cor:imperfect} gives the bound for a projector that is suboptimal.

\begin{definition}[Surrogate, realizability, average regret]
\label{def:regret}
The surrogate objective $\tilde\varepsilon_i(\pi)$ for policy $\pi$ is the $0$--$1$ disagreement
with the projector, measured under $\tilde d_i$:
\[
  \tilde\varepsilon_i(\pi)
  \;\coloneqq\;\mathbb{E}_{\tilde s\sim\tilde d_i}\Big[\Pr_{w\sim\pi(\cdot\mid\tilde s)}\big[w\neq w^\star(\tilde s)\big]\Big]
  \;=\;\mathbb{E}_{\tilde s\sim\tilde d_i}\big[\,1-\pi(w^\star(\tilde s)\mid\tilde s)\,\big].
\]
The approximation and optimization terms are
\[
  \tilde\varepsilon\;\coloneqq\;\min_{\pi\in\Pi}\frac1M\sum_{i=1}^{M}\tilde\varepsilon_i(\pi),
  \qquad
  \tilde\gamma_M\;\coloneqq\;\frac1M\Big(\sum_{i=1}^{M}\tilde\varepsilon_i(\pi_i)
      -\min_{\pi\in\Pi}\sum_{i=1}^{M}\tilde\varepsilon_i(\pi)\Big),
\]
so that $\tfrac1M\sum_i\tilde\varepsilon_i(\pi_i)=\tilde\varepsilon+\tilde\gamma_M$ and the
learner is \emph{no-regret} when $\tilde\gamma_M\to0$. Here $\tilde\varepsilon$ is a minimum
over the policy class, so it measures how well the class can reproduce the projector's tokens. It is
unaffected by how much return those tokens forfeit; that enters separately, as $\bar\delta$ in
\Cref{cor:imperfect}.
\end{definition}

\subsubsection{The assumption}
\label{app:assumption}

The analysis rests on \Cref{asmp:eta}, stated in \Cref{sec:theory}: the map
$s\mapsto\mathbb{E}_{w\sim\pi(\cdot\mid s)}[\Delta(s,w)]$ is $\eta$-Lipschitz in the edit count
$e(\cdot,\cdot)$, uniformly over $\pi\in\Pi$. It is the only assumption the analysis makes about
the MDP. Three points about it follow.

The Lipschitz constant is uniform over $\Pi$ and over pairs of histories, but the proof invokes it only at the
pairs $(s_t,\tilde s_t)$ that \Cref{lem:transfer} compares, so $\eta$ may be taken as the largest
shift over those pairs.

The condition is placed on the \emph{averaged} gap rather than on $\Delta(s,w)$ at a fixed $w$. The
proof compares $\mathbb{E}_{w\sim\pi}[\Delta]$ at two different prefixes, and moving from one to the
other shifts the action distribution as well as the gap, so both must be controlled. Two pointwise
conditions suffice. Suppose the gap is itself Lipschitz,
$|\Delta(s',w)-\Delta(s,w)|\le\eta_\Delta\,e(s,s')$, and the policy is smooth in the same count,
$\mathrm{TV}\big(\pi(\cdot\mid s),\pi(\cdot\mid s')\big)\le\kappa\,e(s,s')$. Adding and
subtracting $\mathbb{E}_{w\sim\pi(\cdot\mid s')}[\Delta(s,w)]$ splits the difference in two, bounded by
$\eta_\Delta e$ and, since $\Delta$ takes values in $[0,\Delta_{\max}]$ and its oscillation is
therefore at most $\Delta_{\max}$, by $\Delta_{\max}\kappa e$. Together they give \Cref{asmp:eta}
with $\eta\le\eta_\Delta+\Delta_{\max}\kappa$.

\Cref{asmp:eta} is also the only reason any result below depends on the size of an edit. An argument
that knows only \emph{which} decisions sit downstream of an edit pays $\Delta_{\max}$ at each of
them, which is identical for a one-token edit and for a hundred-token edit at the same position:
``the prefix differs'' is a binary predicate and it saturates. \Cref{asmp:eta} instead charges in
proportion to $e(\cdot,\cdot)$, which every edit-size claim below uses.

\subsection{Proofs}
\label{app:arms-proofs}

\subsubsection{The gap decomposes across a patch}

\begin{lemma}[Patch additivity]
\label{lem:patchadd}
For patch MDP,
$\ \Delta(s,w_{1:h})=\sum_{j=1}^{h}\Delta(s\oplus w_{<j},w_j)$.
\end{lemma}
\begin{proof}
Since the transitions of Patch MDP are deterministic (c.f. \Cref{def:patch}),  $Q^\star(s,w_{1:h})=r(s,w_{1:h})+V^\star(s\oplus w_{1:h})$, and by
\Cref{def:patch} the reward splits as
$r(s,w_{1:h})=\sum_{j=1}^{h}r_{\mathrm{tok}}(s\oplus w_{<j},w_j)$. Hence
\begin{align*}
  \Delta(s,w_{1:h})
    &= V^\star(s)-r(s,w_{1:h})-V^\star(s\oplus w_{1:h})
       \tag{by determinism}\\
    &= V^\star(s)-V^\star(s\oplus w_{1:h})-\sum_{j=1}^{h}r_{\mathrm{tok}}(s\oplus w_{<j},w_j)
       \tag{by \Cref{def:patch}}\\
    &= \sum_{j=1}^{h}\big[V^\star(s\oplus w_{<j})-V^\star(s\oplus w_{<j+1})\big]
       -\sum_{j=1}^{h}r_{\mathrm{tok}}(s\oplus w_{<j},w_j)
       \tag{telescoping}\\
    &= \sum_{j=1}^{h}\big[V^\star(s\oplus w_{<j})-r_{\mathrm{tok}}(s\oplus w_{<j},w_j)-V^\star(s\oplus w_{<j+1})\big]\\
    &= \sum_{j=1}^{h}\Delta(s\oplus w_{<j},w_j)
       \tag{by determinism}
\end{align*}
where the last step uses
$Q^\star(s\oplus w_{<j},w_j)=r_{\mathrm{tok}}(s\oplus w_{<j},w_j)+V^\star(s\oplus w_{<j+1})$.
\end{proof}

The total gap along a rollout is therefore the same however the rollout is cut into patches; only
the per-step decomposition changes.

\subsubsection{Regret is the accumulated gap, under the learner's own occupancy}

\begin{lemma}[No distribution shift]
\label{lem:pdl}
Under \Cref{def:patch}, for every round $i$,
\[
  J^\star-J(\pi_i)\;=\;\mathbb{E}_{\pi_i}\Big[\sum_{t=0}^{H-1}\Delta(s_t,w_t)\Big]\;=\;H\,L(d^{\pi_i},\pi_i),
\]
where $\mathbb{E}_{\pi_i}$ is over trajectories $(s_0,w_0,\dots,s_{H-1},w_{H-1})$ generated by
$\pi_i$ from $s_0\sim\rho$, and $L$ is as in \eqref{eq:loss}.
\end{lemma}
\begin{proof}
Fix one trajectory $(s_0,w_0,\dots,s_{H-1},w_{H-1})$ generated by $\pi_i$. By \Cref{def:patch} it
has exactly $H$ decisions and $V^\star(s_H)=0$, since $s_H$ is the absorbing state with zero value. Then
\begin{align*}
  V^\star(s_0)-\sum_{t=0}^{H-1}r_{\mathrm{tok}}(s_t,w_t)
    &= \sum_{t=0}^{H-1}\big[V^\star(s_t)-r_{\mathrm{tok}}(s_t,w_t)-V^\star(s_{t+1})\big]
       \tag{telescoping, $V^\star(s_H)=0$}\\
    &= \sum_{t=0}^{H-1}\Delta(s_t,w_t)
       \tag{$Q^\star(s_t,w_t)=r_{\mathrm{tok}}(s_t,w_t)+V^\star(s_{t+1})$}
\end{align*}
Taking $\mathbb{E}_{\pi_i}$ over trajectories turns the left side into
$\mathbb{E}[V^\star(s_0)]-\mathbb{E}\big[\sum_t r_{\mathrm{tok}}(s_t,w_t)\big]=J^\star-J(\pi_i)$, which is the stated
identity. For the second equality,
\begin{align*}
  \mathbb{E}_{\pi_i}\Big[\sum_{t=0}^{H-1}\Delta(s_t,w_t)\Big]
    &= \sum_{t=0}^{H-1}\mathbb{E}\big[\Delta(s_t,w_t)\big]
       \tag{exactly $H$ terms, by \Cref{def:patch}}\\
    &= H\,\mathbb{E}_{s\sim d^{\pi_i}}\mathbb{E}_{w\sim\pi_i}\big[\Delta(s,w)\big]
       \tag{by the definition of $d^{\pi_i}$}\\
    &= H\,L(d^{\pi_i},\pi_i)
\end{align*}
\end{proof}

Both expectations in $L(d^{\pi_i},\pi_i)$ are taken under $\pi_i$'s own occupancy. No change of
measure is performed, so no concentrability coefficient appears.

\subsubsection{Disagreement with the projector dominates the gap}

\begin{lemma}[Reduction]
\label{lem:reduction}
Under \Cref{def:gap}, with the projector idealized as above, for every $\pi\in\Pi$ and every
round $i$,
\[
  L(\tilde d_i,\pi)\;\le\;\Delta_{\max}\,\tilde\varepsilon_i(\pi).
\]
\end{lemma}
\begin{proof}
Fix a corrected state $\tilde s$ and split the expectation over the projector's token and the rest:
\begin{align*}
  \mathbb{E}_{w\sim\pi}\big[\Delta(\tilde s,w)\big]
    &= \pi(w^\star\mid\tilde s)\,\Delta(\tilde s,w^\star)
       +\sum_{w\neq w^\star}\pi(w\mid\tilde s)\,\Delta(\tilde s,w)\\
    &= \sum_{w\neq w^\star}\pi(w\mid\tilde s)\,\Delta(\tilde s,w)
       \tag{$\Delta(\tilde s,w^\star)=0$, idealized projector}\\
    &\le \Delta_{\max}\sum_{w\neq w^\star}\pi(w\mid\tilde s)
       \tag{by \Cref{def:gap}}\\
    &= \Delta_{\max}\big(1-\pi(w^\star\mid\tilde s)\big)
\end{align*}
Taking $\mathbb{E}_{\tilde s\sim\tilde d_i}$ of both sides and applying \Cref{def:regret} gives the
claim.
\end{proof}

This is the reduction of \citet{ross2011dagger}, applied at the corrected prefixes rather than at
the learner's own states.

\subsubsection{Transfer from the corrected prefixes to the learner's own}

\begin{lemma}[Transfer]
\label{lem:transfer}
Under \Cref{asmp:eta} and \Cref{def:patch}, for every round $i$,
\[
L(d^{\pi_i},\pi_i)\;\le\;L(\tilde d_i,\pi_i)+\eta\,\bar e_i,
  ~\textit{where}~
  \bar e_i\;\coloneqq\;\mathbb{E}_i\Big[\tfrac1H\sum_{t=0}^{H-1}e_t\Big]
\]
and $e_t\coloneqq e(s_t,\tilde s_t)$ is the edit count between the learner's own prefix at decision
$t$ and the corrected prefix at that decision. A round draws one problem, one rollout of $\pi_i$ on
it, and the one correction the projector returns for that rollout; $\mathbb{E}_i$ is the average over
that draw.
\end{lemma}
\begin{proof}
Pair the two distributions decision by decision: the $t$-th term pairs the learner's own prefix
$s_t$ with the corrected prefix $\tilde s_t$, which differ by $e_t$ tokens (\Cref{def:edit}). By \Cref{def:patch} both sides of \eqref{eq:loss} are
averages of exactly $H$ such terms, so
\begin{align*}
  L(d^{\pi_i},\pi_i)
    &= \tfrac1H\sum_{t=0}^{H-1}\mathbb{E}_i\,\mathbb{E}_{w\sim\pi_i(\cdot\mid s_t)}\big[\Delta(s_t,w)\big]
       \tag{by \Cref{def:patch}}\\
    &\le \tfrac1H\sum_{t=0}^{H-1}\mathbb{E}_i\Big[\mathbb{E}_{w\sim\pi_i(\cdot\mid\tilde s_t)}\big[\Delta(\tilde s_t,w)\big]+\eta\,e_t\Big]
       \tag{by \Cref{asmp:eta}}\\
    &= L(\tilde d_i,\pi_i)+\eta\,\mathbb{E}_i\Big[\tfrac1H\sum_{t=0}^{H-1}e_t\Big]\\
    &= L(\tilde d_i,\pi_i)+\eta\,\bar e_i
\end{align*}
\end{proof}

\subsubsection{How much edited text a rollout carries}

\begin{lemma}[Edit budget]
\label{lem:budget}
Let $y$ be the learner's rollout in round $i$ and $\tilde y$ its correction, each of $H$ decisions by
\Cref{def:patch}, and let $m=e(y,\tilde y)$ be the edit size, with $e_t$ as in
\Cref{lem:transfer}. Then $e_t\le m$ at every decision
$0\le t<H$, and consequently
\[
  \sum_{t=0}^{H-1}e_t\;\le\;m\,H .
\]
\end{lemma}
\begin{proof}
Let $u$ index the tokens of $y$ and $v$ those of $\tilde y$. A \emph{monotone matching} between them
is a set $\Phi$ of pairs $(u,v)$ with $y_u=\tilde y_v$, whose order agrees on both sides:
$u<u'\iff v<v'$ for distinct $(u,v),(u',v')\in\Phi$. The matched tokens read in order form a common
subsequence of $y$ and $\tilde y$, so a largest matching has $|\Phi|=|\mathrm{LCS}(y,\tilde y)|$.
Take such a $\Phi$. Since $y$ and $\tilde y$ each have $H$ tokens, $m=2H-2|\Phi|$ by
\Cref{def:edit}. For a decision $t$ put $\Phi_t=\{(u,v)\in\Phi:u\le t,\ v\le t\}$, the pairs
falling inside both prefixes, whose matched tokens form a common subsequence of $y_{1:t}$ and
$\tilde y_{1:t}$.
Both coordinates increase along $\Phi$, so $\max(u,v)$ does too, strictly; every pair outside
$\Phi_t$ has $\max(u,v)\ge t+1$, and distinct pairs take distinct values there, so
\begin{equation}
\label{eq:cut}
  \big|\Phi\setminus\Phi_t\big|\;\le\;H-t .
\end{equation}
Therefore
\begin{align*}
  e_t &= 2t-2\,\big|\mathrm{LCS}(y_{1:t},\tilde y_{1:t})\big|
       \tag{by \Cref{def:edit}}\\
      &\le 2t-2\,|\Phi_t|
       \tag{$\Phi_t$ is a matching}\\
      &= 2t-2\,|\Phi|+2\,\big|\Phi\setminus\Phi_t\big|\\
      &\le 2t-2\,|\Phi|+2\,(H-t)
       \tag{by \eqref{eq:cut}}\\
      &= 2H-2\,|\Phi|\;=\;m
       \tag{by \Cref{def:edit}}
\end{align*}
Summing over the $H$ decisions gives the claim.
\end{proof}

An insertion early in the answer shifts every later prefix, and a count of mismatched positions would
read that shift as a rewrite of the whole suffix. \Cref{def:edit} charges it at the size of the
insertion instead, which is why the drift of \Cref{thm:single} scales with $m$ and not with $H$.

\subsubsection{The bound}

\Cref{thm:single} of \Cref{sec:theory} states that under \Cref{def:patch,def:gap}, \Cref{asmp:eta},
an idealized projector and $m$ the mean edit size over the rounds,
\[
  \min_{i\le M}\big[J^\star-J(\pi_i)\big]
  \;\le\;\underbrace{H\,\Delta_{\max}\,\big(\tilde\varepsilon+\tilde\gamma_M\big)}_{\text{fit and optimization}}
  \;+\;\underbrace{\eta\,m\,H}_{\text{corrected-prefix drift}}.
\]

\begin{proof}[Proof of \Cref{thm:single}]

\begin{align*}
    \min_{i\le M}[J^\star-J(\pi_i)] &\le \tfrac1M\sum_{i=1}^{M}[J^\star-J(\pi_i)] \\
    &\le \tfrac{H}{M}\sum_i L(d^{\pi_i},\pi_i) \tag{by \Cref{lem:pdl}}\\
    &\le \tfrac{H}{M}\sum_i L(\tilde d_i,\pi_i)+\eta H\bar e \tag{by \Cref{lem:transfer}} \\
    & \le H\Delta_{\max}\cdot\tfrac1M\sum_i\tilde\varepsilon_i(\pi_i) + \eta H\bar e  \tag{by \Cref{lem:reduction}}\\
    & \le H\Delta_{\max}(\tilde\varepsilon+\tilde\gamma_M) + \eta H\bar e \tag{by \Cref{def:regret}}\\
    & \le H\Delta_{\max}(\tilde\varepsilon+\tilde\gamma_M) + \eta\,m\,H \tag{by \Cref{lem:budget}}
\end{align*}
where $\bar e=\tfrac1M\sum_i\bar e_i$. \Cref{lem:budget} bounds round $i$ by its own edit size,
$H\bar e_i=\mathbb{E}_i\big[\sum_{t=0}^{H-1}e_t\big]\le m_i\,H$, so
$\bar e\le\tfrac1M\sum_i m_i=m$; no bound on the largest round is needed.
\end{proof}

\subsubsection{Rate in the number of rounds}

\begin{proposition}[Convex reference rate]
\label{prop:rate}
Let $\Pi$ and $\Theta$ be as in \Cref{app:setup}, and suppose in addition that $\theta\mapsto\tilde\varepsilon_i(\pi_\theta)$
is convex on $\Theta$ for every $i$, with gradient bound
$G=\max_i\sup_{\theta\in\Theta}\lVert\nabla_\theta\tilde\varepsilon_i(\pi_\theta)\rVert$. Then \emph{projected} online
gradient descent on $\Theta$ with step size $\alpha=D/(G\sqrt M)$ attains $\tilde\gamma_M\le DG/\sqrt M$.
\end{proposition}
\begin{proof}
Since $\Pi$ is parameterized by $\Theta$, the minimum over $\pi\in\Pi$ in \Cref{def:regret} is a minimum over
$\theta\in\Theta$. Write $\theta^\star\in\arg\min_{\theta\in\Theta}\sum_i
\tilde\varepsilon_i(\pi_\theta)$ for a minimiser, so that the comparator in the definition and the
comparator in the analysis are the same point. Then, with
$g_i=\nabla_\theta\tilde\varepsilon_i(\pi_{\theta_i})$,
\begin{align*}
  M\tilde\gamma_M
    &= \sum_i\tilde\varepsilon_i(\pi_{\theta_i})-\sum_i\tilde\varepsilon_i(\pi_{\theta^\star})
       \tag{by \Cref{def:regret}}\\
    &\le \sum_i g_i^{\top}(\theta_i-\theta^\star)
       \tag{convexity}\\
    &\le \tfrac{D^{2}}{2\alpha}+\tfrac{\alpha}{2}MG^{2}
       \tag{potential $\tfrac{1}{2\alpha}\lVert\theta_i-\theta^\star\rVert^{2}$ telescopes}\\
    &\le DG\sqrt M
       \tag{at $\alpha=D/(G\sqrt M)$}
\end{align*}
where $D$ is the Euclidean diameter of $\Theta$ and $G$ the gradient bound.
Dividing by $M$ gives $\tilde\gamma_M\le DG/\sqrt M$.
\end{proof}

\begin{remark}[Scope of the rate]
\label{rem:nonconvex}
\Cref{prop:rate} is the standard online-convex-optimization guarantee, stated to show \emph{how}
regret decays in the number of rounds. Its convexity hypothesis does not hold for a neural policy.
\Cref{thm:single} uses only $\tilde\gamma_M$ as defined, so nothing below depends on
\Cref{prop:rate}; the $O(1/\sqrt M)$ form should be read as a convex reference point.
\end{remark}

\subsubsection{Reading the bound}

\begin{corollary}[Edit size affects the horizon dependence]
\label{cor:regimes}
Under the hypotheses of \Cref{thm:single}, the drift term has two extremes:
\begin{enumerate}[label=(\alph*),leftmargin=2.2em,itemsep=2pt,topsep=2pt]
  \item minimal edits, $m=O(1)$, give drift $O(\eta H)$;
  \item a correction rewriting the whole answer, $m=\Theta(H)$, gives drift $O(\eta H^{2})$.
\end{enumerate}
\end{corollary}

The quantity being bounded is itself of order $H$, since the return lies in $[0,H]$ by
\Cref{def:patch}. Read against that scale, case (a) puts the drift at the same order as the fit
term and the bound is informative, while case (b) puts it an order above the trivial bound
$J^\star-J\le H$ and the bound says nothing. The horizon dependence is therefore linear at one end
and quadratic at the other, and what selects between them is the edit we choose rather than any
property of the environment.

\subsubsection{Scope of the analysis}
\label{app:scope}

\begin{remark}[The analysis assumes a dense per-token reward]
\label{rem:dense}
\Cref{def:patch} makes the reward token-additive with $r_{\mathrm{tok}}\in[0,1]$, so a rollout's
return is at most $H$ and the trivial bound on $J^\star-J(\pi)$ is $H$. The experiments verify only
the final answer, a single $0$--$1$ outcome, under which the trivial bound is $1$. The drift term
$\eta mH$ is therefore informative only where $\eta m$ is small: the guarantee is a statement about
the dense-reward model, and the per-token credit it assumes is not what the verifier supplies.
\end{remark}

\begin{remark}[A statement about horizon dependence, not about algorithm identity]
\label{rem:dagger}
Setting $e_t\equiv0$, which corresponds to a projector queried at every mistake on the learner's own
prefix so that no edited token enters a conditioning prefix, removes the drift and leaves
$H\Delta_{\max}(\tilde\varepsilon+\tilde\gamma_M)$, the shape of the DAgger guarantee
\citep{ross2011dagger}. The linear--quadratic dichotomy of \Cref{cor:regimes} is the same dichotomy
that separates DAgger from behavior cloning. It does not follow that projected distillation
\emph{is} either algorithm: minimal-edit projection does not query a projector at every visited
state and carries a drift term DAgger does not have, and a full rewrite is not behavior cloning for
the reason in \Cref{rem:eps-free}. The correspondence is between rates.
\end{remark}

\begin{remark}[The drift does not vanish under a perfect fit]
\label{rem:eps-free}
In \citet{ross2011dagger} the quadratic term for behavior cloning scales with the fit error, so a
learner that imitates perfectly suffers no compounding. The drift $\eta mH$ of \Cref{thm:single}
carries no factor of $\tilde\varepsilon$. A learner that matches the projector exactly at every
corrected prefix still pays it, because the mismatch between $\tilde d_i$ and $d^{\pi_i}$ is
exogenous, fixed by the edit that was chosen, rather than a consequence of imperfect fitting.
\end{remark}

\begin{remark}[The analyzed loss is not the deployed loss]
\label{rem:surrogate}
\Cref{thm:single} is proved for the $0$--$1$ disagreement of \Cref{def:regret}. Three consequences.

\emph{First}, no gradient method minimizes a $0$--$1$ loss. Following \citet{ross2011dagger} we state
the guarantee for the $0$--$1$ loss and note that any convex upper bound may be substituted, with
$\tilde\varepsilon$ and $\tilde\gamma_M$ read as the floor and regret of that surrogate; the bound
holds verbatim with a larger $\tilde\varepsilon$.

\emph{Second}, cross-entropy on the correction is such a bound. Since $\log p\le p-1$ for
$p\in(0,1]$, we have $1-\pi(w^\star\mid\tilde s)\le-\log\pi(w^\star\mid\tilde s)$, so
\Cref{thm:single} covers maximum-likelihood fine-tuning on the projection.

\emph{Third}, the deployed objective is not covered. The recipe minimizes a diff-localized contrast
on the changed spans together with a self-SFT term on the learner's own correct rollouts. A
preference contrast is not an upper bound on the disagreement with the projector: it can be reduced
by lowering the likelihood of $y^-$ without raising that of $y^+$, so a small contrast loss does not
imply a small $\tilde\varepsilon$. The analysis characterizes the \emph{target}: how far the
correction lies from the attempt, the quantity the edit-size experiment varies and the one both
objectives share. It does not characterize the contrast.
\end{remark}

\begin{corollary}[A projector that is not optimal]
\label{cor:imperfect}
Suppose the projector's token forfeits $\delta(\tilde s)\coloneqq\Delta(\tilde s,w^\star(\tilde s))\ge0$
rather than nothing, and let $\bar\delta$ be its average over the corrected prefixes. Then, under the
remaining hypotheses of \Cref{thm:single},
\[
  \min_{i\le M}\big[J^\star-J(\pi_i)\big]
  \;\le\;H\bar\delta\;+\;H\,\Delta_{\max}\,(\tilde\varepsilon+\tilde\gamma_M)\;+\;\eta\,m\,H .
\]
\end{corollary}
\begin{proof}
The term $\pi(w^\star\mid\tilde s)\,\Delta(\tilde s,w^\star)$ dropped in \Cref{lem:reduction} no
longer vanishes, so that lemma reads
$L(\tilde d_i,\pi)\le\bar\delta+\Delta_{\max}\tilde\varepsilon_i(\pi)$. Substituting it at the
\Cref{lem:reduction} step of \Cref{thm:single} adds $H\bar\delta$, every term there already carrying
the factor $H$ from \Cref{lem:pdl}, and leaves the rest of the chain unchanged.
\end{proof}
The two error sources stay separate: $\bar\delta$ measures the projector and shrinks only with a
better one, while $\tilde\varepsilon$ measures the policy class and shrinks with a larger one.

\end{document}